\documentclass{article}

    \PassOptionsToPackage{numbers, compress}{natbib}

    \usepackage[final]{neurips_2021}

\usepackage[utf8]{inputenc} 
\usepackage[T1]{fontenc}    
\usepackage{hyperref}       
\usepackage{url}            
\usepackage{booktabs}       
\usepackage{amsfonts}       
\usepackage{nicefrac}       
\usepackage{microtype}      
\usepackage{xcolor}         

\usepackage[chang]{changdefs}
\usepackage{xspace}
\usepackage{graphicx}
\usepackage{wrapfig}
\usepackage{subcaption}
\usepackage[leftcaption]{sidecap}
\usepackage{multirow}
\usepackage{array}
\usepackage[shortlabels]{enumitem}
\setlist{nosep}
\allowdisplaybreaks[2]
\graphicspath{{./figures/}}

\newcommand{\bbEt}{{\tilde \bbE}}
\newcommand{\pptb}{{\bar \ppt}}
\newcommand{\pind}{p^{\pperp\!}}
\newcommand{\qind}{q^{\pperp\!}}

\newcommand{\ourmodel}{CSG\xspace}
\newcommand{\ourmodels}{CSGs\xspace}
\newcommand{\ourmodelablated}{CSGz\xspace}
\newcommand{\exmp}[1]{(\eg, #1)}
\newcommand{\subpm}[1]{\scriptsize{$\pm\!$#1}}

\usepackage{titlesec}
\titlespacing*{\section}{2pt}{4pt}{2pt} 
\titlespacing*{\subsection}{1pt}{3pt}{1pt} 
\titlespacing*{\subsubsection}{1pt}{2pt}{1pt} 

\title{Learning Causal Semantic Representation for Out-of-Distribution Prediction}

\author{%
  Chang Liu$^1$\thanks{Correspondence to: Chang Liu <\texttt{changliu@microsoft.com}>.},
  Xinwei Sun$^1$, Jindong Wang$^1$,
  Haoyue Tang$^2$\thanks{Work done during an internship at Microsoft Research Asia.},
  Tao Li$^{3\dagger}$, \\[2pt]
  \textbf{Tao Qin$^1$, Wei Chen$^1$, Tie-Yan Liu$^1$} \\[2pt]
  $^1$ Microsoft Research Asia, Beijing, 100080. \\
  $^2$ Tsinghua University, Beijing, 100084. \;
  $^3$ Peking University, Beijing, 100871.
}

\begin{document}
\abovedisplayskip=3pt
\belowdisplayskip=5pt
\abovedisplayshortskip=2pt
\belowdisplayshortskip=4pt

\maketitle

\begin{abstract}
  Conventional supervised learning methods, especially deep ones, are found to be sensitive to out-of-distribution (OOD) examples, largely because the learned representation mixes the semantic factor with the variation factor due to their domain-specific correlation, while only the semantic factor \emph{causes} the output.
  To address the problem, we propose a Causal Semantic Generative model (\ourmodel) based on a causal reasoning so that the two factors are modeled separately, and develop methods for OOD prediction from a \emph{single} training domain, which is common and challenging. 
  The methods are based on the causal invariance principle, 
  with a novel design in variational Bayes for both efficient learning and easy prediction. 
  Theoretically, we prove that under certain conditions, \ourmodel can identify the semantic factor by fitting training data, and this semantic-identification guarantees the boundedness of OOD generalization error and the success of adaptation.
  Empirical study shows improved OOD performance over prevailing baselines.
\end{abstract}

\vspace{-1pt}
\section{Introduction} \label{sec:intr}
\vspace{-1pt}

Deep learning has initiated a new era of artificial intelligence where the potential of machine learning models is greatly unleashed. 
Despite the great success, these methods heavily rely on the assumption that data from training and test domains follow the same distribution (\ie, the IID assumption),
while in practice the test domain is often out-of-distribution (OOD), meaning that the test data distribute differently from the training data.
Popular models for predicting the output (or label, response, outcome) $y$ from the input (or covariate) $x$ have been found erroneous when confronted with a distribution change,
even from an essentially irrelevant perturbation like a position shift or background change for images~\citep{ribeiro2016why, beery2018recognition, shen2018causally, he2019towards, arjovsky2019invariant, d2020underspecification}.
These phenomena pose serious concerns on the robustness and trustworthiness of machine learning methods and severely impede them from risk-sensitive scenarios. 

Looking into the problem, although deep learning models allow extracting abstract representation for prediction with their powerful approximation capacity, the representation may unconsciously mix up semantic factors $s$ (\eg, shape of an object) and variation factors $v$ (\eg, background, object position) due to a correlation between them \exmp{desks often appear in a workspace background and beds in bedrooms},
so the model also relies on the variation factors $v$ for prediction via this correlation.
However, this correlation tends to be superficial and spurious \exmp{a desk can also appear in a bedroom, but this does not make it a bed}, 
and may change drastically in a new domain, making the effect from $v$ misleading.
So it is desired to learn a representation that identifies $s$ against $v$.

Formally, the essence of this goal is to leverage \emph{causal relations} for prediction, since the fundamental distinction between $s$ and $v$ is that only $s$ is the cause of $y$. 
Causal relations better reflect basic mechanisms of nature. 
They bring the merit to machine learning that they tend to be universal and \emph{invariant} across domains~\citep{scholkopf2012causal, peters2016causal, rojas2018invariant, magliacane2018domain, buhlmann2018invariance, scholkopf2019causality, scholkopf2021toward}, thus provide the most transferable and reliable information to unseen domains.
This causal invariance has been shown to lead to proper domain adaptation~\citep{scholkopf2012causal, zhang2013domain}, lower adaptation cost and lighter catastrophic forgetting~\citep{peters2016causal, bengio2019meta, ke2019learning}.

In this work, we propose a Causal Semantic Generative model (\ourmodel) 
following a causal consideration to separately model the semantic (cause of prediction) and variation latent factors, and develop OOD prediction methods
with theoretical guarantees on identifiability and the boundedness of OOD prediction error.
Addressing the complaint that OOD prediction and causality methods often require multi-domain or intervention data,
we focus on the most common and also challenging tasks where only one \emph{single} training domain is available, including \emph{OOD generalization} and \emph{domain adaptation},
where in the latter, unsupervised test-domain data are additionally available for training.
The methods and theory are based on the causal invariance principle, which suggests to share generative mechanisms across domains, while the latent factor distribution (\ie, the prior $p(s,v)$) changes.
We argue that this causal invariance is more reliable than \emph{inference invariance} in the other direction adopted by many existing methods~\citep{ganin2016domain, shankar2018generalizing, arjovsky2019invariant, krueger2020out, mitrovic2021representation}.
For our method, we design novel and delicate reformulations of the ELBO objective so that we avoid the cost to build and learn two inference models.
Theoretically, we prove that under certain conditions, \ourmodel \emph{can identify} the semantic factor on the single training domain, even in presence of an $s$-$v$ correlation.
We further prove the merits from this identification: prediction error is bounded for OOD generalization, and for domain adaptation, the test-domain prior is identifiable which leads to an accurate prediction.
To sum up our contributions,
\vspace{-2pt}
\begin{itemize}[leftmargin=18pt]
  \item 
    Up to our knowledge, we are the first to show a theoretical guarantee (under appropriate conditions) to identify the latent cause of prediction (\ie, the semantic factor) on a single training domain,
    and also the first to show the theoretical benefits of this identification for OOD prediction.
    The results also contribute to generative representation learning for revealing what is learned. 
  \item We develop effective methods for OOD generalization and domain adaptation,
    and achieve mostly better performance than prevailing methods 
    on real-world image classification tasks.
\end{itemize}
\vspace{-2pt}

\vspace{0pt}
\section{Related Work} \label{sec:relw}
\vspace{-1pt}

\textbf{OOD generalization with causality.}
There are trials that ameliorate discriminative models towards a causal behavior. 
\citet{bahadori2017causal} introduce a regularizer that reweights input dimensions based on their approximated causal effects to the output, 
and \citet{shen2018causally} reweight training samples by amortizing causal effects among input samples. 
Their linear input-output assumption is then extended~\citep{bahadori2017causal, he2019towards} by learning a representation.
Some recent works require identity data (finer than label) and enforce inference invariance via variance minimization 
\citep{heinze2019conditional}, or leverage a strong domain knowledge to augment images as an independent intervention on variation factors
\citep{mitrovic2021representation}. 
These methods introduce no additional generative modeling efforts, at the cost of limited capacity for invariant causal mechanisms. 

\textbf{Domain adaptation/generalization with causality.}
There are methods developed under various causal assumptions~\citep{scholkopf2012causal, zhang2013domain} 
or using learned causal relations~\citep{rojas2018invariant, magliacane2018domain}.
\citet{zhang2013domain, gong2016domain, gong2018causal} also consider certain ways of mechanism change. 
The considered causality is among directly observed variables, which may not well suit general data like image pixels where causality rather lies in the conceptual latent level~\citep{lopez2017discovering, besserve2018group, kilbertus2018generalization}. 

To consider latent factors, there are domain adaptation~\citep{pan2010domain, baktashmotlagh2013unsupervised, ganin2016domain, long2015learning, long2018conditional} and generalization methods~\citep{muandet2013domain, shankar2018generalizing, wang2021generalizing} that learn a representation with a domain-invariant marginal distribution.
Remarkable results have been achieved.
Nevertheless, it is found that this invariance is neither sufficient nor necessary to identify the true semantics or lower the adaptation error (\citep{johansson2019support, zhao2019learning}; see also Appx.~\ref{supp:da-dir}). 
Moreover, these methods are based on
inference invariance, which may not be as reliable 
as causal invariance (see Sec.~\ref{sec:model-infinv}). 

There are also generative methods for domain adaptation/generalization that model latent factors.
\citet{cai2019learning} and \citet{ilse2020diva} introduce a semantic factor and a domain-feature factor.
They assume the two latent factors are independent in both generative and inference models, which is unrealistic. 
Correlated factors are then considered~\citep{atzmon2020causal}. 
But all these works do not adapt the prior for 
domain change thus resort to inference invariance. 
\citet{zhang2020causal} consider a partially observed manipulation variable, 
while still assuming its independence from the output in both the joint and posterior,
and the adaptation is inconsistent with causal invariance.
The above methods also do not show guarantees to identify their latent factors.
\citet{teshima2020few} leverage causal invariance and adapt the prior, yet also assume latent independence and do not separate the semantic factor.
They require some supervised test-domain data, and their deterministic and invertible mechanism also indicates inference invariance.
In addition, most domain generalization methods require \emph{multiple} training domains, with exceptions~\citep{qiao2020learning} that still seek to augment domains. 
In contrast, \ourmodel leverages causal invariance, 
and has \emph{guarantee} to identify the semantic factor from a \emph{single training domain}, even with a \emph{correlation} to the variation factor.

\textbf{Disentangled latent representations} \hspace{4pt}
is also of interest in unsupervised learning.
Despite empirical success~\citep{chen2016infogan, higgins2017beta, chen2018isolating}, \citet{locatello2019challenging} conclude that it is impossible to guarantee the disentanglement in unsupervised settings.
Subsequent works then introduce ways of supervision like a few latent variable observations~\citep{locatello2019disentangling} or sample similarity~\citep{chen2020weakly, locatello2020weakly, shu2020weakly}.
Identifiable VAE~\citep{khemakhem2019variational} and extensions~\citep{khemakhem2020ice, yang2020causalvae} leverage the data of a cause variable of the latent variables and have established theoretical guarantees under a diversity condition.
But the works do not depict domain change thus not suitable for OOD prediction.
Instead of disentangling latent factors, we focus on identifying the semantic factor $s$ (Sec.~\ref{sec:thry-id}) and its benefit for OOD prediction.
Appx.~\ref{supp:relw} shows more related work.

\vspace{-1pt}
\section{The Causal Semantic Generative Model} \label{sec:model}
\vspace{-1pt}

\begin{figure}
  \centering
  \vspace{-12pt}
  \subcaptionbox{\ourmodel\label{fig:gen-sv}}{
    \includegraphics[width=.270\textwidth]{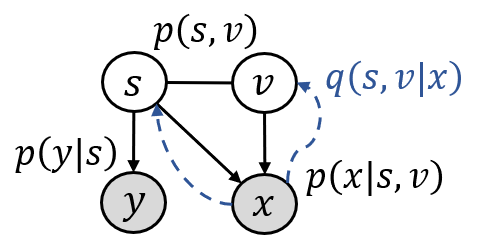}
  }
  \hspace{-10pt}
  \includegraphics[width=.175\textwidth]{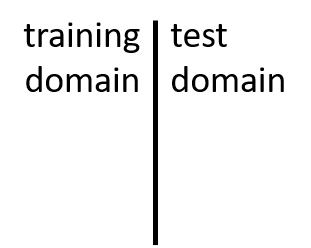}
  \hspace{-16pt}
  \subcaptionbox{\ourmodel{}-ind\label{fig:gen-ind}}{
    \includegraphics[width=.270\textwidth]{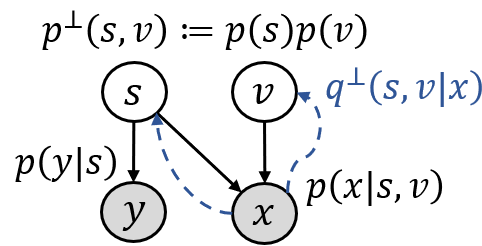}
  }
  \hspace{-6pt}
  \subcaptionbox{\ourmodel{}-DA\label{fig:gen-da}}{
    \includegraphics[width=.270\textwidth]{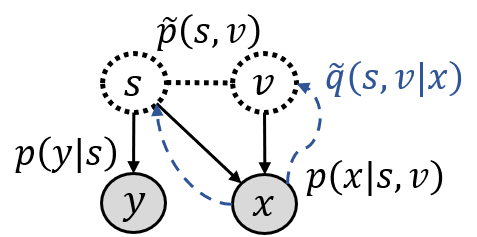}
  }
  \vspace{-4pt}
  \caption{\textbf{(a)} Graphical structure of the proposed \ourmodel. 
    Solid arrows represent causal mechanisms $p(x|s,v)$ and $p(y|s)$, the undirected $s$-$v$ clique represents a domain-specific prior $p(s,v)$, and the dashed bended arrows represent the inference model $q(s,v|x)$ for learning. 
    \textbf{(b, c)} Graphical structures of \ourmodel{}-ind and \ourmodel{}-DA for prediction on the \emph{test domain}.
    An independent prior $\pind(s,v)$ (constructed from $p(s,v)$) and a new prior $\ppt(s,v)$ (the dotted $s$-$v$ clique) are introduced reflecting the intervention on the test domain.
    Respective inference models $\qind(s,v|x)$ and $\qqt(s,v|x)$ are also shown.
    All three models share the same causal mechanisms $p(x|s,v)$ and $p(y|s)$.
  }
  \vspace{-20pt}
  \label{fig:gen}
\end{figure}

To develop the model soberly based on causality, we require its formal definition:
\emph{two variables have a causal relation, denoted as ``\emph{cause $\to$ effect}'', if intervening the \emph{cause} (by changing external variables out of the considered system) may change the \emph{effect}, but not \emph{vice versa}}~\citep{pearl2009causality, peters2017elements}.
We follow this definition to build our model (Fig.~\ref{fig:gen-sv}) by analyzing the example that an photographer takes a photo in a scene as $x$ and labels it as $y$.
Appx.~\ref{supp:model} provides more explanations under other perspectives.

\bfone 
It is likely that neither $y \to x$ \exmp{intervening the label with noise by distracting the photographer 
does not change the image}
nor $x \to y$ holds \exmp{intervening an image by 
breaking a camera sensor unit 
does not change how the photographer labels it},
as also argued in [\citealp{peters2017elements}, Sec.~1.4; \citealp{kilbertus2018generalization}].
So we introduce a latent variable $z$ to capture factors with causal relations.
Also for this reason, we need a generative model (vs. discriminative model that only learns $x \to y$).

\bftwo The latent variable $z$ as underlying generating factors \exmp{object shape and texture, background and illumination during imaging} is plausible to cause both 
$x$ \exmp{changing object shape or background makes a different image, but breaking the camera does not change the shape or background}
and $y$ \exmp{
the photographer would give a different label if the object shape had been different, 
but noise-corrupting the label does not change the shape}.
So we orient the edges in the generative direction $z \to (x,y)$, as also adopted in~\citep{mcauliffe2008supervised, peters2017elements, teshima2020few}.
This is in contrast to prior works~\citep{cai2019learning, ilse2020diva, ilse2020designing, castro2020causality} that treat $y$ as the cause of a semantic factor,
which, when $y$ is also a noisy observation, 
makes unreasonable implications \exmp{
adding noise to the labels in a dataset automatically changes object features and consequently the images,
and changing the object features does not change the label}.
This difference is also discussed in [\citealp{peters2017elements}, Sec.~1.4; \citealp{kilbertus2018generalization}].

\bfthr We attribute all $x$-$y$ relation to the existence of some latent factor [\citealp{lee2019leveraging}, ``purely common cause''; \citealp{janzing2009identifying}] and exclude $x$-$y$ edges.
This can be achieved as long as $z$ holds sufficient information of data
\exmp{with shape, background \etc fixed, breaking the camera does not change the label, and noise-corrupting the label does not change the image}.
Promoting this structure reduces arbitrariness in explaining $x$-$y$ relation thus 
helps identify (part of) $z$.
This is in contrast to prior works~\citep{kingma2014semi, zhang2020causal, castro2020causality} that treat $y$ as a cause of $x$ 
as no latent variable is introduced between.

\bffor Not all latent factors are the causes of $y$ \exmp{changing the shape may alter the label, while changing the background does not}.
We thus split the latent variable as $z = (s,v)$ and remove the $v \to y$ edge, where $s$ represents the \emph{semantic} factor that causes $y$, and $v$ describes the \emph{variation} or diversity in generating $x$. 
This formalizes the intuition on the concepts in Introduction (Sec.~\ref{sec:intr}).

\bffiv The two factors $s$ and $v$ often have a relation \exmp{a desk/bed shape tends to appear with a workspace/bedroom background}, 
but it is usually a spurious correlation \exmp{putting a desk in a bedroom does not automatically change the room as a workspace, nor does it turn the desk into a bed}. 
So we keep the undirected $s$-$v$ edge.
This is in contrast to prior works~\citep{cai2019learning, ilse2020diva, zhang2020causal, teshima2020few, mitrovic2021representation} which assume independent latent variables.
Although $v$ is not a cause of $y$, modeling it explicitly is worth the effort since otherwise it would still be implicitly mixed into $s$ anyway through the $s$-$v$ correlation. 
We summarize these conclusions in the following definition.
\vspace{-1pt}
\begin{definition}[\ourmodel]
  A \emph{Causal Semantic Generative Model} (\ourmodel), 
  $p := \lrangle{p(s,v), p(x|s,v), p(y|s)}$, is a generative model on data variables $x \in \clX \subseteq \bbR^{d_\clX}$ and $y \in \clY$ with semantic $s \in \clS \subseteq \bbR^{d_\clS}$ and variation $v \in \clV \subseteq \bbR^{d_\clV}$ latent variables, following the graphical structure shown in Fig.~\ref{fig:gen-sv}.
  \label{def:model}
\end{definition}
\vspace{-1pt}

\vspace{-1pt}
\subsection{The Causal Invariance Principle} \label{sec:model-cauinv}
\vspace{-1pt}

Through the above process, we see that the $s$-$v$ correlation embodied in the prior $p(s,v)$ tends to change across domains.
Under a causal view, this means that the domain change comes from a (soft) intervention on $s$ or $v$ or both, leading to a different prior.
On the other hand, the generative processes are likely causal mechanisms, so they enjoy the celebrated Independent Causal Mechanisms principle~\citep{peters2017elements, scholkopf2021toward} indicating that they are unaffected under the intervention on prior.
This leads to the following causal invariance principle for \ourmodel.
\vspace{-1pt}
\begin{principle}[causal invariance] \label{prin:inv}
  The causal generative mechanisms $p(x|s,v)$ and $p(y|s)$ in \ourmodel are invariant across domains,
  and the change of prior $p(s,v)$ is the only source of domain change.
\end{principle}
\vspace{-1pt}
This invariance reflects the universality of basic laws of nature and is considered in some prior works~\citep{scholkopf2012causal, peters2017elements, besserve2018group, buhlmann2018invariance}. 
Other works instead introduce domain index~\citep{cai2019learning, ilse2020diva, ilse2020designing, castro2020causality} or manipulation variables~\citep{zhang2020causal, khemakhem2019variational, khemakhem2020ice} to model distribution change explicitly.
They then require multiple training domains or additional observations, while such changes can also be explained under causal invariance as long as the latent variables include all changing factors.

\vspace{-1pt}
\subsection{Comparison with Inference Invariance} \label{sec:model-infinv}
\vspace{0pt}

\begin{wrapfigure}{r}{.380\textwidth}
  \centering
  \vspace{-35pt}
  \includegraphics[width=.110\textwidth]{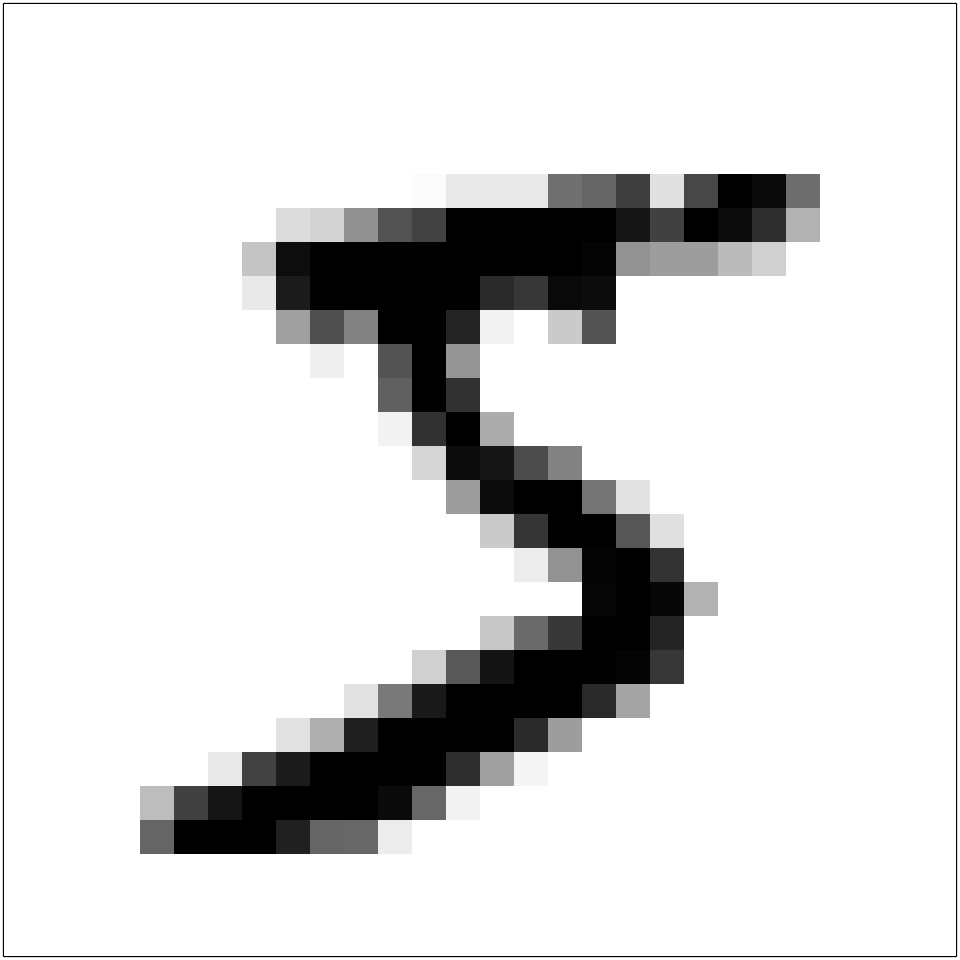}
  \hspace{8pt}
  \includegraphics[width=.125\textwidth]{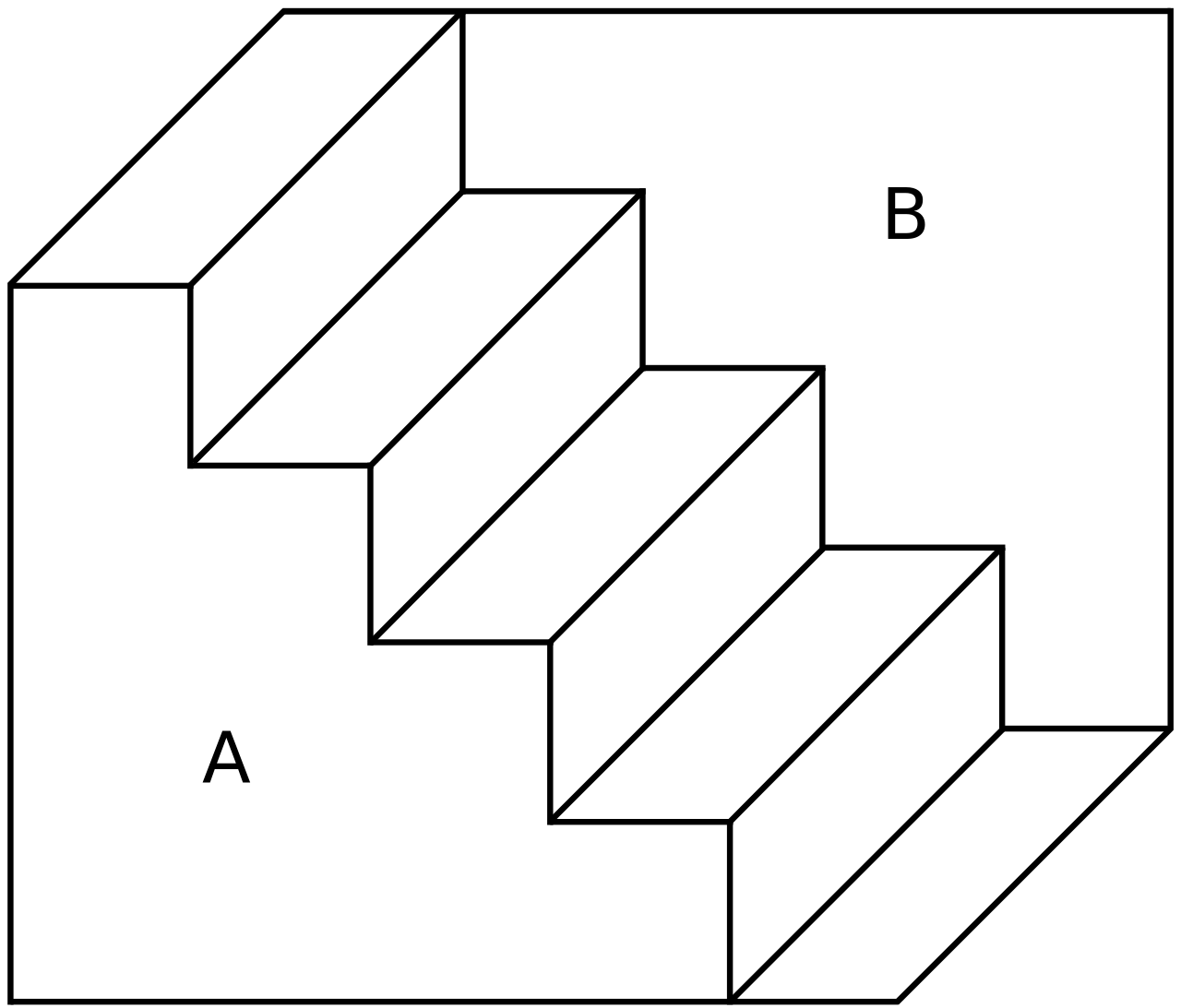}
  \vspace{-2pt}
  \caption{Examples of noisy (left) or degenerate (right) generating mechanisms that lead to ambiguity in inference.
    Left: handwritten digit that may be generated as either ``3'' or ``5''.
    Right: Schr{\"o}der's stairs that may be generated with either A or B being the nearer surface.
    Inference results notably rely on the prior on the digits/surfaces, which is domain-specific.
  }
  \vspace{-10pt}
  \label{fig:ambiguity}
\end{wrapfigure}

Most domain adaptation and generalization methods (incl. domain-invariant-representation based~\citep{ganin2016domain, shankar2018generalizing}, invariant-latent-predictor based~\citep{arjovsky2019invariant, krueger2020out, mitrovic2021representation}) 
use a shared representation extractor across domains.
This effectively assumes the invariance in the other direction, \ie inferring latent factors $z$ from observed data $x$.
We note in its supportive examples \exmp{inferring object position from image, extracting the fundamental frequency from audio},
the causal mechanism $p(x|z)$ is nearly deterministic and invertible such that it preserves the information of $z$. 
Formally, for a given $x$, only one single $z$ value achieves a positive $p(x|z)$ while all other values lead to zero. 
The inferred representation given by the posterior via the Bayes rule $p(z|x) \propto p(z) p(x|z)$ then concentrates on this $z$ value, which is determined by the causal mechanism $p(x|z)$ alone, regardless of the domain-specific prior $p(z)$.
Causal invariance then implies inference invariance. 

In more general cases, the causal mechanism may be noisy or degenerate (Fig.~\ref{fig:ambiguity}), such that there are multiple $z$ values that give a positive $p(x|z)$, \ie they all could generate the same $x$.
Inference is then ambiguous, and the posterior relies on the prior to choose from these $z$ values.
Since the prior changes across domains \exmp{different labelers have different mindset}, the inference rule then \emph{changes by nature} and is not invariant,
\footnote{
  Particularly, although \citet{mitrovic2021representation} consider a similar causal structure and promote the invariance of $p(y|s)$,
  $s$ actually depends on $v$ for a given $x$, 
  even when they are independent in the prior.
  So $p(s|x)$ must depend on the domain-specific $p(v)$, and a domain-invariant representation extractor does not exist.
}
while the causal invariance is rather more fundamental and reliable.
To leverage causal invariance, we 
use a different prior for the test domain (\ourmodel{}-ind and \ourmodel{}-DA), which gives
a different and more reliable prediction than following inference invariance. 

\vspace{-1pt}
\section{Method} \label{sec:meth}
\vspace{-1pt}

We now develop 
methods based on variational Bayes~\citep{jordan1999introduction, kingma2014auto} for OOD generalization and domain adaptation using \ourmodel. 
Appx.~\ref{supp:meth-obj} shows all details. 

\vspace{-1pt}
\subsection{Method for OOD Generalization} \label{sec:meth-ood}
\vspace{-1pt}

For OOD generalization, one only has supervised data from the underlying data distribution $p^*(x,y)$ on the \emph{training domain}.
Fitting a \ourmodel $p := \lrangle{p(s,v), p(x|s,v), p(y|s)}$ to data by maximizing likelihood $\bbE_{p^*(x,y)}[\log p(x,y)]$ is intractable, since $p(x,y) := \int p(s,v,x,y) \dd s \ud v$ where $p(s,v,x,y) := p(s,v) p(x|s,v) p(y|s)$, is hard to estimate.
The Evidence Lower BOund (ELBO) $\clL_{p, \, q_{s,v|x,y}} (x,y) := \bbE_{q(s,v|x,y)} [ \log \frac{p(s,v,x,y)}{q(s,v|x,y)} ]$~\citep{jordan1999introduction, wainwright2008graphical}
is a tractable surrogate with the help of 
an inference model $q(s,v|x,y)$ that enjoys easy sampling and density evaluation.
It is known that $\max_{q_{s,v|x,y}} \clL_{p, \, q_{s,v|x,y}} (x,y)$ drives $q(s,v|x,y)$ towards the posterior $p(s,v|x,y) := \frac{p(s,v,x,y)}{p(x,y)}$,
meanwhile makes $\clL_{p, \, q_{s,v|x,y}} (x,y)$ a tighter lower bound of $\log p(x,y)$ for optimizing \ourmodel $p$.

However, the subtlety with supervised learning is that prediction is still hard, as the introduced model $q(s,v|x,y)$ does not help estimate $p(y|x)$.
To address this, we propose to employ an auxiliary model $q(s,v,y|x)$ targeting $p(s,v,y|x)$.
It allows easy sampling of $y$ given $x$ for prediction, 
and can also serve as the required inference model:
$q(s,v|x,y) = \frac{q(s,v,y|x)}{q(y|x)}$, where $q(y|x) := \int q(s,v,y|x) \dd s \ud v$ is also determined by $q(s,v,y|x)$.
The ELBO objective $\bbE_{p^*(x,y)} [\clL_{p, \, q_{s,v|x,y}} (x,y)]$ then becomes:
{\abovedisplayskip=3pt
\begin{align}
  \bbE_{p^*\!(x)} \bbE_{p^*\!(y|x)} [\log q(y|x)]
  + \bbE_{p^*\!(x)} \bbE_{q(s,v,y|x)} \!\Big[ \frac{p^*\!(y|x)}{q(y|x)} \!\log \frac{p(s,v,x,y)}{q(s,v,y|x)} \Big]. \!\!
  \label{eqn:elbo-interp}
\end{align} }%
As a functional of $q(s,v,y|x)$ (instead of $q(s,v|x,y)$) and the \ourmodel $p$, this objective also drives them towards their targets:
the first term is the negative of the standard cross entropy (CE) loss which drives $q(y|x)$ towards $p^*(y|x)$,
and once this is achieved, the second term becomes the expected ELBO $\bbE_{p^*(x)}[\clL_{p, \, q_{s,v,y|x}} (x)]$ 
that drives $q(s,v,y|x)$ towards $p(s,v,y|x)$ and $p(x)$ towards $p^*(x)$.
%
Furthermore, as the target of $q(s,v,y|x)$ factorizes as $p(s,v,y|x) = p(s,v|x) p(y|s)$ (due to Fig.~\ref{fig:gen-sv}) 
where $p(y|s)$ is already known (part of the \ourmodel), 
we can instead employ a lighter inference model $q(s,v|x)$ for the minimally intractable component $p(s,v|x)$ therein, and use $q(s,v|x) p(y|s)$ as $q(s,v,y|x)$.
This turns the objective \eqref{eqn:elbo-interp} to:
{\abovedisplayskip=2pt
\begin{align}
  \max_{p, \, q_{s,v|x}} \bbE_{p^*\!(x,y)} \Big[
    \log q(y|x)
    + \frac{1}{q(y|x)} \bbE_{q(s,v|x)} \! \Big[ p(y|s) \log \frac{p(s,v) p(x|s,v)}{q(s,v|x)} \Big] \Big], \!
  \label{eqn:elbo-src}
\end{align} }%
where $q(y|x) := \bbE_{q(s,v|x)} [p(y|s)]$.
The expectations can be estimated by Monte Carlo after applying the reparameterization trick~\citep{kingma2014auto}.
This is the basic \ourmodel method.

\textbf{\ourmodel{}-ind} \hspace{8pt}
To actively improve OOD generalization performance, we consider using an \textbf{ind}ependent prior $\pind(s,v) := p(s) p(v)$ for prediction in the \emph{test domain} (Fig.~\ref{fig:gen-ind}), where $p(s)$ and $p(v)$ are the marginals of the training-domain prior $p(s,v)$.
Intuitively, $\pind(s,v)$ discards the spurious correlation between $s$ and $v$ on the training domain \exmp{the ``desk-workspace'', ``bed-bedroom'' association},
and promotes a cautious neutral belief on the unknown test-domain correlation 
in defence against all possibilities \exmp{a ``desk-bedroom'', ``bed-workspace'' association}.
Formally, $\pind(s,v)$ has a larger entropy than $p(s,v)$~\citep[Thm.~2.6.6]{cover2006elements}, so it reduces training-domain-specific information and encourages reliance on the causal mechanisms for better generalization.
It also amounts to applying the do-operator~\citep{pearl2009causality} to Fig.~\ref{fig:gen-sv}, representing a randomized experiment by independently soft-intervening $s$ or $v$.
In this way, causal invariance is properly leveraged, making a different and more reliable prediction than following inference invariance.
Our theory below also shows that $\pind(s,v)$ leads to a smaller generalization error bound (Thm.~\ref{thm:ood} Remark).

Methodologically, we need the test-domain inference model $\qind(s,v|x)$ for prediction $\pind(y|x) \approx \bbE_{\qind(s,v|x)}[p(y|s)]$,
but also need $q(s,v|x)$ for learning on the training domain.
To save the cost of building and learning two inference models, we propose to use $\qind(s,v|x)$ to represent $q(s,v|x)$.
Noting that their targets are related by $p(s,v|x) = \frac{p(s,v)}{\pind(s,v)} \frac{\pind(x)}{p(x)} \pind(s,v|x)$,
we formulate $q(s,v|x) = \frac{p(s,v)}{\pind(s,v)} \frac{\pind(x)}{p(x)} \qind(s,v|x)$ accordingly, so that this $q(s,v|x)$ achieves its target if and only if $\qind(s,v|x)$ does.
The objective \eqref{eqn:elbo-interp} then becomes:
{\abovedisplayskip=2pt
\belowdisplayskip=3pt
\begin{align}
  \max_{p, \, \qind_{s,v|x}} \bbE_{p^*\!(x,y)} \Big[
    \log \pi(y|x) + \frac{1}{\pi(y|x)} \bbE_{\qind(s,v|x)} \!
    \Big[ \frac{p(s,v)}{\pind(s,v)} p(y|s) \log \frac{\pind(s,v) p(x|s,v)}{\qind(s,v|x)} \Big] \Big],
  \label{eqn:elbo-src-ind}
\end{align} }%
where $\pi(y|x) := \bbE_{\qind(s,v|x)} \! \big[ \frac{p(s,v)}{\pind(s,v)} p(y|s) \big]$.
(Note $\pind(s,v)$ is determined by $p(s,v)$ in the \ourmodel $p$.)

\vspace{-1pt}
\subsection{Method for Domain Adaptation} \label{sec:meth-da}
\vspace{-1pt}

In domain adaptation, one also has unsupervised data from the underlying data distribution $\ppt^*(x)$ on the \emph{test domain}.
We can leverage them 
for better prediction.
According to the causal invariance principle~(\ref{prin:inv}), 
we only need a new prior $\ppt(s,v)$ for the test-domain \ourmodel $\ppt := \lrangle{\ppt(s,v), p(x|s,v), p(y|s)}$ (Fig.~\ref{fig:gen-da}).
Fitting test-domain data can be done through the standard ELBO objective with the test-domain inference model $\qqt(s,v|x)$:
{\abovedisplayskip=0pt
\belowdisplayskip=2pt
\begin{align}
  \max_{\ppt, \, \qqt_{s,v|x}} \bbE_{\ppt^*(x)}  [\clL_{\ppt, \, \qqt_{s,v|x}} (x)],
  \text{where } \clL_{\ppt, \, \qqt_{s,v|x}} \!(x)
  = \bbE_{\qqt(s,v|x)} \! \Big[ \! \log \frac{\ppt(s,v) p(x|s,v)}{\qqt(s,v|x)} \Big].
  \label{eqn:elbo-tgt}
\end{align} }%
Prediction is given by $\ppt(y|x) \approx \bbE_{\qqt(s,v|x)}[p(y|s)]$. 
Similar to the \ourmodel{}-ind case, we still need $q(s,v|x)$ for fitting training-domain data, and we can also avoid a separate $q(s,v|x)$ model by representing it using $\qqt(s,v|x)$.
Following the same relation between their targets, we let $q(s,v|x) = \frac{\ppt(x)}{p(x)} \frac{p(s,v)}{\ppt(s,v)} \qqt(s,v|x)$,
which reformulates the same training-domain objective \eqref{eqn:elbo-interp} as:
{\abovedisplayskip=2pt
\belowdisplayskip=3pt
\begin{align}
  \max_{p, \, \qqt_{s,v|x}} \bbE_{p^*\!(x,y)} \Big[
    \log \pi(y|x) + \frac{1}{\pi(y|x)} \bbE_{\qqt(s,v|x)}
    \Big[ \frac{p(s,v)}{\ppt(s,v)} p(y|s) \log \frac{\ppt(s,v) p(x|s,v)}{\qqt(s,v|x)} \Big] \Big],
  \label{eqn:elbo-src-qt}
\end{align} }%
where $\pi(y|x) := \bbE_{\qqt(s,v|x)} \big[ \frac{p(s,v)}{\ppt(s,v)} p(y|s) \big]$.
The resulting method, termed \ourmodel{}-DA, solves both optimization problems Eqs.~(\ref{eqn:elbo-tgt},~\ref{eqn:elbo-src-qt}) 
simultaneously.

\vspace{-1pt}
\subsection{Implementation and Model Selection} \label{sec:meth-impl}
\vspace{-1pt}

To implement the three \ourmodel methods, we only need one inference model in each.
Appx.~\ref{supp:meth-instant} shows its construction from a general discriminative model (\eg, how to select its hidden nodes as $s$ and $v$).
In practice $x$ often has a much larger dimension than $y$, making the first supervision term overwhelmed by the second unsupervised term in Eqs.~(\ref{eqn:elbo-src},\ref{eqn:elbo-src-ind},\ref{eqn:elbo-src-qt}).
So we downscale the second term.

As recently emphasized~\citep{gulrajani2020search}, an OOD method should include a model selection method, since it is nontrivial and significantly affects performance~\citep{rothenhausler2018anchor, you2019towards}.
For our methods, we use a validation set from the \emph{training domain} for model selection.
This complies with the OOD setup, and is also suggested by our theory below which gives guarantees based on a good fit to the training-domain data distribution.
For \ourmodel{}-ind/DA, the learned predictor targets the \emph{test} domain, so we \emph{do not} use it directly for evaluating validation accuracy, but by normalizing $\pi(y|x)$.
Appx.~\ref{supp:meth-mselect} shows details.

\vspace{-1pt}
\section{Theory} \label{sec:thry}
\vspace{-1pt}

We now establish theory for the identification of the semantic factor (cause of prediction) and subsequent merits for OOD generalization and domain adaptation.
We focus on the distribution-level generalization instead of from finite samples to unseen samples under the same distribution, 
so we only consider the infinite-data regime. 
Appx.~\ref{supp:proofs} shows all the proofs and auxiliary theory.

Latent variable identification is hard~\citep{koopmans1950identification, murphy2012machine, yacoby2019learning, locatello2019challenging}
as it is beyond observational relations~\citep{janzing2009identifying, peters2017elements}.
Assumptions are thus required to draw definite conclusions.
\vspace{-1pt}
\begin{assumption} \label{assm:anm-bij} 
  (\textbf{Additive noise})
  There exist nonlinear functions $f$ and $g$ with bounded derivatives up to the third-order, and independent random variables $\mu$ and $\nu$,
  such that $p(x|s,v) = p_\mu(x - f(s,v))$, and $p(y|s) = p_\nu(y - g(s))$ for continuous $y$ or 
  $p(y|s) = \Cat(y|g(s))$ for categorical $y$. \\
  (\textbf{Bijectivity}) Assume $f$ is bijective and $g$ is injective.
\end{assumption}
\vspace{-2pt}
The additive noise assumption is widely adopted in causal discovery~\citep{janzing2009identifying, buhlmann2014cam}.
It disables expressing the same joint in the other direction 
[\citealp[Thm.~8]{zhang2009identifiability}; \citealp[Prop.~23]{peters2014causal}]
so that \ourmodel unnecessarily indicates inference invariance. 
For this reason, we exclude GAN~\citep{goodfellow2014generative} and flow-based~\citep{kingma2018glow} implementations.
%
Bijectivity is a common assumption for identifiability~\citep{janzing2009identifying, shalit2017estimating, khemakhem2019variational, lee2019leveraging}.
It is sufficient [\citealp[Prop.~17]{peters2014causal}; \citealp[Prop.~7.4]{peters2017elements}] for the more fundamental [\citealp[Prop.~7]{peters2014causal}; \citealp[p.109]{peters2017elements}] requirement of causal minimality [\citealp[p.2012]{peters2014causal}; \citealp[Def.~6.33]{peters2017elements}].
Particularly, $s$ and $v$ may otherwise have dummy dimensions that $f$ and $g$ simply ignore, raising another ambiguity against identifiability.
On the other hand, according to the commonly acknowledged manifold hypothesis~\citep{weinberger2006unsupervised, fefferman2016testing}, 
we can take $\clX$ as the lower-dimensional data manifold and such a bijection exists as a coordinate map, which is an injection to the original data space and also allows $d_\clS + d_\clV < d_\clX$.

\vspace{-1pt}
\subsection{Identifiability Theory} \label{sec:thry-id}
\vspace{-1pt}

We first formalize the goal of identifying the semantic factor.
\vspace{-1pt}
\begin{definition}[semantic-identification] \label{def:id}
  We say a learned \ourmodel $p$ is \emph{semantic-identified}, if there exists a homeomorphism\footnote{
    A transformation is a homeomorphism if it is a continuous bijection with continuous inverse.
  } $\Phi$ on $\clS\times\clV$,
  such that \bfi its output dimensions in $\clS$ is constant of $v$: 
  $\Phi^\clS(s,v) = \Phi^\clS(s,v'), \forall v, v' \in \clV$ (hence denote $\Phi^\clS(s,v)$ as $\Phi^\clS(s)$),
  and \bfii it is a \emph{reparameterization} of the ground-truth \ourmodel $p^*$: $\Phi_\#[p^*_{s,v}] = p_{s,v}$, $p^*(x|s,v) = p(x|\Phi(s,v))$ and $p^*(y|s) = p(y|\Phi^\clS(s))$. 
\end{definition}
\vspace{-1pt}

Here, $\Phi_\#[p^*_{s,v}]$ denotes the pushed-forward distribution\footnote{
  The definition of $\Phi_\#[p^*_{s,v}]$ requires $\Phi$ to be measurable.
  This is satisfied by the continuity of $\Phi$ as a homeomorphism (as long as the Borel $\sigma$-field is considered)~\citep[Thm.~13.2]{billingsley2012probability}.
} of $p^*_{s,v}$ by $\Phi$, \ie the distribution of $\Phi(s,v)$ when $(s,v) \sim p^*_{s,v}$.
As the ground-truth \ourmodel could at most provide its information via the data distribution $p^*(x,y)$, a well-learned \ourmodel that achieves $p(x,y) = p^*(x,y)$ still has the degree of freedom in parameterizing $(s,v)$.
This is described by this reparameterization $\Phi$ (Appx. Lemma~\ref{lem:repar-same-pxy}).
At the heart of the definition, the $v$-constancy of $\Phi^\clS$ implies that $\Phi$ is \emph{semantic-preserving}:
the learned model \emph{does not mix} the ground-truth $v$ into its $s$, so that the learned $s$ holds equivalent information to the ground-truth $s$.
The definition can thus be seen as the semantic equivalence (Appx. Def.~\ref{def:equiv}, Prop.~\ref{prop:equiv}) to the ground-truth \ourmodel $p^*$. 

For related concepts, this identification cannot be characterized by the \emph{statistical independence} between $s$ and $v$ (vs.~\citep{cai2019learning, ilse2020diva, zhang2020causal}), which is not sufficient~\citep{locatello2019challenging} nor necessary (due to the existence of spurious correlation).
It is also weaker than \emph{disentanglement}~\citep{higgins2018towards, besserve2020counterfactuals}, which 
additionally requires the learned $v$ to be constant of the ground-truth $s$.
%
%
The following theorem shows that semantic-identification can be achieved on a single domain under certain conditions.
\vspace{0pt}
\begin{theorem}[semantic-identifiability] \label{thm:id}
  With Assumption~\ref{assm:anm-bij},
  a \ourmodel $p$ is semantic-identified,
  if it is well-learned such that $p(x,y) = p^*(x,y)$,
  under the conditions that $\log p(s,v)$ and $\log p^*(s,v)$ are bounded up to the second-order, and that\footnote{
    To be precise, the conclusions are that the equalities in Def.~\ref{def:id} hold asymptotically in the limit $1/\sigma_\mu^2 \to \infty$ for condition \bfi, and hold a.e. for condition \bfii.
  }
  \bfemi $1/\sigma_\mu^2 \to \infty$ where $\sigma_\mu^2 := \bbE[\mu\trs \mu]$, \textbf{or}
  \bfemii $p_\mu$ (\eg, a Gaussian) has an a.e. non-zero characteristic function.
\end{theorem}
\vspace{0pt}
\textbf{Remarks.}
\bfone (\textbf{Condition and Intuition})
Compared with the multi-domain case~\citep{peters2016causal, rojas2018invariant, arjovsky2019invariant}, 
identifiability on a single training domain comes at a cost and requires certain conditions. 
One may imagine that in some extreme cases \eg, all desks appear in workspace and all beds in bedrooms, it is impossible to distinguish whether $y$ labels the object or the background (unlearnable OOD problem~\citep{ye2021towards}).
The theorem finds an \emph{appropriate condition} that excludes such cases:
when $\log p^*(s,v)$ is bounded, deterministic $s$-$v$ relations are not allowed as they concentrate $p^*(s,v)$ on a lower-dimensional subspace in $\clS\times\clV$ thus make it unbounded.

It also leads to the \emph{intuition of identifiability}: a bounded $\log p^*(s,v)$ indicates a stochastic $s$-$v$ relation,
so mixing the ground-truth $v$ into the learned $s$ makes the inference of $s$ more noisy due to the intrinsic diversity/uncertainty of this $v$. 
As prediction is made via the inferred $s$, this worsens prediction accuracy thus violates the ``well-learned'' requirement.
Compared with discriminative models, \ourmodel makes more faithful inference, and its causal structure leads to a proper description of domain change.

\bftwo In condition \bfi, $1/\sigma_\mu^2$ measures the \emph{intensity} of the causal mechanism $p(x|s,v)$. 
When it is large, the ``strong'' $p(x|s,v)$ helps disambiguating values of $(s,v)$ in generating a given $x$. 
The formal version in Appx. Thm.~\ref{thm:id-formal} shows a quantitative reference for large enough intensity, 
and Appx.~\ref{supp:id-delta} gives a non-asymptotic extension showing how the intensity trades-off the tolerance of equalities in Def.~\ref{def:id}.
Condition \bfii goes beyond inference invariance. 
It roughly implies that different $(s,v)$ values a.s. produce different $p(x|s,v)$, so their roles in generating $x$ become clear which helps identification.

\bfthr The theorem does not contradict the impossibility result by \citet{locatello2019challenging}, 
which considers disentangling each latent dimension with an unconstrained $(s,v) \to (x,y)$, while we only identify $s$ as a whole, with the $v \to y$ edge removed which breaks the $s$-$v$ symmetry.

\vspace{-2pt}
\subsection{OOD Generalization Theory} \label{sec:thry-ood}
\vspace{-2pt}

Now we show the benefit of semantic-identification for OOD generalization that the prediction error is bounded.
Note the optimal predictor $\bbEt^*[y|x]$
\footnote{For categorical $y$, the expectation of $y$ is taken under the one-hot representation.
} on the test domain is defined by the corresponding ground-truth \ourmodel $\ppt^*$, 
which differs from $p^*$ only in the test-domain prior $\ppt^*(s,v)$ (Principle~\ref{prin:inv}).
\vspace{0pt}
\begin{theorem}[OOD generalization error] \label{thm:ood}
  \hspace{-4pt} \footnote{See Appx. Thm.~\ref{thm:ood-formal} for the formal version.}
  With Assumption~\ref{assm:anm-bij},
  for a semantic-identified \ourmodel $p$ on the training domain with semantic-preserving reparameterization $\Phi$,
  we have up to $O(\sigma_\mu^4)$, 
  {\abovedisplayskip=2pt
  \begin{align} \label{eqn:ood}
    \bbE_{\ppt^*(x)} \lrVert*{\bbE[y|x] - \bbEt^*[y|x]}_2^2
    \le \sigma_\mu^4 B'^4_{f^{-1}} B'^2_g \, \bbE_{\ppt_{s,v}} \lrVert{\nabla \log (\ppt_{s,v} / p_{s,v})}_2^2,
  \end{align} }%
  where $B'_{f^{-1}}$ and $B'_g$ bound the 2-norms\footnote{
    As the induced operator norm for matrices (not the Frobenius norm).
  } of the Jacobians of $f^{-1}$ and $g$, respectively,
  and $\ppt_{s,v} := \Phi_\#[\ppt^*_{s,v}]$ is the test-domain prior under the parameterization of the \ourmodel $p$.
\end{theorem}
\vspace{-2pt}
In the bound, the term $\bbE_{\ppt_{s,v}} \lrVert{\nabla \log (\ppt_{s,v} / p_{s,v})}_2^2$ is the Fisher divergence measuring the difference between the two priors.
As the prior change is the only source of domain change, this term also measures the ``OODness'' 
in terms of the effect on prediction.
The bound also shows that when the causal mechanism $p(x|s,v)$ is strong (small $\sigma_\mu$), 
it dominates prediction over the prior change, as the generalization error becomes small.
Compared with other methods, using a \ourmodel enforces causal invariance, so the boundedness of OOD generalization error becomes more plausible in practice. 

\textbf{Remark.} \hspace{4pt}
The bound also shows the advantage of \ourmodel{}-ind (Sec.~\ref{sec:meth-ood}). 
The Fisher divergence is revealed~\citep{durkan2021maximum} to have a similar behavior as the forward KL divergence $p_{s,v} \mapsto \KL(\ppt_{s,v} \Vert p_{s,v})$ that it is very sensitive to the insufficient coverage of $p_{s,v}$ on the support of $\ppt_{s,v}$~\citep{huszar2015not, theis2016note},
since $\log (\ppt_{s,v} / p_{s,v})$ is infinitely large on the uncovered region. 
As the independent prior $\pind_{s,v}$ has a larger support than $p_{s,v}$, 
it is less likely to miss the support of $\ppt_{s,v}$,
so it induces a generally smaller Fisher divergence. 
\ourmodel{}-ind thus generally has a smaller OOD generalization error bound than \ourmodel.

\vspace{-2pt}
\subsection{Domain Adaptation Theory} \label{sec:thry-da}
\vspace{-2pt}

\ourmodel{}-DA (Sec.~\ref{sec:meth-da}) learns a new prior $\ppt_{s,v}$ by fitting unsupervised test-domain data, with causal mechanisms shared.
If the mechanisms are semantic-identified, the ground-truth test-domain prior $\ppt^*_{s,v}$ can also be identified under the learned parameterization, and prediction is made precise.
\vspace{0pt}
\begin{theorem}[domain adaptation error] \label{thm:da}
  With conditions of Thm.~\ref{thm:id}, for a semantic-identified \ourmodel $p$ on the training domain with semantic-preserving reparameterization $\Phi$,
  if its new prior $\ppt_{s,v}$ is well-learned such that $\ppt(x) = \ppt^*(x)$,
  then $\ppt_{s,v} = \Phi_\#[\ppt^*_{s,v}]$, and
  $\bbEt[y|x] = \bbEt^*[y|x]$ for any $x \in \supp(\ppt^*_x)$.
\end{theorem}
\vspace{-2pt}
Different from existing domain adaptation bounds (Appx.~\ref{supp:da-dir}), Theorems~\ref{thm:ood},\ref{thm:da} allow different inference models in the two domains, thus go beyond inference invariance.

\vspace{-2pt}
\section{Experiments} \label{sec:expm}
\vspace{-2pt}

For OOD generalization baselines, there is not much choice beyond the standard CE loss optimization, 
as domain adaptation methods require test-domain data and most domain generalization methods degenerate to CE with one training domain.
The exception within our scope is a causal discriminative method CNBB~\citep{he2019towards}. 
For domain adaptation, we consider well-acknowledged methods DANN~\citep{ganin2016domain}, DAN~\citep{long2015learning}, CDAN~\citep{long2018conditional} and recent compelling methods MDD~\citep{zhang2019bridging} and BNM~\citep{cui2020towards} (shown in Appx. Tables~\ref{tab:res-oodgen},\ref{tab:res-da}). 
Appx.~\ref{supp:expm} shows more details, results, and discussions. 
\footnote{Codes are available at \url{https://github.com/changliu00/causal-semantic-generative-model}.}



\textbf{Shifted-MNIST.} \hspace{4pt}
We first consider an OOD prediction task on MNIST to classify digits ``0''s and ``1''s. 
To make a spurious correlation, in the training data, we horizontally shift each ``0'' at random by $\delta_0 \sim \clN(-5, 1^2)$ pixels, while each ``1'' by $\delta_1 \sim \clN(5, 1^2)$ pixels.
We consider two test domains with different digit-position distributions:
each digit is not moved $\delta_0 = \delta_1 = 0$ in the first, and is shifted at random by $\delta_0, \delta_1 \sim \clN(0, 2^2)$ pixels in the second.
We implement all methods using a multilayer perceptron which is not naturally shift invariant.
We use a larger architecture for non-generative methods 
to compensate the additional generative component of generative methods.

The performance is shown in Table~\ref{tab:res-sum}(top 2 rows).
For OOD generalization, CE is misled by the more noticeable position factor due to the spurious correlation to digits, and resorts to random guess (even worse) when position is not informative for prediction.
CNBB ameliorates the position confusion, but not as thoroughly without modeling causal mechanisms. 
In contrast, our \ourmodel gives more genuine predictions in unseen domains, thanks to the identification of the semantic factor.
\ourmodel{}-ind performs even better, justifying the merit of using an independent prior for prediction.
For domain adaptation, \ourmodel{}-DA achieves the best results.
Existing adaptation methods even worsen the result (negative transfer), as the misleading position representation gets strengthened on the unsupervised test data. 
\ourmodel is benefited from adaptation in a proper way that identifies the semantic factor.

\textbf{ImageCLEF-DA} \hspace{2pt}
is a standard benchmark for domain adaptation~\citep{imageclef2014}. 
It has 12 classes and three domains of real-world images: \textbf{C}altech-256, \textbf{I}mageNet, \textbf{P}ascal~VOC~2012. 
We select four OOD prediction tasks \textbf{C}$\leftrightarrow$\textbf{P}, \textbf{I}$\leftrightarrow$\textbf{P} that have not seen good enough results. 
We adopt the same setup as~\citep{long2018conditional}. 
%
As shown in Table~\ref{tab:res-sum}(middle 4 rows), \ourmodel{}-ind again achieves the best OOD generalization results, and even outperforms some domain adaptation methods.
Our \ourmodel also outperforms the baselines mostly.
For domain adaptation, \ourmodel{}-DA is the best in most cases and on par with the best in others. 

\textbf{PACS} \hspace{4pt}
is a more recent benchmark dataset~\citep{li2017deeper}.
It has 7 classes and is named after its four domains: \textbf{P}hoto, \textbf{A}rt, \textbf{C}artoon, \textbf{S}ketch;
each contains images of a certain style.
We follow the same setup as~\citep{gulrajani2020search}; particularly, we pool together all domains but the test one as the single training domain.
%
Results in Table~\ref{tab:res-sum}(bottom 4 rows) show the same trend.
\ourmodel{}-DA even outperforms most domain generalization methods reported in~\citep{gulrajani2020search}, which are fed with more information.
Appx. Tables~\ref{tab:res-oodgen},\ref{tab:res-da} also show the results on an even larger dataset \textbf{VLCS}~\citep{fang2013unbiased}, which present a similar observation. 

\textbf{Visualization.} \hspace{4pt}
Appx. Fig.~\ref{fig:viz} visualizes the learned models using LIME~\citep{ribeiro2016why}.
The results show our methods focus more on the semantic regions and shapes, indicating a causal representation is learned.

\textbf{Dataset analysis.} \hspace{4pt}
The results indicate our methods are more powerful on shifted-MNIST and PACS (and VLCS) than ImageCLEF-DA.
This meets the intuition of identifiability (Thm.~\ref{thm:id} Remark~\bfone):
the random position or pooled training domain shows a diverse $v$ for each $s$ (while with a misleading spurious correlation), 
so identification is better guaranteed to overcome the spurious correlation. 

\textbf{Ablation study.} \hspace{4pt}
To show the benefit of modeling $s$ and $v$ separately, we compare with a counterpart of \ourmodel that treats $s$ and $v$ as a whole (equivalently, $v \to y$ is kept; see Appx.~\ref{supp:meth-obj-svae} for method details). 
Appx. Tables~\ref{tab:res-oodgen},\ref{tab:res-da} show that our methods outperform this baseline in all cases.
This shows the separate modeling 
makes \ourmodel consciously drive semantic representation into the dedicated variable $s$.

\begin{table*}[t]
  \vspace{-4pt}
  \centering
  \setlength{\tabcolsep}{2.0pt}
  \caption{Test accuracy (\%) by various methods (ours in bold) for OOD generalization (left 4 cols) and domain adaptation (right 5 cols) on Shifted-MNIST (top 2 rows), ImageCLEF-DA (middle 4 rows) and PACS (bottom 4 rows) datasets.
    Averaged over 10 runs.
    Appx. Tables~\ref{tab:res-oodgen},\ref{tab:res-da} show more results.
  }
  \label{tab:res-sum}
  \vspace{-6pt}
  \small
  \begin{tabular}{c@{}||cccc||ccccc}
    \toprule
    task &
    CE & CNBB & \textbf{\ourmodel} & \textbf{\ourmodel{}-ind} &
    DANN & DAN & CDAN & MDD & \textbf{\ourmodel{}-DA} \\
    \midrule
    $\delta_0 = \delta_1 = 0$ &
    42.9\subpm{3.1} & 54.7\subpm{3.3} & 81.4\subpm{7.4} & \textbf{82.6\subpm{4.0}} &
    40.9\subpm{3.0} & 40.4\subpm{2.0} & 41.0\subpm{0.5} & 41.9\subpm{0.8} & \textbf{97.6\subpm{4.0}} \\
    $\delta_0, \! \delta_1 \! \sim \! \clN \! (0, \! 2^2) \,$ &
    47.8\subpm{1.5} & 59.2\subpm{2.4} & 61.7\subpm{3.6} & \textbf{62.3\subpm{2.2}} &
    46.2\subpm{0.7} & 45.6\subpm{0.7} & 46.3\subpm{0.6} & 45.8\subpm{0.3} & \textbf{72.0\subpm{9.2}} \\
    \midrule
    \textbf{C}$\to$\textbf{P} &
    65.5\subpm{ 0.3} & 72.7\subpm{ 1.1} & 73.6\subpm{ 0.6} & \textbf{74.0\subpm{ 1.3}} &
    74.3\subpm{ 0.5} & 69.2\subpm{ 0.4} & 74.5\subpm{ 0.3} & 74.1\subpm{ 0.7} & \textbf{75.1\subpm{ 0.5}} \\
    \textbf{P}$\to$\textbf{C} &
    91.2\subpm{ 0.3} & 91.7\subpm{ 0.2} & 92.3\subpm{ 0.4} & \textbf{92.7\subpm{ 0.2}} &
    91.5\subpm{ 0.6} & 89.8\subpm{ 0.4} & \textbf{93.5\subpm{ 0.4}} & 92.1\subpm{ 0.6} & \textbf{93.4\subpm{ 0.3}} \\
    \textbf{I}$\to$\textbf{P} &
    74.8\subpm{ 0.3} & 75.4\subpm{ 0.6} & 76.9\subpm{ 0.3} & \textbf{77.2\subpm{ 0.2}} &
    75.0\subpm{ 0.6} & 74.5\subpm{ 0.4} & 76.7\subpm{ 0.3} & 76.8\subpm{ 0.4} & \textbf{77.4\subpm{ 0.3}} \\
    \textbf{P}$\to$\textbf{I} &
    83.9\subpm{ 0.1} & 88.7\subpm{ 0.5} & 90.4\subpm{ 0.3} & \textbf{90.9\subpm{ 0.2}} &
    86.0\subpm{ 0.3} & 82.2\subpm{ 0.2} & 90.6\subpm{ 0.3} & 90.2\subpm{ 1.1} & \textbf{91.1\subpm{ 0.5}} \\
    \midrule
    others$\to$\textbf{P} &
    \textbf{97.8\subpm{ 0.0}} & 96.9\subpm{ 0.2} & 97.7\subpm{ 0.2} & \textbf{97.8\subpm{ 0.2}} &
    97.6\subpm{ 0.2} & 97.6\subpm{ 0.4} & 97.0\subpm{ 0.4} & 97.6\subpm{ 0.3} & \textbf{97.9\subpm{ 0.2}} \\
    others$\to$\textbf{A} &
    88.1\subpm{ 0.1} & 73.1\subpm{ 0.3} & \textbf{88.5\subpm{ 0.6}} & \textbf{88.6\subpm{ 0.6}} &
    85.9\subpm{ 0.5} & 84.5\subpm{ 1.2} & 84.0\subpm{ 0.9} & 88.1\subpm{ 0.8} & \textbf{88.8\subpm{ 0.7}} \\
    others$\to$\textbf{C} &
    77.9\subpm{ 1.3} & 50.2\subpm{ 1.2} & 84.4\subpm{ 0.9} & \textbf{84.6\subpm{ 0.8}} &
    79.9\subpm{ 1.4} & 81.9\subpm{ 1.9} & 78.5\subpm{ 1.5} & 83.2\subpm{ 1.1} & \textbf{84.7\subpm{ 0.8}} \\
    others$\to$\textbf{S} &
    79.1\subpm{ 0.9} & 43.3\subpm{ 1.2} & 80.7\subpm{ 1.0} & \textbf{81.1\subpm{ 1.2}} &
    75.2\subpm{ 2.8} & 77.4\subpm{ 3.1} & 71.8\subpm{ 3.9} & 80.2\subpm{ 2.2} & \textbf{81.4\subpm{ 0.8}} \\
    \bottomrule
  \end{tabular}
  \vspace{-8pt}
\end{table*}

\vspace{-2pt}
\section{Conclusion and Discussion} \label{sec:conc}
\vspace{-3pt}

We propose a Causal Semantic Generative model for single-domain OOD prediction tasks, which builds upon a causal reasoning, and models the semantic (cause of prediction) and variation factors separately. 
By the causal invariance principle, 
we develop novel and efficient learning and prediction methods,
and prove the semantic-identifiability and the subsequent bounded generalization error and the success of adaptation. 
Experiments show the improved performance over prevailing baselines.

Notably, we answered the questions in the recent farseeing paper~\citep{scholkopf2021toward} on causal representation learning:
we found an appropriate condition under which ``causal variables can be recovered'', and provided ``compelling evidence on the advantages (of causal modeling) in terms of generalization''.
Also, separating semantics from variation extends to broader examples. 
Neural nets are found to change their prediction under a different texture 
\citep{geirhos2019imagenet, brendel2019approximating}.
Adversarial vulnerability~\citep{szegedy2014intriguing, goodfellow2015explaining, kurakin2016adversarial} extends variation factors to human-imperceptible features, \ie adversarial noise, which is found to have a strong correlation to the semantics~\citep{ilyas2019adversarial}.
The separation also matters for fairness when a sensitive variation factor may affect prediction. 
This work also inspires the dual connection between causal representation learning (``fill in the blanks'' given a graph) and causal discovery (``link the nodes'' given observed variables).
Our theory shows the identifiability condition for causal discovery (the additive noise assumption) also makes causal representation identifiable.
Studying the general connection between the two tasks is an interesting future work.


\bibliographystyle{abbrvnat}
\bibliography{causupv}

\onecolumn
\appendix
\section*{Appendix}
\titlespacing*{\section}{0pt}{4pt}{2pt}
\titlespacing*{\subsection}{0pt}{2pt}{2pt}
\titlespacing*{\subsubsection}{0pt}{2pt}{1pt}

\section{Proofs} \label{supp:proofs}

We first introduce some handy concepts and results to make the proof succinct, meanwhile providing more information for understanding our model and theory.
We begin with some extended discussions on \ourmodel.
\begin{definition} \label{def:repar}
  A homeomorphism $\Phi$ on $\clS\times\clV$ is called a \emph{reparameterization} from \ourmodel $p$ to \ourmodel $p'$, if 
  $\Phi_\#[p_{s,v}] = p'_{s,v}$, and $p(x|s,v) = p'(x|\Phi(s,v))$ and $p(y|s) = p'(y|\Phi^\clS(s,v))$ for any $(s,v) \in \clS\times\clV$.
  A reparameterization $\Phi$ is called to be \emph{semantic-preserving}, if its output dimensions in $\clS$ is constant of $v$: $\Phi^\clS(s,v) = \Phi^\clS(s,v')$ for any $v, v' \in \clV$ (hence denote $\Phi^\clS(s,v)$ as $\Phi^\clS(s)$ in this case).
\end{definition}
Note that a reparameterization unnecessarily has its output dimensions in $\clS$, \ie $\Phi^\clS(s,v)$, constant of $v$.
The condition that $p(y|s) = p'(y|\Phi^\clS(s,v))$ for any $v\in\clV$ does not indicate that $\Phi^\clS(s,v)$ is constant of $v$, since $p'(y|s')$ may ignore the change of $s' = \Phi^\clS(s,v)$ from the change of $v$.
The following lemma shows the meaning of a reparameterization: it allows a \ourmodel to vary while inducing the same distribution on the observed data variables $(x,y)$ (\ie, holding the same effect on describing data).
\begin{lemma} \label{lem:repar-same-pxy}
  If there exists a reparameterization $\Phi$ from \ourmodel $p$ to \ourmodel $p'$, then $p(x,y) = p'(x,y)$.
\end{lemma}
\begin{proof}
  By the definition of a reparameterization, we have:
  \begin{align}
    p(x,y) ={} & \int p(s,v) p(x|s,v) p(y|s) \dd s \ud v
    = \int \Phi^{-1}_\#[p'_{s,v}](s,v) p'(x|\Phi(s,v)) p'(y|\Phi^\clS(s,v)) \dd s \ud v \\
    ={} & \int p'_{s,v}(s',v') p'(x|s',v') p'(y|s') \dd s' \ud v'
    = p'(x,y),
  \end{align}
  where we used variable substitution $(s',v') := \Phi(s,v)$ in the second-last equality.
  Note that by the definition of pushed-forward distribution and the bijectivity of $\Phi$, $\Phi_\#[p_{s,v}] = p'_{s,v}$ implies $p_{s,v} = \Phi^{-1}_\#[p'_{s,v}]$,
  and $\int f(s',v') p'_{s,v}(s',v') \dd s' \ud v' = \int f(\Phi(s,v)) \Phi^{-1}_\#[p'_{s,v}](s,v) \dd s \ud v$
  (can also be verified deductively using the rule of change of variables, \ie Lemma~\ref{lem:changevar} in the following).
\end{proof}

We can now define and verify an equivalent relation on \ourmodels so that the resulting equivalent class contains \ourmodels that induce the same $(x,y)$ data distribution and hold the same semantic information in their $s$ variables.
\begin{definition}[semantic-equivalence] \label{def:equiv}
  We say two \ourmodels $p$ and $p'$ are \emph{semantic-equivalent}, if there exists a homeomorphism\footnote{
    A transformation is a homeomorphism if it is a continuous bijection with continuous inverse.
  } $\Phi$ on $\clS\times\clV$,
  such that \bfi is \emph{semantic-preserving}: its output dimensions in $\clS$ is constant of $v$, $\Phi^\clS(s,v) = \Phi^\clS(s)$ for any $v\in\clV$,
  and \bfii it acts as a \emph{reparameterization} from $p$ to $p'$: $\Phi_\#[p_{s,v}] = p'_{s,v}$, $p(x|s,v) = p'(x|\Phi(s,v))$ and $p(y|s) = p'(y|\Phi^\clS(s))$.
\end{definition}
Proposition~\ref{prop:equiv} in Appx.~\ref{supp:proofs-equiv} below shows that the defined binary relation is indeed an equivalence relation in common cases.
As a reparameterization, $\Phi$ allows the two models to have different latent-variable parameterizations while inducing the same distribution on the observed data variables $(x,y)$ (Lemma~\ref{lem:repar-same-pxy}).
The definition of semantic-identification (Def.~\ref{def:id}) is then the semantic-equivalence of the ground-truth \ourmodel $p^*$ to the learned \ourmodel $p$, which is also the semantic-equivalence of the learned \ourmodel $p$ to the ground-truth \ourmodel $p^*$ in common cases where it is an equivalence relation (Prop.~\ref{prop:equiv}).

This definition of semantic-equivalence can be rephrased as the \emph{existence} of a semantic-preserving reparameterization.
With proper model assumptions, we can show that \emph{any} reparameterization between two \ourmodels is semantic-preserving,
so that semantic-preserving \ourmodels cannot be converted to each other by a reparameterization that mixes $s$ with $v$.
\begin{lemma} \label{lemma:nomix}
  For two \ourmodels $p$ and $p'$, if $p'(y|s)$ has a statistics $M'(s)$ that is an injective function of $s$,
  then \emph{any} reparameterization $\Phi$ from $p$ to $p'$, if exists, has its $\Phi^\clS$ constant of $v$.
\end{lemma}
\begin{proof}
  Let $\Phi = (\Phi^\clS, \Phi^\clV)$ be any reparameterization from $p$ to $p'$.
  Then the condition that $p(y|s) = p'(y|\Phi^\clS(s,v))$ for any $v\in\clV$ indicates that $M(s) = M'(\Phi^\clS(s,v))$.
  If there exist $s\in\clS$ and $v^{(1)} \ne v^{(2)} \in \clV$ such that $\Phi^\clS(s,v^{(1)}) \ne \Phi^\clS(s,v^{(2)})$, then $M'(\Phi^\clS(s,v^{(1)})) \ne M'(\Phi^\clS(s,v^{(2)}))$ since $M'$ is injective.
  This violates $M(s) = M'(\Phi^\clS(s,v))$ which requires both $M'(\Phi^\clS(s,v^{(1)}))$ and $M'(\Phi^\clS(s,v^{(2)}))$ to be equal to $M(s)$.
  So $\Phi^\clS(s,v)$ must be constant of $v$.
\end{proof}

We then introduce two mathematical facts.
\begin{lemma}[rule of change of variables] \label{lem:changevar}
  Let $z$ be a random variable on a Euclidean space $\bbR^{d_\clZ}$ with density function $p_z(z)$,
  and let $\Phi$ be a homeomorphism on $\bbR^{d_\clZ}$ whose inverse $\Phi^{-1}$ is differentiable.
  Then the distribution of the transformed random variable $z' = \Phi(z)$ has a density function
  $\Phi_\#[p_z](z') = p_z(\Phi^{-1}(z')) \lrvert{J_{\Phi^{-1}}(z')}$,
  where $\lrvert{J_{\Phi^{-1}} (z')}$ denotes the absolute value of the determinant of the Jacobian matrix $(J_{\Phi^{-1}} (z'))_{ia} := \frac{\partial}{\partial z'_i} (\Phi^{-1})_a(z')$ of $\Phi^{-1}$ at $z'$.
\end{lemma}
\begin{proof}
  See \eg, \citet[Thm.~17.2]{billingsley2012probability}.
  Note that a homeomorphism is (Borel) measurable since it is continuous~\citep[Thm.~13.2]{billingsley2012probability}, so the definition of $\Phi_\#[p_z]$ is valid.
\end{proof}

\begin{lemma} \label{lem:conveq}
  Let $\mu$ be a random variable whose characteristic function is a.e. non-zero.
  For two functions $f$ and $f'$ on the same space, we have:
  $f * p_\mu = f' * p_\mu  \Longleftrightarrow  f = f'$ a.e.,
  where $(f * p_\mu)(x) := \int f(x) p_\mu(x-\mu) \dd \mu$ denotes convolution.
\end{lemma}
\begin{proof}
  The function equality $f * p_\mu = f' * p_\mu$ leads to the equality under Fourier transformation $\scF[f * p_\mu] = \scF[f' * p_\mu]$, which gives $\scF[f] \scF[p_\mu] = \scF[f'] \scF[p_\mu]$.
  Since $\scF[p_\mu]$ is the characteristic function of $p_\mu$, the condition that it is a.e. non-zero indicates that $\scF[f] = \scF[f']$ a.e. thus $f = f'$ a.e.
  See also \citet[Thm.~1]{khemakhem2019variational}.
\end{proof}

\subsection{Proof of the Equivalence Relation} \label{supp:proofs-equiv}

\begin{proposition} \label{prop:equiv}
  The semantic-equivalence in Def.~\ref{def:equiv} is an equivalence relation if $\clV$ is connected and is either open or closed in $\bbR^{d_\clV}$.
\end{proposition}
\begin{proof}
  Let $\Phi$ be a semantic-preserving reparameterization from one \ourmodel $p = \lrangle{p(s,v), p(x|s,v), p(y|s)}$ to another $p' = \lrangle{p'(s,v), p'(x|s,v), p'(y|s)}$.
  It has its $\Phi^\clS$ constant of $v$, so we can write $\Phi(s,v) = (\Phi^\clS(s), \Phi^\clV(s,v)) =: (\phi(s), \psi_s(v))$.

  \bfone We first show that $\phi$, and $\psi_s$ for any $s\in\clS$, are homeomorphisms on $\clS$ and $\clV$, respectively, and that $\Phi^{-1}(s',v') = (\phi^{-1}(s'), \psi_{\phi^{-1}(s')}^{-1}(v'))$.
  \begin{itemize}
    \item Since $\Phi(\clS\times\clV) = \clS\times\clV$, so $\phi(\clS) = \Phi^\clS(\clS) = \clS$, so $\phi$ is surjective.
    \item Suppose that there exists $s'\in\clS$ such that $\phi^{-1}(s') = \{ s^{(i)} \}_{i\in\clI}$ contains multiple distinct elements.
      \begin{enumerate}
        \item Since $\Phi$ is surjective, for any $v'\in\clV$, there exist $i\in\clI$ and $v\in\clV$ such that $(s',v') = \Phi(s^{(i)}, v) = (\phi(s^{(i)}), \psi_{s^{(i)}}(v))$, which means that $\bigcup_{i\in\clI} \psi_{s^{(i)}}(\clV) = \clV$.
        \item Since $\Phi$ is injective, the sets $\{\psi_{s^{(i)}}(\clV)\}_{i\in\clI}$ must be mutually disjoint.
          Otherwise, there would exist $i \ne j \in \clI$ and $v^{(1)}, v^{(2)} \in \clV$ such that $\psi_{s^{(i)}}(v^{(1)}) = \psi_{s^{(j)}}(v^{(2)})$ thus $\Phi(s^{(i)}, v^{(1)}) = (s', \psi_{s^{(i)}}(v^{(1)})) = (s', \psi_{s^{(j)}}(v^{(2)})) = \Phi(s^{(j)}, v^{(2)})$,
          which violates the injectivity of $\Phi$ since $s^{(i)} \ne s^{(j)}$.
        \item In the case where $\clV$ is open, then so is any $\psi_{s^{(i)}}(\clV) = \Phi({s^{(i)}}, \clV)$ since $\Phi$ is continuous.
          But the union of disjoint open sets $\bigcup_{i\in\clI} \psi_{s^{(i)}}(\clV) = \clV$ cannot be connected.
          This violates the condition that $\clV$ is connected.
        \item A similar argument holds in the case where $\clV$ is closed.
      \end{enumerate}
      So $\phi^{-1}(s')$ contains only one unique element for any $s'\in\clS$.
      So $\phi$ is injective.
    \item The above argument also shows that for any $s'\in\clS$, we have $\bigcup_{i\in\clI} \psi_{s^{(i)}}(\clV) = \psi_{\phi^{-1}(s')}(\clV) = \clV$.
      For any $s\in\clS$, there exists $s'\in\clS$ such that $s = \phi^{-1}(s')$, so we have $\psi_s(\clV) = \clV$.
      So $\psi_s$ is surjective for any $s\in\clS$.
    \item Suppose that there exist $v^{(1)} \ne v^{(2)} \in \clV$ such that $\psi_s(v^{(1)}) = \psi_s(v^{(2)})$.
      Then $\Phi(s,v^{(1)}) = (\phi(s), \psi_s(v^{(1)})) = (\phi(s), \psi_s(v^{(2)})) = \Phi(s,v^{(2)})$, which contradicts the injectivity of $\Phi$ since $v^{(1)} \ne v^{(2)}$.
      So $\psi_s$ is injective for any $s\in\clS$.
    \item That $\Phi$ is continuous and $\Phi(s,v) = (\phi(s), \psi_s(v))$ indicates that $\phi$ and $\psi_s$ are continuous.
      For any $(s',v') \in \clS\times\clV$, we have $\Phi( \phi^{-1}(s'), \psi_{\phi^{-1}(s')}^{-1}(v') ) = ( \phi(\phi^{-1}(s')), \psi_{\phi^{-1}(s')} (\psi_{\phi^{-1}(s')}^{-1}(v')) ) = (s', v')$.
      Applying $\Phi^{-1}$ to both sides gives $\Phi^{-1}(s',v') = (\phi^{-1}(s'), \psi_{\phi^{-1}(s')}^{-1}(v'))$.
    \item Since $\Phi^{-1}$ is continuous, $\phi^{-1}$ and $\psi_s^{-1}$ are also continuous.
  \end{itemize}

  \bftwo We now show that the relation is an equivalence relation.
  It amounts to showing the following three properties.
  \begin{itemize}
    \item Reflexivity.
      For two identical \ourmodels, we have $p(s,v) = p'(s,v)$, $p(x|s,v) = p'(x|s,v)$ and $p(y|s) = p'(y|s)$.
      So the identity map as $\Phi$ obviously satisfies all the requirements.
    \item Symmetry.
      Let $\Phi$ be a semantic-preserving reparameterization from $p = \lrangle{p(s,v), p(x|s,v), p(y|s)}$ to $p' = \lrangle{p'(s,v), p'(x|s,v), p'(y|s)}$. 
      From the above conclusion in \bfone, we know that $(\Phi^{-1})^\clS(s',v') = \phi^{-1}(s')$ is semantic-preserving. 
      Also, $\Phi^{-1}$ is a homeomorphism on $\clS\times\clV$ since $\Phi$ is.
      So we only need to show that $\Phi^{-1}$ is a reparameterization from $p'$ to $p$ for symmetry. 
      \begin{enumerate}
        \item From the definition of pushed-forward distribution, we have $\Phi^{-1}_\#[p'_{s,v}] = p_{s,v}$ if $\Phi_\#[p_{s,v}] = p'_{s,v}$.
          It can also be verified through the rule of change of variables (Lemma~\ref{lem:changevar}) when $\Phi$ and $\Phi^{-1}$ are differentiable.
          From $\Phi_\#[p_{s,v}] = p'_{s,v}$, we have for any $(s',v')$, $p_{s,v}(\Phi^{-1}(s',v')) \lrvert{J_{\Phi^{-1}}(s',v')} = p'_{s,v}(s',v')$.
          Since for any $(s,v)$ there exists $(s',v')$ such that $(s,v) = \Phi^{-1}(s',v')$, this implies that for any $(s,v)$,
          $p_{s,v}(s,v) \lrvert{J_{\Phi^{-1}}(\Phi(s,v))} = p'_{s,v}(\Phi(s,v))$, or
          $p_{s,v}(s,v) = p'_{s,v}(\Phi(s,v)) / \lrvert{J_{\Phi^{-1}}(\Phi(s,v))} = p'_{s,v}(\Phi(s,v)) \lrvert{J_\Phi (s,v)}$ (inverse function theorem),
          which means that $p_{s,v} = \Phi^{-1}_\#[p'_{s,v}]$ by the rule of change of variables.
        \item For any $(s',v')$, there exists $(s,v)$ such that $(s',v') = \Phi(s,v)$, so
          $p'(x|s',v') = p'(x|\Phi(s,v)) = p(x|s,v) = p(x|\Phi^{-1}(s',v'))$, and
          $p'(y|s') = p'(y|\Phi^\clS(s)) = p(y|s) = p(y|(\Phi^{-1})^\clS(s'))$.
      \end{enumerate}
      So $\Phi^{-1}$ is a reparameterization from $p'$ to $p$.
    \item Transitivity.
      Given a third \ourmodel $p'' = \lrangle{p''(s,v), p''(x|s,v), p''(y|s)}$ that is semantic-equivalent to $p'$, there exists a semantic-preserving reparameterization $\Phi'$ from $p'$ to $p''$. 
      It is easy to see that $(\Phi'\circ\Phi)^\clS(s,v) = \Phi'^\clS(\Phi^\clS(s,v)) = \Phi'^\clS(\Phi^\clS(s))$ is constant of $v$ thus semantic-preserving.
      As the composition of two homeomorphisms $\Phi$ and $\Phi'$ on $\clS\times\clV$, $\Phi'\circ\Phi$ is also a homeomorphism.
      So we only need to show that $\Phi'\circ\Phi$ is a reparameterization from $p$ to $p''$ for transitivity.
      \begin{enumerate}
        \item From the definition of pushed-forward distribution, we have $(\Phi'\circ\Phi)_\#[p_{s,v}] = \Phi'_\#[\Phi_\#[p_{s,v}]] = \Phi'_\#[p'_{s,v}] = p''_{s,v}$ if $\Phi_\#[p_{s,v}] = p'_{s,v}$ and $\Phi'_\#[p'_{s,v}] = p''_{s,v}$.
          It can also be verified through the rule of change of variables (Lemma~\ref{lem:changevar}) when $\Phi^{-1}$ and $\Phi'^{-1}$ are differentiable.
          For any $(s'',v'')$, we have
          \begin{align}
            & (\Phi'\circ\Phi)_\#[p_{s,v}](s'',v'')
            =     p_{s,v}((\Phi'\circ\Phi)^{-1}(s'',v'')) \lrvert{J_{(\Phi'\circ\Phi)^{-1}} (s'',v'')} \\
            ={} & p_{s,v}(\Phi^{-1}(\Phi'^{-1}(s'',v''))) \lrvert{J_{\Phi^{-1}} (\Phi'^{-1}(s'',v''))} \lrvert{J_{\Phi'^{-1}} (s'',v'')} \\
            ={} & \Phi_\#[p_{s,v}] (\Phi'^{-1}(s'',v'')) \lrvert{J_{\Phi'^{-1}} (s'',v'')} \\
            ={} & p'_{s,v}(\Phi'^{-1}(s'',v'')) \lrvert{J_{\Phi'^{-1}} (s'',v'')} = \Phi'_\#[p'_{s,v}](s'',v'')
            =     p''_{s,v}(s'',v'').
          \end{align}
        \item For any $(s,v)$, we have:
          \begin{align}
            & p(x|s,v) = p'(x|\Phi(s,v)) = p''(x|\Phi'(\Phi(s,v))) = p''(x|(\Phi'\circ\Phi)(s,v)), \\
            & p(y|s) = p'(y|\Phi^\clS(s)) = p''(y|\Phi'^\clS(\Phi^\clS(s))) = p''(y|(\Phi'\circ\Phi)^\clS(s)).
          \end{align}
      \end{enumerate}
      So $\Phi'\circ\Phi$ is a reparameterization from $p$ to $p''$.
  \end{itemize}
  This completes the proof for an equivalence relation.
\end{proof}

\subsection{Proof of the Semantic-Identifiability Thm.~\ref{thm:id}} \label{supp:proofs-id}

We present a more general and detailed version of Thm.~\ref{thm:id} and prove it.
The conclusions in the theorem in the main context corresponds to conclusions \bfii and \bfi below by taking the two \ourmodels $p'$ and $p$ as the well-learned \ourmodel $p$ and the ground-truth \ourmodel $p^*$, respectively.

\begin{theoremnum}{\ref*{thm:id}'}[semantic-identifiability] \label{thm:id-formal}
  Consider two \ourmodels $p$ and $p'$ that have 
  Assumption~\ref{assm:anm-bij} hold,
  with the bounded derivative conditions specified to be that for both \ourmodels, $f^{-1}$ and $g$ are twice and $f$ thrice differentiable with mentioned derivatives bounded.
  Further assume that they have absolutely continuous priors whose log-densities $\log p(s,v)$ and $\log p'(s,v)$ are bounded up to the second-order.
  If the two \ourmodels induce the same distribution on data, \ie $p(x,y) = p'(x,y)$, then they are semantic-equivalent, under \textbf{one of} the following three conditions:
  \footnote{To be precise, the conclusions are that the equalities in Def.~\ref{def:equiv} hold a.e. for condition \bfi, hold asymptotically in the limit $\frac{1}{\sigma_\mu^2} \to \infty$ for condition \bfii, and hold up to a negligible quantity for condition \bfiii.} \\
  \bfemi $p_\mu$ has an a.e. non-zero characteristic function (\eg, a Gaussian distribution);\footnote{
    This also requires that $p$ and $p'$ have the same $p_\mu$, or that the ground-truth $p_\mu$ is known in learning. 
    However, $p_\mu$ is easier to model/specify/learn than $f$, and $f$ dominates $p(x|s,v)$ over $p_\mu$ when the causal mechanism tends to be strong.
    So learning or specifying $p_\mu$ in learning is not a significant violation of this requirement.
  }\\
  \bfemii $\frac{1}{\sigma_\mu^2} \to \infty$, where $\sigma_\mu^2 := \bbE[\mu\trs \mu]$; \\
  \bfemiii $\frac{1}{\sigma_\mu^2} \!\gg\! B'^2_{f^{-1}}\! \max \{ 
    B'_{\log p} B'_g + \frac{1}{2} B''_g + \frac{3}{2} d B'_{f^{-1}}\! B''_f B'_g,
    B_p B'^d_{f^{-1}}\! ( B'^2_{\log p} \!+\! B''_{\log p} \!+\! 3 d B'_{f^{-1}}\! B''_f B'_{\log p} \!+\! 3 d^{\frac{3}{2}} B'^2_{f^{-1}}\! B''^2_f \!+\! d^3 B'''_f B'_{f^{-1}}\! )
  \}$, where $d := d_\clS + d_\clV$, and for both \ourmodels, the constant $B_p$ bounds $p(s,v)$, $B'_{f^{-1}}, B'_g, B'_{\log p}$ and $B''_f, B''_g, B''_{\log p}$ bound the 2-norms\footnote{
    As an induced operator norm for matrices (not the Frobenius norm).
  } of the gradient/Jacobian and the Hessians of the respective functions, and $B'''_f$ bounds all the 3rd-order derivatives of $f$.
\end{theoremnum}
\begin{proof}
  Without loss of generality, we assume that $\mu$ and $\nu$ (for continuous $y$) have zero mean.
  If it is not, we can redefine $f(s,v) := f(s,v) + \bbE[\mu]$ and $\mu := \mu - \bbE[\mu]$ (similarly for $\nu$ for continuous $y$) which does not alter the joint distribution $p(s,v,x,y)$ nor violates any assumptions.
  Also without loss of generality, we consider one scalar component (dimension) $l$ of $y$, and abuse the use of symbols $y$ and $g$ for $y_l$ and $g_l$ to avoid unnecessary complication.
  Note that for continuous $y$, due to the additive noise structure $y = g(s) + \nu$ and that $\nu$ has zero mean, we also have $\bbE[y|s] = g(s)$ as the same as the categorical $y$ case (under the one-hot representation).
  We sometimes denote $z := (s,v)$ for convenience.

  First note that for both \ourmodels and both continuous and categorical $y$, by construction $g(s)$ is a sufficient statistics of $p(y|s)$ (not only the expectation $\bbE[y|s]$), and it is injective.
  So by Lemma~\ref{lemma:nomix}, we only need to show that there exists a reparameterization from $p$ to $p'$.
  We will show that $\Phi := f'^{-1} \circ f$ is such a reparameterization.

  Since $f$ and $f'$ are bijective and continuous, we have $\Phi^{-1} = f^{-1} \circ f'$, so $\Phi$ is bijective and $\Phi$ and $\Phi^{-1}$ are continuous.
  So $\Phi$ is a homeomorphism.
  Also, by construction, we have:
  \begin{align}
    p(x|z) = p_\mu(x - f(z)) = p_\mu(x - f'(f'^{-1}(f(z)))) = p_\mu(x - f'(\Phi(z))) = p'(x|\Phi(z)). \label{eqn:x-convert}
  \end{align}
  So we only need to show that $p(x,y) = p'(x,y)$ indicates $\Phi_\#[p_z] = p'_z$ and $p(y|s) = p'(y|\Phi^\clS(s,v)), \forall v\in\clV$ under the conditions.

  \paragraph{Proof under condition \bfi.}
  We begin with a useful reformulation of the integral $\int t(z) p(x|z) \dd z$ for a general function $t$ of $z$.
  We will encounter integrals in this form.
  By the additive noise Assumption~\ref{assm:anm-bij}, we have $p(x|z) = p_\mu(x - f(z))$, so we consider a transformation $\Psi_x(z) := x - f(z)$ and let $\mu = \Psi_x(z)$.
  It is invertible, $\Psi_x^{-1}(\mu) = f^{-1}(x - \mu)$, and $J_{\Psi_x^{-1}}(\mu) = -J_{f^{-1}}(x - \mu)$.
  By these definitions and the rule of change of variables, we have:
  \begin{align}
    \int t(z) p(x|z) \dd z ={} & \int t(z) p_\mu(\Psi_x(z)) \dd z
    =     \int t(\Psi_x^{-1}(\mu)) p(\mu) \lrvert{J_{\Psi_x^{-1}} (\mu)} \dd \mu \\
    ={} & \int t(f^{-1}(x - \mu)) p(\mu) \lrvert{J_{f^{-1}} (x - \mu)} \dd \mu \\
    ={} & \bbE_{p(\mu)} [(\ttb V) (x - \mu)] \label{eqn:useful-expc} \\
    ={} & (f_\#[t] * p_\mu) (x), \label{eqn:useful-conv}
  \end{align}
  where we have denoted functions $\ttb := t \circ f^{-1}$, $V := \lrvert{J_{f^{-1}}}$, and abused the push-forward notation $f_\#[t]$ for a general function $t$ to formally denote $(t \circ f^{-1}) \lrvert{J_{f^{-1}}} = \ttb V$.

  According to the graphical structure of \ourmodel, we have:
  \begin{align}
    p(x) ={} & \int p(z) p(x|z) \dd z, \label{eqn:px-raw} \\
    \bbE[y|x] ={} & \frac{1}{p(x)} \int y p(x,y) \dd y = \frac{1}{p(x)} \iint y p(z) p(x|z) p(y|s) \dd z \ud y \\
    ={} & \frac{1}{p(x)} \int p(z) p(x|z) \bbE[y|s] \dd z = \frac{1}{p(x)} \int g(s) p(z) p(x|z) \dd z. \label{eqn:Ey1x-raw}
  \end{align}
  So from \eqref{eqn:useful-conv}, we have:
  \begin{align}
    & p(x) = (f_\#[p_z] * p_\mu) (x), 
    & \bbE[y|x] = \frac{1}{p(x)} (f_\#[g p_z] * p_\mu) (x). \label{eqn:Ey1x-conv}
  \end{align}
  Matching the data distribution $p(x,y) = p'(x,y)$ indicates both $p(x) = p'(x)$ and $\bbE[y|x] = \bbE'[y|x]$.
  Using Lemma~\ref{lem:conveq} under condition \bfi, this further indicates:
  \begin{align}
    & f_\#[p_z] = f'_\#[p'_z] \text{ a.e.},
    & f_\#[g p_z] = f'_\#[g' p'_z] \text{ a.e.},
  \end{align}
  given that $p$ and $p'$ have the same $p_\mu$.
  The former indicates $\Phi_\#[p_z] = p'_z$.
  The latter can be reformed as $\ggb f_\#[p_z] = \ggb' f'_\#[p'_z]$ a.e., so $\ggb = \ggb'$ a.e.,
  where we have denoted $\ggb := g \circ (f^{-1})^\clS$ and $\ggb' := g' \circ (f'^{-1})^\clS$ similarly.
  From $\ggb = \ggb'$, we have for any $v\in\clV$,
  \begin{align}
    g(s) ={} & g( (f^{-1} \circ f)^\clS(s,v) ) = g( (f^{-1})^\clS (f(s,v)) ) = \ggb(f(s,v)) \\
    ={} & \ggb'(f(s,v)) = g'( (f'^{-1})^\clS (f(s,v)) ) = g'(\Phi^\clS(s,v)). \label{eqn:g-convert}
  \end{align}
  For both continuous and categorical $y$, $g(s)$ uniquely determines $p(y|s)$.
  So the above equality means that $p(y|s) = p'(y|\Phi^\clS(s,v))$ for any $v\in\clV$.

  \paragraph{Proof under condition \bfii.}
  Applying \eqref{eqn:useful-expc} to Eqs.~(\ref{eqn:px-raw},~\ref{eqn:Ey1x-raw}) (or expanding \eqref{eqn:Ey1x-conv}), we have:
  \begin{align}
    & p(x) = \bbE_{p(\mu)} [(\ppb_z V) (x - \mu)], 
    & \bbE[y|x] = \frac{1}{p(x)} \bbE_{p(\mu)} [(\ggb \ppb_z V) (x - \mu)], \label{eqn:Ey1x-expc}
  \end{align}
  where we have similarly denoted $\ppb_z := p_z \circ f^{-1}$.
  Under condition \bfii, $\bbE[\mu\trs \mu]$ is infinitesimal, so we can expand the expressions w.r.t $\mu$.
  For $p(x)$, we have:
  \begin{align}
    p(x) ={} & \bbE_{p(\mu)} \big[ \ppb_z V - \nabla (\ppb_z V)\trs \mu 
    + \frac{1}{2} \mu\trs \nabla\nabla\trs (\ppb_z V) \mu + O(\bbE[\lrVert{\mu}_2^3]) \big] \\
    ={} & \ppb_z V + \frac{1}{2} \bbE_{p(\mu)} \big[ \mu\trs \nabla\nabla\trs (\ppb_z V) \mu \big] + O(\sigma_\mu^3),
    \label{eqn:px-expd}
  \end{align}
  where all functions are evaluated at $x$.
  For $\bbE[y|x]$, we first expand $1/p(x)$ using $\frac{1}{x+\varepsilon} = \frac{1}{x} - \frac{\varepsilon}{x^2} + O(\varepsilon^2)$ to get:
    $\frac{1}{p(x)} = \frac{1}{\ppb_z V} - \frac{1}{2 \ppb_z^2 V^2} \bbE_{p(\mu)} \big[ \mu\trs \nabla\nabla\trs (\ppb_z V) \mu \big] + O(\sigma_\mu^3)$. 
  The second term is expanded as:
    $\ggb \ppb_z V + \frac{1}{2} \bbE_{p(\mu)} \big[ \mu\trs \nabla\nabla\trs (\ggb \ppb_z V) \mu \big] + O(\sigma_\mu^3)$.
  Combining the two parts, we have:
  \begin{align}
    \bbE[y|x] ={} & \ggb + \frac{1}{2} \bbE_{p(\mu)} \big[ \mu\trs \big( (\nabla \log \ppb_z V) \nabla \ggb \trs
    + \nabla \ggb (\nabla \log \ppb_z V)\trs + \nabla\nabla\trs \ggb \big) \mu \big] + O(\sigma_\mu^3).
    \label{eqn:Ey1x-expd}
  \end{align}
  This equation holds for any $x\in\supp(p_x)$ since the expectation is taken w.r.t the distribution $p(x,y)$.
  Since $p(x,y) = p'(x,y)$, the considered $x$ here is any value generated by the model.
  So up to $O(\sigma_\mu^2)$,
  \begin{align}
    \lrvert{ p(x) - (\ppb_z V)(x) }
    ={} & \frac{1}{2} \lrvert{ \bbE_{p(\mu)} \big[ \mu\trs \nabla\nabla\trs (\ppb_z V) \mu \big] }
    \le \frac{1}{2} \bbE_{p(\mu)} \big[ \lrvert{ \mu\trs \nabla\nabla\trs (\ppb_z V) \mu } \big] \\
    \le{} & \frac{1}{2} \bbE_{p(\mu)} \big[ \lrVert{\mu}_2 \lrVert{\nabla\nabla\trs (\ppb_z V)}_2 \lrVert{\mu}_2 \big]
    = \frac{1}{2} \bbE[\mu\trs \mu] \lrVert{ \nabla\nabla\trs (\ppb_z V) }_2 \\
    ={} & \frac{1}{2} \bbE[\mu\trs \mu] \lrvert{\ppb_z V} \lrVert{ \nabla\nabla\trs \log \ppb_z V + (\nabla \log \ppb_z V) (\nabla \log \ppb_z V)\trs }_2 \\
    \le{} & \frac{1}{2} \bbE[\mu\trs \mu] \lrvert{\ppb_z V} \big( \lrVert{ \nabla\nabla\trs \log \ppb_z V }_2 + \lrVert{\nabla \log \ppb_z V}_2^2 \big),
    \label{eqn:px-resd} \\
    \lrvert{ \bbE[y|x] - \ggb(x) }
    ={} & \frac{1}{2} \Big\vert \bbE_{p(\mu)} \big[ \mu\trs \big( (\nabla \log \ppb_z V) \nabla \ggb \trs
    + \nabla \ggb (\nabla \log \ppb_z V)\trs + \nabla\nabla\trs \ggb \big) \mu \big] \Big\vert \\
    \le{} & \frac{1}{2} \bbE_{p(\mu)} \big[ \big\vert \mu\trs \big( (\nabla \log \ppb_z V) \nabla \ggb \trs
    + \nabla \ggb (\nabla \log \ppb_z V)\trs + \nabla\nabla\trs \ggb \big) \mu \big\vert \big] \\
    \le{} & \frac{1}{2} \bbE_{p(\mu)} \big[ \lrVert{\mu}_2 \big\Vert (\nabla \log \ppb_z V) \nabla \ggb \trs
    + \nabla \ggb (\nabla \log \ppb_z V)\trs + \nabla\nabla\trs \ggb \big\Vert_2 \lrVert{\mu}_2 \big] \\
    \le{} & \frac{1}{2} \bbE[\mu\trs \mu] \big( \lrVert{(\nabla \log \ppb_z V) \nabla \ggb \trs}_2
    + \lrVert{\nabla \ggb (\nabla \log \ppb_z V)\trs}_2 + \lrVert{\nabla\nabla\trs \ggb}_2 \big) \\
    ={} & \bbE[\mu\trs \mu] \Big( \lrvert{ (\nabla \log \ppb_z V)\trs \nabla \ggb } + \frac{1}{2} \lrVert{\nabla\nabla\trs \ggb}_2 \Big).
    \label{eqn:Ey1x-resd}
  \end{align}
  Given the bounding conditions in the theorem, the multiplicative factors to $\bbE[\mu\trs \mu]$ in the last expressions are bounded by a constant.
  So when $\frac{1}{\sigma_\mu^2} \to \infty$, \ie $\bbE[\mu\trs \mu] \to 0$, we have $p(x)$ and $\bbE[y|x]$ converge uniformly to $(\ppb_z V)(x) = f_\#[p_z] (x)$ and $\ggb(x)$, respectively.
  So $p(x,y) = p'(x,y)$ indicates $f_\#[p_z] = f'_\#[p'_z]$ and $\ggb = \ggb'$, which means $\Phi_\#[p_z] = p'_z$ and $p(y|s) = p'(y|\Phi^\clS(s,v))$ for any $v\in\clV$, due to \eqref{eqn:g-convert} and the explanation that follows.

  \paragraph{Proof under condition \bfiii.}
  We only need to show that when $\frac{1}{\sigma_\mu^2}$ is much larger than the given quantity, we still have $p(x,y) = p'(x,y) \Longrightarrow \ppb_z V = \ppb_z' V', \ggb = \ggb'$ up to a negligible effect.
  This task amounts to showing that the residuals $\lrvert{p(x) - (\ppb_z V)(x)}$ and $\lrvert{\bbE[y|x] - \ggb(x)}$ controlled by Eqs.~(\ref{eqn:px-resd},~\ref{eqn:Ey1x-resd}) are negligible.
  To achieve this, we need to further expand the controlling functions using derivatives of $f$, $g$ and $p_z$ explicitly, and bound them by the bounding constants.
  In the following, we use indices $a,b,c$ for the components of $x$ and $i,j,k$ for those of $z$.
  For functions of $z$ appearing in the following (\eg, $f$, $g$, $p_z$ and their derivatives), they are evaluated at $z = f^{-1}(x)$ since we are bounding functions of $x$.

  \bfone Bounding $\lrvert{\bbE[y|x] - \ggb(x)} \le \bbE[\mu\trs \mu] \big( \lrvert{ (\nabla \log \ppb_z V)\trs \nabla \ggb } + \frac{1}{2} \lrVert{\nabla\nabla\trs \ggb}_2 \big)$ from \eqref{eqn:Ey1x-resd}.

  From the chain rule of differentiation, it is easy to show that:
  \begin{align}
    & \nabla \log \ppb_z = J_{f^{-1}} \nabla \log p_z,
    & \nabla \ggb = J_{(f^{-1})^\clS} \nabla g = J_{f^{-1}} \nabla_z g,
    \label{eqn:grad-logp-g}
  \end{align}
  where $\nabla_z g = (\nabla g\trs, 0_{d_\clV}\trs)\trs$ (recall that $g$ is a function only of $s$).
  For the term $\nabla \log V$, we apply Jacobi's formula for the derivative of the log-determinant:
  \begin{align}
    \partial_a \log V(x) ={} & \partial_a \log \lrvert{J_{f^{-1}}(x)}
    = \tr\Big( J_{f^{-1}}^{-1}(x) \big( \partial_a J_{f^{-1}}(x) \big) \Big)
    = \sum_{b,i} J_{f^{-1}}^{-1} (x)_{ib} \big( \partial_a J_{f^{-1}} (x)_{bi} \big) \\
    ={} & \sum_{b,i} J_f (f^{-1}(x))_{ib} \partial_b \partial_a f^{-1}_i(x)
    = \sum_i \big( J_f (\nabla\nabla\trs f^{-1}_i) \big)_{ia}.
    \label{eqn:grad-logv-orig}
  \end{align}
  However, as bounding \eqref{eqn:grad-logp-g} already requires bounding $\lrVert{J_{f^{-1}}}_2$, directly using this expression to bound $\lrVert{\nabla \log V}_2$ would require to also bound $\lrVert{J_f}_2$.
  This requirement to bound the first-order derivatives of both $f$ and $f^{-1}$ is a relatively restrictive one.
  To ease the requirement, we would like to express $\nabla \log V$ in terms of $J_{f^{-1}}$.
  This can be achieved by expressing $\nabla\nabla\trs f^{-1}_i$'s in terms of $\nabla\nabla\trs f_c$'s.
  To do this, first consider a general invertible-matrix-valued function $A(\alpha)$ on a scalar $\alpha$.
  We have $0 = \partial_\alpha \big( A(\alpha)^{-1} A(\alpha) \big) = (\partial_\alpha A^{-1}) A + A^{-1} \partial_\alpha A$,
  so we have $A^{-1} \partial_\alpha A = -(\partial_\alpha A^{-1}) A$,
  consequently $\partial_\alpha A = - A (\partial_\alpha A^{-1}) A$.
  Using this relation (in the fourth equality below), we have:
  \begin{align}
    & \big( \nabla\nabla\trs f^{-1}_i \big)_{ab} = \partial_a \partial_b f^{-1}_i
    = \partial_a \big( J_{f^{-1}} \big)_{bi} = \big( \partial_a J_{f^{-1}} \big)_{bi} \\
    ={} & -\Big( J_{f^{-1}} (\partial_a J_{f^{-1}}^{-1}) J_{f^{-1}} \Big)_{bi}
    = -\Big( J_{f^{-1}} \big( \partial_a J_f \big) J_{f^{-1}} \Big)_{bi} \\
    ={} & -\sum_{jc} (J_{f^{-1}})_{bj} \big( \partial_a (\partial_j f_c) \big) (J_{f^{-1}})_{ci}
    = -\sum_{jck} (J_{f^{-1}})_{bj} (\partial_k \partial_j f_c) (\partial_a f^{-1}_k) (J_{f^{-1}})_{ci} \\
    ={} & -\sum_c (J_{f^{-1}})_{ci} \sum_{jk} (J_{f^{-1}})_{bj} (\partial_k \partial_j f_c) (J_{f^{-1}})_{ak}
    = -\sum_c (J_{f^{-1}})_{ci} \big( J_{f^{-1}} (\nabla\nabla\trs f_c) J_{f^{-1}}\trs \big)_{ab},
  \end{align}
  or in matrix form,
  \begin{align}
    \nabla\nabla\trs f^{-1}_i = -\sum_c (J_{f^{-1}})_{ci} J_{f^{-1}} (\nabla\nabla\trs f_c) J_{f^{-1}}\trs =: -\sum_c (J_{f^{-1}})_{ci} K^c,
    \label{eqn:hess-finv}
  \end{align}
  where we have defined the matrix $K^c := J_{f^{-1}} (\nabla\nabla\trs f_c) J_{f^{-1}}\trs$ which is symmetric.
  Substituting with this result, we can transform \eqref{eqn:grad-logv-orig} into a desired form:
  \begin{align}
    \nabla \log V(x) ={} & \sum_i \big( J_f (\nabla\nabla\trs f^{-1}_i) \big)_{i:}\trs
    = -\sum_i \Big( J_f \sum_c (J_{f^{-1}})_{ci} J_{f^{-1}} (\nabla\nabla\trs f_c) J_{f^{-1}}\trs \Big)_{i:}\trs \\
    ={} & -\sum_i \Big( \sum_c (J_{f^{-1}})_{ci} J_f J_f^{-1} (\nabla\nabla\trs f_c) J_{f^{-1}}\trs \Big)_{i:}\trs
    = -\sum_{ci} (J_{f^{-1}})_{ci} \Big( (\nabla\nabla\trs f_c) J_{f^{-1}}\trs \Big)_{i:}\trs \\
    ={} & -\sum_c \Big( J_{f^{-1}} (\nabla\nabla\trs f_c) J_{f^{-1}}\trs \Big)_{c:}\trs
    = -\sum_c (K^c_{c:})\trs
    = -\sum_c K^c_{:c},
    \label{eqn:grad-logv}
  \end{align}
  so its norm can be bounded by:
  \begin{align}
    \lrVert{\nabla \log V(x)}_2 ={} & \Big\Vert \sum_c K^c_{c:} \Big\Vert_2 = \Big\Vert \sum_c (J_{f^{-1}})_{c:} (\nabla\nabla\trs f_c) J_{f^{-1}}\trs \Big\Vert_2 \\
    \le{} & \sum_c \lrVert{(J_{f^{-1}})_{c:}}_2 \lrVert{\nabla\nabla\trs f_c}_2 \lrVert{J_{f^{-1}}}_2
    \le B''_f B'_{f^{-1}} \sum_c \lrVert{(J_{f^{-1}})_{c:}}_2 \\
    \le{} & d B'^2_{f^{-1}} B''_f,
    \label{eqn:gradnorm-logv}
  \end{align}
  where we have used the following result in the last inequality:
  \begin{align}
    \sum_c \lrVert{(J_{f^{-1}})_{c:}}_2
    \le d^{1/2} \sqrt{\sum_c \lrVert{(J_{f^{-1}})_{c:}}_2^2}\
    = d^{1/2} \lrVert{J_{f^{-1}}}_F
    \le d \lrVert{J_{f^{-1}}}_2 \le d B'_{f^{-1}}.
    \label{eqn:gradnorm-sumJrows}
  \end{align}
  Integrating \eqref{eqn:grad-logp-g} and \eqref{eqn:gradnorm-logv}, we have:
  \begin{align}
    \lrvert{(\nabla \log \ppb_z V)\trs \nabla \ggb}
    ={} & (J_{f^{-1}} \nabla \log p_z + \nabla \log V)\trs J_{f^{-1}} \nabla_z g \\
    \le{} & \big( \lrVert{J_{f^{-1}}}_2 \lrVert{\nabla \log p_z}_2 + \lrVert{\nabla \log V}_2 \big) \lrVert{J_{f^{-1}}} \lrVert{\nabla g}_2 \\
    \le{} & \big( B'_{f^{-1}} B'_{\log p} + d B'^2_{f^{-1}} B''_f \big) B'_{f^{-1}} B'_g \\
    ={} & \big( B'_{\log p} + d B'_{f^{-1}} B''_f \big) B'^2_{f^{-1}} B'_g.
    \label{eqn:bound-1.1}
  \end{align}
  For the Hessian of $\ggb$, direct calculus gives:
  \begin{align}
    \nabla\nabla\trs \ggb ={} & J_{(f^{-1})^\clS} (\nabla\nabla\trs g) J_{(f^{-1})^\clS}\trs + \sum_{i=1}^{d_\clS} (\nabla g)_{s_i} (\nabla\nabla\trs f^{-1}_{s_i}) \\
    ={} & J_{f^{-1}} (\nabla_z\nabla_z\trs g) J_{f^{-1}}\trs + \sum_i (\nabla_z g)_i (\nabla\nabla\trs f^{-1}_i).
    \label{eqn:hess-ggb-orig}
  \end{align}
  To avoid the requirement of bounding both $\nabla\nabla\trs f_c$'s and $\nabla\nabla\trs f^{-1}_i$'s, we substitute $\nabla\nabla\trs f^{-1}_i$ using \eqref{eqn:hess-finv}:
  \begin{align}
    \nabla\nabla\trs \ggb
    ={} & J_{f^{-1}} (\nabla_z\nabla_z\trs g) J_{f^{-1}}\trs
    - \sum_i (\nabla_z g)_i \sum_c (J_{f^{-1}})_{ci} K^c \\
    ={} & J_{f^{-1}} (\nabla_z\nabla_z\trs g) J_{f^{-1}}\trs
    - \sum_c \Big( (J_{f^{-1}})_{c,:} (\nabla_z g) \Big) K^c.
    \label{eqn:hess-ggb}
  \end{align}
  So its norm can be bounded by:
  \begin{align}
    \lrVert{\nabla\nabla\trs \ggb}_2
    \le{} & \lrVert{J_{f^{-1}}}_2^2 \lrVert{\nabla\nabla\trs g}_2 + \sum_c \lrvert{(J_{f^{-1}})_{c:} (\nabla_z g)} \lrVert{K^c}_2 \\
    \le{} & B'^2_{f^{-1}} B''_g + \sum_c \lrvert{(J_{f^{-1}})_{c:} (\nabla_z g)} B'^2_{f^{-1}} B''_f \\
    \le{} & B'^2_{f^{-1}} \Big( B''_g + B''_f \sum_c \lrVert{(J_{f^{-1}})_{c:}}_2 \lrVert{\nabla_z g}_2 \Big) \\
    \le{} & B'^2_{f^{-1}} \Big( B''_g + B''_f B'_g \sum_c \lrVert{(J_{f^{-1}})_{c:}}_2 \Big) \\
    \le{} & B'^2_{f^{-1}} \Big( B''_g + d B'_{f^{-1}} B''_f B'_g \Big),
    \label{eqn:bound-1.2}
  \end{align}
  where we have used \eqref{eqn:gradnorm-sumJrows} in the last inequality.
  Assembling \eqref{eqn:bound-1.1} and \eqref{eqn:bound-1.2} into \eqref{eqn:Ey1x-resd}, we have:
  \begin{align}
    \lrvert{ \bbE[y|x] - \ggb(x) }
    \le \bbE[\mu\trs \mu] B'^2_{f^{-1}} \big( B'_{\log p} B'_g  + \frac{1}{2} B''_g + \frac{3}{2} d B'_{f^{-1}} B''_f B'_g \big).
    \label{eqn:bound-1.3}
  \end{align}
  So given the condition \bfiii, this residual can be neglected.

  \bftwo Bounding $\lrvert{p(x) - (\ppb_z V)(x)} \le \frac{1}{2} \bbE[\mu\trs \mu] \lrvert{\ppb_z V} \big( \lrVert{\nabla \log \ppb_z V}_2^2 + \lrVert{ \nabla\nabla\trs \log \ppb_z }_2 + \lrVert{ \nabla\nabla\trs \log V }_2 \big)$ from \eqref{eqn:px-resd}.

  To begin with, for any $x$, $\ppb_z(x) = p_z(f^{-1}(x)) \le B_p$, and $V(x) = \lrvert{J_{f^{-1}}(x)}$ is the product of absolute eigenvalues of $J_{f^{-1}}(x)$.
Since $\lrVert{J_{f^{-1}}(x)}_2$ is the largest absolute eigenvalue of $J_{f^{-1}}(x)$, so $V(x) \le \lrVert{J_{f^{-1}}(x)}_2^{d} \le B'^{d}_{f^{-1}}$.

  For the first norm in the bracket of the r.h.s of \eqref{eqn:px-resd}, we have:
  \begin{align}
    \lrVert{\nabla \log \ppb_z V}_2^2
    ={} & \lrVert{\nabla \log \ppb_z}_2^2 + 2 (\nabla \log \ppb_z)\trs \nabla \log V + \lrVert{\nabla \log V}_2^2 \\
    \le{} & \lrVert{\nabla \log \ppb_z}_2^2 + 2 \lrVert{\nabla \log \ppb_z}_2 \lrVert{\nabla \log V}_2 + \lrVert{\nabla \log V}_2 \\
    \le{} & B'^2_{f^{-1}} B'^2_{\log p} + 2 d B'^3_{f^{-1}} B''_f B'_{\log p} + \lrVert{\nabla \log V}_2^2,
    \label{eqn:bound-2.1-orig}
  \end{align}
  where we have utilized \eqref{eqn:grad-logp-g} and \eqref{eqn:gradnorm-logv} in the last inequality.
  We consider bounding $\lrVert{\nabla \log V}_2^2$ separately.
  Using \eqref{eqn:grad-logv} (in the second equality below), we have:
  \begin{align}
    \lrVert{\nabla \log V}_2^2
    ={} & \lrvert{(\nabla \log V)\trs (\nabla \log V)}
    = \Big\vert \sum_c (K^c_{:c})\trs \sum_d K^d_{:d} \Big\vert \\
    ={} & \Big\vert \sum_{cd} K^c_{c:} K^d_{:d} \Big\vert
    \le \sum_{cd} \lrvert{K^c_{c:} K^d_{:d}} \\
    ={} & \sum_{cd} \lrvert{(J_{f^{-1}})_{c:} (\nabla\nabla\trs f_c) J_{f^{-1}}\trs J_{f^{-1}} (\nabla\nabla\trs f_d) (J_{f^{-1}})_{d:}\trs} \\
    \le{} & \sum_{cd} \lrvert{(J_{f^{-1}})_{c:} (J_{f^{-1}})_{d:}\trs} \lrVert{(\nabla\nabla\trs f_c) J_{f^{-1}}\trs J_{f^{-1}} (\nabla\nabla\trs f_d)}_2 \\
    \le{} & \sum_{cd} \lrvert{(J_{f^{-1}})_{c:} (J_{f^{-1}})_{d:}\trs} B'^2_{f^{-1}} B''^2_f
    = B'^2_{f^{-1}} B''^2_f \sum_{cd} \lrvert{(J_{f^{-1}} J_{f^{-1}}\trs)_{cd}} \\
    \le{} & d^{3/2} B'^2_{f^{-1}} B''^2_f \lrVert{J_{f^{-1}} J_{f^{-1}}\trs}_2
    \le d^{3/2} B'^4_{f^{-1}} B''^2_f,
    \label{eqn:bound-2.1.1}
  \end{align}
  where we have used the facts for general matrix $A$ and (column) vectors $\alpha, \beta$ that
  \begin{align}
    \lrvert{\alpha\trs A \beta} = \lrVert{\alpha (A \beta)\trs}_2 = \lrVert{\alpha \beta\trs A\trs}_2 \le \lrVert{\alpha \beta\trs}_2 \lrVert{A}_2 = \lrvert{\alpha\trs \beta} \lrVert{A}_2
    \label{eqn:fact-1}
  \end{align}
  in the fifth last inequality, and that
  \begin{align}
    \sum_{cd} \lrvert{A_{cd}} \le \sqrt{d^2} \sqrt{\sum_{cd} \lrvert{A_{cd}}^2} = d \lrVert{A}_F \le d^{3/2} \lrVert{A}_2
    \label{eqn:fact-2}
  \end{align}
  in the second last inequality.
  Substituting \eqref{eqn:bound-2.1.1} into \eqref{eqn:bound-2.1-orig}, we have:
  \begin{align}
    \lrVert{\nabla \log \ppb_z V}_2^2
    \le{} & B'^2_{f^{-1}} B'^2_{\log p} + 2 d B'^3_{f^{-1}} B''_f B'_{\log p} + d^{3/2} B'^4_{f^{-1}} B''^2_f.
    \label{eqn:bound-2.1}
  \end{align}

  For the second norm in the bracket of the r.h.s of \eqref{eqn:px-resd}, similar to \eqref{eqn:bound-1.2}, we have:
  \begin{align}
    \lrVert{\nabla\nabla\trs \log \ppb_z}_2
    \le B'^2_{f^{-1}} \big( B''_{\log p} + d B'_{f^{-1}} B''_f B'_{\log p} \big).
    \label{eqn:bound-2.2}
  \end{align}

  The third norm $\lrVert{ \nabla\nabla\trs \log V }_2$ in the bracket of the r.h.s of \eqref{eqn:px-resd} needs some more effort.
  From \eqref{eqn:grad-logv}, we have $\partial_b \log V = -\sum_{cij} (J_{f^{-1}})_{ci} (\partial_i \partial_j f_c) (J_{f^{-1}})_{bj}$, thus
  \begin{align}
    \partial_a \partial_b \log V ={} & -\sum_{cij} \partial_a (J_{f^{-1}})_{ci} (\partial_i \partial_j f_c) (J_{f^{-1}})_{bj}
    - \sum_{cij} (J_{f^{-1}})_{ci} (\partial_i \partial_j f_c) \partial_a (J_{f^{-1}})_{bj} \\
    & {} - \sum_{cij} (J_{f^{-1}})_{ci} \partial_a (\partial_i \partial_j f_c) (J_{f^{-1}})_{bj} \\
    ={} & -\sum_{cij} (\partial_a \partial_c f^{-1}_i) (\partial_i \partial_j f_c) (J_{f^{-1}})_{bj}
    - \sum_{cij} (J_{f^{-1}})_{ci} (\partial_i \partial_j f_c) (\partial_a \partial_b f^{-1}_j) \\
    & {} - \sum_{cijk} (J_{f^{-1}})_{ci} (\partial_a f^{-1}_k) (\partial_k \partial_i \partial_j f_c) (J_{f^{-1}})_{bj} \\
    ={} & \sum_{cijd} (J_{f^{-1}})_{di} K^d_{ac} (\partial_i \partial_j f_c) (J_{^{-1}})_{bj}
    + \sum_{cijd} (J_{f^{-1}})_{ci} (\partial_i \partial_j f_c) (J_{f^{-1}})_{dj} K^d_{ab} \\*
    & {} - \sum_{cijk} (J_{f^{-1}})_{ci} (\partial_k \partial_i \partial_j f_c) (J_{f^{-1}})_{ak} (J_{f^{-1}})_{bj} \\
    ={} & \sum_{cd} K^d_{ac} K^c_{db}
    + \sum_{cd} K^c_{cd} K^d_{ab}
    - \sum_{cijk} (J_{f^{-1}})_{ci} (\partial_k \partial_i \partial_j f_c) (J_{f^{-1}})_{ak} (J_{f^{-1}})_{bj},
  \end{align}
  where we have used \eqref{eqn:hess-finv} in the third equality for the first two terms.
  In matrix form, we have:
  \begin{align}
    \nabla\nabla\trs \log V
    = \sum_{cd} K^d_{:c} K^c_{d:} + \sum_{cd} K^c_{cd} K^d
    - \sum_{cijk} (J_{f^{-1}})_{ci} (\partial_k \partial_i \partial_j f_c) (J_{f^{-1}})_{:k} (J_{f^{-1}})_{:j}\trs.
  \end{align}

  We now bound the norms of the three terms in turn.
  For the first term,
  \begin{align}
    & \Big\Vert \sum_{cd} K^d_{:c} K^c_{d:} \Big\Vert_2
    \le \sum_{cd} \lrVert{K^d_{:c} K^c_{d:}}_2
    = \sum_{cd} \lrvert{K^c_{d:} K^d_{:c}} \\
    ={} & \sum_{cd} \lrvert{ (J_{f^{-1}})_{d:} (\nabla\nabla\trs f_c) J_{f^{-1}}\trs J_{f^{-1}} (\nabla\nabla\trs f_d) (J_{f^{-1}})_{c:}\trs } \\
    \le{} & \sum_{cd} \lrvert{ (J_{f^{-1}})_{d:} (J_{f^{-1}})_{c:}\trs } \lrVert{ (\nabla\nabla\trs f_c) J_{f^{-1}}\trs J_{f^{-1}} (\nabla\nabla\trs f_d) }_2 \\
    \le{} & B'^2_{f^{-1}} B''^2_f \sum_{cd} \lrvert{(J_{f^{-1}} J_{f^{-1}}\trs)_{dc}}
    \le d^{3/2} B'^2_{f^{-1}} B''^2_f \lrVert{J_{f^{-1}} J_{f^{-1}}\trs}_2 \\
    \le{} & d^{3/2} B'^4_{f^{-1}} B''^2_f,
    \label{eqn:bound-2.3.1}
  \end{align}
  where we have used \eqref{eqn:fact-1} in the fourth last inequality and \eqref{eqn:fact-2} in the second last inequality.
  For the second term,
  \begin{align}
    & \Big\Vert \sum_{cd} K^c_{cd} K^d \Big\Vert_2
    \le \sum_{cd} \lrvert{K^c_{cd}} \lrVert{K^d}_2
    \le B'^2_{f^{-1}} B''_f \sum_{cd} \lrvert{K^c_{cd}} \\
    \le{} & d^{1/2} B'^2_{f^{-1}} B''_f \sum_c \sqrt{\sum_d \lrvert{K^c_{cd}}^2}
    = d^{1/2} B'^2_{f^{-1}} B''_f \sum_c \lrVert{K^c_{c:}}_2 \\
    \le{} & d^{1/2} B'^2_{f^{-1}} B''_f \sum_c \lrVert{(J_{f^{-1}})_{c:}}_2 \lrVert{(\nabla\nabla\trs f_c) J_{f^{-1}}\trs}_2
    \le d^{1/2} B'^3_{f^{-1}} B''^2_f \sum_c \lrVert{(J_{f^{-1}})_{c:}}_2 \\
    \le{} & d^{3/2} B'^4_{f^{-1}} B''^2_f,
    \label{eqn:bound-2.3.2}
  \end{align}
  where we have used \eqref{eqn:gradnorm-sumJrows} in the last inequality.
  For the third term,
  \begin{align}
    & \Big\Vert \sum_{cijk} (J_{f^{-1}})_{ci} (\partial_k \partial_i \partial_j f_c) (J_{f^{-1}})_{:k} (J_{f^{-1}})_{:j}\trs \Big\Vert_2 \\
    \le{} & \sum_{cijk} \lrvert{(J_{f^{-1}})_{ci} (\partial_k \partial_i \partial_j f_c)} \lrVert{(J_{f^{-1}})_{:k} (J_{f^{-1}})_{:j}\trs}_2
    \le B'''_f \sum_{ci} \lrvert{(J_{f^{-1}})_{ci}} \sum_{jk} \lrVert{(J_{f^{-1}})_{:k} (J_{f^{-1}})_{:j}\trs}_2 \\
    \le{} & d^{3/2} B'''_f \lrVert{J_{f^{-1}}}_2 \sum_{jk} \lrvert{(J_{f^{-1}})_{:k}\trs (J_{f^{-1}})_{:j}}
    \le d^{3/2} B'''_f B'_{f^{-1}} \sum_{jk} \lrvert{(J_{f^{-1}}\trs J_{f^{-1}})_{kj}} \\
    \le{} & d^3 B'''_f B'_{f^{-1}} \lrVert{J_{f^{-1}}\trs J_{f^{-1}}}_2
    \le d^3 B'''_f B'^3_{f^{-1}},
    \label{eqn:bound-2.3.3}
  \end{align}
  where we have used \eqref{eqn:fact-2} in the fourth last and second last inequalities.

  Finally, by assembling Eqs.~(\ref{eqn:bound-2.1}, \ref{eqn:bound-2.2}, \ref{eqn:bound-2.3.1}, \ref{eqn:bound-2.3.2}, \ref{eqn:bound-2.3.3}) into \eqref{eqn:px-resd}, we have:
  \begin{align}
    \lrvert{ p(x) - (\ppb_z V)(x) }
    \le{} & \frac{1}{2} \bbE[\mu\trs \mu] B_p B'^{d}_{f^{-1}} \big(
      B'^2_{f^{-1}} B'^2_{\log p} + 2 d B'^3_{f^{-1}} B''_f B'_{\log p} + d^{3/2} B'^4_{f^{-1}} B''^2_f \\
      & {} + B'^2_{f^{-1}} ( B''_{\log p} + d B'_{f^{-1}} B''_f B'_{\log p} )
      + 2 d^{3/2} B'^4_{f^{-1}} B''^2_f + d^3 B'''_f B'^3_{f^{-1}} \big) \\
    ={} & \frac{1}{2} \bbE[\mu\trs \mu] B_p B'^{d+2}_{f^{-1}} \big(
      B'^2_{\log p} + B''_{\log p} + 3 d B'_{f^{-1}} B''_f B'_{\log p} \\
      & {} + 3 d^{3/2} B'^2_{f^{-1}} B''^2_f + d^3 B'''_f B'_{f^{-1}} \big).
  \end{align}
  So given the condition \bfiii, this residual can be neglected.
\end{proof}

\subsection{Proof of the OOD Generalization Error Bound Thm.~\ref{thm:ood}} \label{supp:proofs-ood}

We give the following more detailed version of Thm.~\ref{thm:ood} and prove it.
The theorem in the main context corresponds to conclusion \bfii below (\ie, \eqref{eqn:ood-expc-orig-repr} 
below recovers \eqref{eqn:ood}),
by taking the \ourmodels $p'$, $p$ and $\ppt$,
as the semantic-identified \ourmodel $p$ on the training domain, and the ground-truth \ourmodels $p^*$ and $\ppt^*$ on the training and test domains, respectively.
In the theorem in the main context, the semantic-identification requirement on the learned \ourmodel $p$ is to guarantee that it is semantic-equivalent to the ground-truth \ourmodel $p^*$ on the training domain, 
so that the condition in conclusion \bfii below is satisfied.
\begin{theoremnum}{\ref*{thm:ood}'}[OOD generalization error] \label{thm:ood-formal}
  Let Assumption~\ref{assm:anm-bij} hold.
  \bfemi Consider two \ourmodels $p$ and $\ppt$ that share the same generative mechanisms $p(x|s,v)$ and $p(y|s)$ but have different priors $p_{s,v}$ and $\ppt_{s,v}$.
  Then up to $O(\sigma_\mu^2)$ where $\sigma_\mu^2 := \bbE[\mu\trs \mu]$, we have for any $x \in \supp(p_x)\cap\supp(\ppt_x)$,
  \begin{align}
    \lrvert{\bbE[y|x] - \bbEt[y|x]}
    \le{} & \sigma_\mu^2 \lrVert{\nabla g}_2 \lrVert{J_{f^{-1}}}_2^2 \lrVert{\nabla \log (p_{s,v}/\ppt_{s,v})}_2 \Big|_{(s,v)=f^{-1}(x)},
    \label{eqn:ood-same}
  \end{align}
  where $J_{f^{-1}}$ is the Jacobian of $f^{-1}$. 
  Further assume that the bounds $B$'s defined in \emph{Thm.~\ref{thm:id-formal}\bfiii} hold.
  Then the error is negligible for any $x \in \supp(p_x)\cap\supp(\ppt_x)$ if $\frac{1}{\sigma_\mu^2} \gg B'_{\log p} B'_g B'^2_{f^{-1}}$, and:
  \begin{align}
    \bbE_{\ppt(x)} \lrvert{\bbE[y|x] - \bbEt[y|x]}^2
    \le{} & \sigma_\mu^4 B'^2_g B'^4_{f^{-1}} \bbE_{\ppt_{s,v}} \lrVert{\nabla \log (p_{s,v}/\ppt_{s,v})}_2^2 \label{eqn:ood-expc-orig} \\
    ={} & \sigma_\mu^4 B'^2_g B'^4_{f^{-1}} \bbE_{\ppt_{s,v}} [2 \Delta \log p_{s,v} - \Delta \log \ppt_{s,v} + \lrVert{\nabla \log p_{s,v}}_2^2]
    \label{eqn:ood-expc}
  \end{align}
  if $\supp(p_x) = \supp(\ppt_x)$, where $\Delta$ denotes the Laplacian operator.

  \bfemii Let $p'$ be a \ourmodel that is semantic-equivalent to the \ourmodel $p$ introduced in \bfemi.
  Then up to $O(\sigma_\mu^2)$, we have for any $x \in \supp(p'_x)\cap\supp(\ppt_x)$,
  \begin{align}
    \lrvert{\bbE'[y|x] - \bbEt[y|x]}
    \le \sigma_\mu^2 \lrVert{\nabla g'}_2 \lrVert{J_{f'^{-1}}}_2^2 \lrVert{\nabla \log (p'_{s,v}/\ppt'_{s,v})}_2 \Big|_{(s,v)=f'^{-1}(x)},
    \label{eqn:ood-repr}
  \end{align}
  where $\ppt'_{s,v} := \Phi_\#[\ppt_{s,v}]$ is the prior of \ourmodel $\ppt$ under the parameterization of \ourmodel $p'$,
  derived as the pushed-forward distribution by the reparameterization $\Phi := f'^{-1} \circ f$ from $p$ to $p'$.
  Similarly,
  \begin{align}
    \bbE_{\ppt(x)} \lrvert{\bbE'[y|x] - \bbEt[y|x]}^2
    \le{} & \sigma_\mu^4 B'^2_g B'^4_{f^{-1}} \bbE_{\ppt'_{s,v}} \lrVert{\nabla \log (p'_{s,v}/\ppt'_{s,v})}_2^2 \label{eqn:ood-expc-orig-repr} \\
    ={} & \sigma_\mu^4 B'^2_g B'^4_{f^{-1}} \bbE_{\ppt'_{s,v}} [2 \Delta \log p'_{s,v} - \Delta \log \ppt'_{s,v} + \lrVert{\nabla \log p'_{s,v}}_2^2].
    \label{eqn:ood-expc-repr}
  \end{align}
\end{theoremnum}

In the expected OOD generalization error in Eqs.~(\ref{eqn:ood-expc},~\ref{eqn:ood-expc-repr}), the term $\bbE_{\ppt_{s,v}} [2 \Delta \log p_{s,v} - \Delta \log \ppt_{s,v} + \lrVert{\nabla \log p_{s,v}}_2^2]$ is actually the score matching objective (Fisher divergence)~\citep{hyvarinen2005estimation} that measures the difference between $\ppt_{s,v}$ and $p_{s,v}$.
For Gaussian priors $p(s,v) = \clN(0, \Sigma)$ and $\ppt(s,v) = \clN(0, \Sigmat)$, the term reduces to the matrix trace, $\tr(-2 \Sigma^{-1} + \Sigmat^{-1} + \Sigma^{-1} \Sigmat \Sigma^{-1})$.
For $\Sigma = \Sigmat$, the term vanishes.

For conclusion \bfii, note that since $p$ and $p'$ are semantic-equivalent, we have $p'_x = p_x$ and $\bbE'[y|x] = \bbE[y|x]$ (from Lemma~\ref{lem:repar-same-pxy}).
So Eqs.~(\ref{eqn:ood-same},~\ref{eqn:ood-repr}) and Eqs.~(\ref{eqn:ood-expc},~\ref{eqn:ood-expc-repr}) bound the same quantity.
Equation~(\ref{eqn:ood-repr}) expresses the bound using the structures of the \ourmodel $p'$.
It is considered since recovering the exact \ourmodel $p$ from $(x,y)$ data is impractical and we can only learn a \ourmodel $p'$ that is semantic-equivalent to $p$.

\begin{proof}
  Following the proof~\ref{supp:proofs-id} of Thm.~\ref{thm:id-formal}, we assume the additive noise variables $\mu$ and $\nu$ (for continuous $y$) have zero mean without loss of generality, and we denote $z := (s,v)$.

  \paragraph{Proof under condition \bfi.}
  Under the assumptions, we have \eqref{eqn:Ey1x-expd} in the proof~\ref{supp:proofs-id} of Thm.~\ref{thm:id-formal} hold.
  Noting that the two \ourmodels share the same $\ggb$ and $V$ (since they share the same $p(x|s,v)$ and $p(y|s)$ thus $f$ and $g$), we have for any $x \in \supp(p_x)\cap\supp(\ppt_x)$,
  \begin{align}
    \bbE[y|x] ={} & \ggb + \frac{1}{2} \bbE_{p(\mu)} \big[ \mu\trs \big( (\nabla \log \ppb_z V) \nabla \ggb \trs
    + \nabla \ggb (\nabla \log \ppb_z V)\trs + \nabla\nabla\trs \ggb \big) \mu \big] + O(\sigma_\mu^3),
    \label{eqn:Ey1x-expd-p} \\
    \bbEt[y|x] ={} & \ggb + \frac{1}{2} \bbE_{p(\mu)} \big[ \mu\trs \big( (\nabla \log \pptb_z V) \nabla \ggb \trs
    + \nabla \ggb (\nabla \log \pptb_z V)\trs + \nabla\nabla\trs \ggb \big) \mu \big] + O(\sigma_\mu^3),
    \label{eqn:Ey1x-expd-ppt}
  \end{align}
  where we have similarly defined $\pptb_z := \ppt_z \circ f^{-1}$.
  By subtracting the two equations, we have that up to $O(\sigma_\mu^2)$,
  \begin{align}
    \lrvert{\bbE[y|x] - \bbEt[y|x]}
    ={} & \frac{1}{2} \Big\vert \bbE_{p(\mu)} \big[ \mu\trs \big( \nabla \log (\ppb_z/\pptb_z)  \nabla \ggb \trs + \nabla \ggb  \nabla \log (\ppb_z/\pptb_z) \trs \big) \mu \big] \Big\vert \\
    \le{} & \frac{1}{2} \bbE_{p(\mu)} \big[ \big\vert \mu\trs \big( \nabla \log (\ppb_z/\pptb_z)  \nabla \ggb \trs + \nabla \ggb  \nabla \log (\ppb_z/\pptb_z) \trs \big) \mu \big\vert \big] \\
    \le{} & \frac{1}{2} \bbE_{p(\mu)} \big[ \lrVert{\mu}_2^2 \big( \lrVert{\nabla \log (\ppb_z/\pptb_z)  \nabla \ggb \trs}_2 + \lrVert{\nabla \ggb  \nabla \log (\ppb_z/\pptb_z) \trs}_2 \big) \big] \\
    ={} & \lrvert{ \nabla \ggb \trs \nabla \log (\ppb_z/\pptb_z) } \bbE[\mu\trs \mu].
    \label{eqn:ood-deduc-0}
  \end{align}
  The multiplicative factor to $\bbE[\mu\trs \mu]$ on the right hand side can be further bounded by:
  \begin{align}
    \lrvert{ \nabla \ggb \trs  \nabla \log (\ppb_z/\pptb_z) }
    ={} & \lrvert{ (J_{(f^{-1})^\clS} \nabla g)\trs (J_{f^{-1}} \nabla \log (p_z/\ppt_z)) } \\
    ={} & \lrvert{ \nabla g \trs J_{(f^{-1})^\clS}\trs J_{f^{-1}} \nabla \log (p_z/\ppt_z) } \\
    ={} & \lrvert{ ((\nabla g)\trs, 0_{d_\clV}\trs) J_{f^{-1}}\trs J_{f^{-1}} \nabla \log (p_z/\ppt_z) } \\
    \le{} & \lrVert{\nabla g}_2 \lrVert{J_{f^{-1}}}_2^2 \lrVert{\nabla \log (p_z/\ppt_z)}_2,
    \label{eqn:ood-deduc-1}
  \end{align}
  where $\nabla g$ and $\nabla \log (p_z/\ppt_z)$ are evaluated at $z = f^{-1}(x)$.
  This gives:
  \begin{align}
    \lrvert{\bbE[y|x] - \bbEt[y|x]}
    \le \sigma_\mu^2 \lrVert{\nabla g}_2 \lrVert{J_{f^{-1}}}_2^2 \lrVert{\nabla \log (p_z/\ppt_z)}_2,
  \end{align}
  \ie \eqref{eqn:ood-same} in conclusion \bfi.
  When the bounds $B$'s in Thm.~\ref{thm:id-formal}\bfiii hold, we further have:
  \begin{align}
    \lrvert{\bbE[y|x] - \bbEt[y|x]}
    \le{} & \sigma_\mu^2 \lrVert{\nabla g}_2 \lrVert{J_{f^{-1}}}_2^2 \lrVert{\nabla \log p_z - \nabla \log \ppt_z}_2 \\
    \le{} & \sigma_\mu^2 \lrVert{\nabla g}_2 \lrVert{J_{f^{-1}}}_2^2 (\lrVert{\nabla \log p_z}_2 + \lrVert{\nabla \log \ppt_z}_2) \\
    \le{} & 2 \sigma_\mu^2 B'_g B'^2_{f^{-1}} B'_{\log p}.
  \end{align}
  So when $\frac{1}{\sigma_\mu^2} \gg B'_{\log p} B'_g B'^2_{f^{-1}}$, this difference is negligible for any $x \in \supp(p_x)\cap\supp(\ppt_x)$.

  We now turn to the expected OOD generalization error \eqref{eqn:ood-expc} in conclusion \bfi.
  When $\supp(p_x) = \supp(\ppt_x)$, \eqref{eqn:ood-same} hold on $\ppt_x$.
  Together with the bounds in Thm.~\ref{thm:id-formal}\bfiii, we have:
  \begin{align}
    \bbE_{\ppt(x)} \lrvert{\bbE[y|x] - \bbEt[y|x]}^2
    \le{} & \sigma_\mu^4 B'^2_g B'^4_{f^{-1}} \bbE_{\ppt(x)} \lrVert{\nabla \log (p_z/\ppt_z) \big|_{z = f^{-1}(x)}}_2^2 \\
    ={} & \sigma_\mu^4 B'^2_g B'^4_{f^{-1}} \bbE_{\ppt_z} \lrVert{\nabla \log (p_z/\ppt_z)}_2^2,
  \end{align}
  where the equality holds due to the generating process of the model.
  Note that the term $\bbE_{\ppt_z} \lrVert{\nabla \log (p_z/\ppt_z)}_2^2$ therein is the score matching objective (Fisher divergence).
  By \citet[Thm.~1]{hyvarinen2005estimation}, we can reformulate it as $\bbE_{\ppt_z} [2 \Delta \log p_z - \Delta \log \ppt_z + \lrVert{\nabla \log p_z}_2^2]$, so we have:
  \begin{align}
    \bbE_{\ppt(x)} \lrvert{\bbE[y|x] - \bbEt[y|x]}^2
    \le{} & \sigma_\mu^4 B'^2_g B'^4_{f^{-1}} \bbE_{\ppt_z} [2 \Delta \log p_z - \Delta \log \ppt_z + \lrVert{\nabla \log p_z}_2^2].
  \end{align}

  \paragraph{Proof under condition \bfii.}
  From \eqref{eqn:Ey1x-expd} in the proof~\ref{supp:proofs-id} of Thm.~\ref{thm:id-formal}, we have for \ourmodel $p'$ that for any $x \in \supp(p'_x)$ or equivalently $x \in \supp(p_x)$,
  \begin{align}
    \bbE'[y|x] ={} & \ggb' + \frac{1}{2} \bbE_{p(\mu)} \! \big[ \mu\trs \! \big( (\nabla \log \ppb'_z V') \nabla \ggb'{} \trs \!
    + \nabla \ggb' (\nabla \log \ppb'_z V')\trs \! + \nabla\nabla\trs \ggb' \big) \mu \big] + O(\sigma_\mu^3), \!
    \label{eqn:Ey1x-expd-p'}
  \end{align}
  where we have similarly defined $\ppb'_z := p'_z \circ f'^{-1}$ and $\ggb' := g' \circ (f'^{-1})^\clS$.
  Since $p$ and $p'$ are semantic-equivalent with reparameterization $\Phi$ from $p$ to $p'$, we have $p(y|s) = p'(y|\Phi^\clS(s,v))$ thus $g(s) = g'(\Phi^\clS(s,v))$ for any $v\in\clV$.
  So for any $x \in \supp(p_x)$ or equivalently $x \in \supp(p'_x)$, we have $g((f^{-1})^\clS(x)) = g'(\Phi^\clS((f^{-1})^\clS(x), (f^{-1})^\clV(x))) = g'(\Phi^\clS(f^{-1}(x))) = g'((f'^{-1})^\clS(f(f^{-1}(x)))) = g'((f'^{-1})^\clS(x))$, \ie, $\ggb = \ggb'$.
  For another fact, since $\ppt'_z := \Phi_\#[\ppt_z] = (f'^{-1} \circ f)_\#[\ppt_z]$ by definition, we have $f'_\#[\ppt'_z] = f_\#[\ppt_z]$, \ie, $\pptb'_z V' = \pptb_z V$.
  Subtracting Eqs.~(\ref{eqn:Ey1x-expd-p'},~\ref{eqn:Ey1x-expd-ppt}) and applying these two facts, we have up to $O(\sigma_\mu^2)$, for any $x \in \supp(p'_x)\cap\supp(\ppt_x)$,
  \begin{align}
    \lrvert{\bbE'[y|x] - \bbEt[y|x]}
    ={} & \frac{1}{2} \Big\vert \bbE_{p(\mu)} \big[ \mu\trs \big( \nabla \log (\ppb'_z/\pptb'_z)  \nabla \ggb'{} \trs + \nabla \ggb'  \nabla \log (\ppb'_z/\pptb'_z) \trs \big) \mu \big] \Big\vert \\
    \le{} & \lrvert{ \nabla \ggb'{} \trs \nabla \log (\ppb'_z/\pptb'_z) } \bbE[\mu\trs \mu],
    \label{eqn:ood-deduc-0'}
  \end{align}
  where the inequality follows \eqref{eqn:ood-deduc-0}.
  Using a similar result of \eqref{eqn:ood-deduc-1}, we have:
  \begin{align}
    \lrvert{\bbE'[y|x] - \bbEt[y|x]}
    \le \sigma_\mu^2 \lrVert{\nabla g'}_2 \lrVert{J_{f'^{-1}}}_2^2 \lrVert{\nabla \log (p'_z/\ppt'_z)}_2,
  \end{align}
  where $\nabla g'$ and $\nabla \log (p'_z/\ppt'_z)$ are evaluated at $z = f'^{-1}(x)$.
  This gives \eqref{eqn:ood-repr}.
  Derivation of Eqs.~(\ref{eqn:ood-expc-orig-repr},~\ref{eqn:ood-expc-repr}) is similar as in conclusion \bfi.
\end{proof}

\subsection{Proof of the Domain Adaptation Error Thm.~\ref{thm:da}} \label{supp:proofs-da}

To be consistent with the notation in the proofs, we prove the theorem by denoting the semantic-identified \ourmodel $p$ and the ground-truth \ourmodel $\ppt^*$ on the test domain as $p'$ and $\ppt$, respectively.

\begin{proof}
  The new prior $\ppt'(z)$ is learned by fitting unsupervised data from the test domain $\ppt(x)$.
  Applying the deduction in the proof~\ref{supp:proofs-id} of Thm.~\ref{thm:id-formal} to the test domain, we have that under any of the three conditions in Thm.~\ref{thm:id-formal},
  $\ppt(x) = \ppt'(x)$ indicates $f_\#[\ppt_z] = f'_\#[\ppt'_z]$.
  This gives $\ppt'_z = (f'^{-1} \circ f)_\#[\ppt_z] = \Phi_\#[\ppt_z]$.

  From \eqref{eqn:Ey1x-conv} in the same proof, we have that:
  \begin{align}
    & \ppt(x) \bbEt[y|x] = (f_\#[g \ppt_z] * p_\mu)(x) = ((f_\#[\ppt_z] \ggb) * p_\mu)(x), \\
    & \ppt'(x) \bbEt'[y|x] = (f'_\#[g' \ppt'_z] * p_\mu)(x) = ((f'_\#[\ppt'_z] \ggb') * p_\mu)(x).
  \end{align}
  From the proof~\ref{supp:proofs-ood} of Thm.~\ref{thm:ood-formal}\bfii (the paragraph under \eqref{eqn:Ey1x-expd-p'}), the semantic-equivalence between \ourmodels $p$ and $p'$ indicates that $\ggb = \ggb'$.
  So from the above two equations, we have $\ppt(x) \bbEt[y|x] = \ppt'(x) \bbEt'[y|x]$ (recall that $\ppt(x) = \ppt'(x)$ indicates $f_\#[\ppt_z] = f'_\#[\ppt'_z]$).
  Since $\ppt(x) = \ppt'(x)$ (that is how $\ppt'_z$ is learned), we have for any $x \in \supp(\ppt_x)$ or equivalently $x \in \supp(\ppt'_x)$,
  \begin{align} \label{eqn:da-gen}
    \bbEt'[y|x] = \bbEt[y|x].
  \end{align}
\end{proof}

\section{Alternative Identifiability Theory for \ourmodel} \label{supp:id-delta}

The presented identifiability theory, particularly Thm.~\ref{thm:id}, shows that the semantic-identifiability can be achieved in the deterministic limit ($\frac{1}{\sigma_\mu^2} \to \infty$), but does not quantitatively describe the extent of violation of the identifiability for a finite variance $\sigma_\mu^2$.
Here we define a ``soft'' version of semantic-equivalence and show that it can be achieved with a finite variance, with a trade-off between the ``softness'' and the variance.
\begin{definition}[$\delta$-semantic-dependency]
  For $\delta > 0$ and two \ourmodels $p$ and $p'$, we say that they are $\delta$-semantic-dependent, if there exists a homeomorphism $\Phi$ on $\clS\times\clV$ such that:
  \bfi $p(x|s,v) = p'(x|\Phi(s,v))$,
  \bfii $\sup_{v\in\clV} \lrVert{g(s) - g'(\Phi^\clS(s,v))}_2 \le \delta$ where we have denoted $g(s) := \bbE[y|s]$, and
  \bfiii $\sup_{v^{(1)}, v^{(2)} \in \clV} \lrVert{\Phi^\clS(s, v^{(1)}) - \Phi^\clS(s, v^{(2)})}_2 \le \delta$.
  \label{def:equiv-delta}
\end{definition}

In the definition, we have released the prior conversion requirement, and relaxed the exact likelihood conversion for $p(y|s)$ in \bfii and the $v$-constancy of $\Phi^\clS$ in \bfiii to allow an error bounded by $\delta$.
When $\delta = 0$, the $v$-constancy of $\Phi^\clS$ is exact, and under the additive noise Assumption~\ref{assm:anm-bij} we also have the exact likelihood conversion $p(y|s) = p'(y|\Phi^\clS(s,v))$ for any $v\in\clV$.
So $0$-semantic-dependency with the prior conversion requirement reduces to the semantic-equivalence.

Due to the quantitative nature, the binary relation cannot be made an equivalence relation but only a dependency.
Here, a dependency refers to a binary relation with reflexivity and symmetry, but no transitivity.
\begin{proposition}
  The $\delta$-semantic-dependency is a dependency relation if the function $g := \bbE[y|s]$ is bijective and its inverse $g^{-1}$ is $\frac{1}{2}$-Lipschitz.
  \label{prop:equiv-delta}
\end{proposition}
\begin{proof}
  Showing a dependency relation amounts to showing the following two properties.
  \begin{itemize}
    \item Reflexivity.
      For two identical \ourmodels $p$ and $p'$, we have $p(x|s,v) = p'(x|s,v)$ and $p(y|s) = p'(y|s)$.
      So the identity map as $\Phi$ obviously satisfies all the requirements in Def.~\ref{def:equiv-delta}.
    \item Symmetry.
      Let \ourmodel $p$ be $\delta$-semantic-dependent to \ourmodel $p'$ with homeomorphism $\Phi$.
      Obviously $\Phi^{-1}$ is also a homeomorphism.
      For any $(s',v') \in \clS\times\clV$, we have $p'(x|s',v') = p'(x|\Phi(\Phi^{-1}(s',v'))) = p(x|\Phi^{-1}(s',v'))$,
      and $\lrVert{g'(s') - g((\Phi^{-1})^\clS(s',v'))}_2 = \lrVert{g'(\Phi^\clS(s,v)) - g(s)}_2 \le \delta$ where we have denoted $(s,v) := \Phi^{-1}(s',v')$ here.
      So $\Phi^{-1}$ satisfies requirements \bfi and \bfii in Def.~\ref{def:equiv-delta}.

      For requirement \bfiii, we need the following fact:
      for any $s^{(1)}, s^{(2)} \in \clS$, $\lrVert{s^{(1)} - s^{(2)}}_2 = \lrVert{g^{-1}(g(s^{(1)})) - g^{-1}(g(s^{(2)}))}_2 \le \frac{1}{2} \lrVert{g(s^{(1)}) - g(s^{(2)})}_2$,
      where the inequality holds since $g^{-1}$ is $\frac{1}{2}$-Lipschitz.
      Then for any $s' \in \clS$, we have:
      \begin{align}
        & \sup_{v'^{(1)}, v'^{(2)} \in \clV} \lrVert{(\Phi^{-1})^\clS(s',v'^{(1)}) - (\Phi^{-1})^\clS(s',v'^{(2)})}_2 \\
        \le{} & \sup_{v'^{(1)}, v'^{(2)} \in \clV} \frac{1}{2} \lrVert{g\big( (\Phi^{-1})^\clS(s',v'^{(1)}) \big) - g\big( (\Phi^{-1})^\clS(s',v'^{(2)}) \big)}_2 \\
        ={} & \sup_{v'^{(1)}, v'^{(2)} \in \clV} \frac{1}{2} \lrVert{\Big( g\big( (\Phi^{-1})^\clS(s',v'^{(1)}) \big) - g'(s') \Big) - \Big( g\big( (\Phi^{-1})^\clS(s',v'^{(2)}) \big) - g'(s') \Big)}_2 \\
        \le{} & \sup_{v'^{(1)}, v'^{(2)} \in \clV} \frac{1}{2} \Big( \lrVert{g\big( (\Phi^{-1})^\clS(s',v'^{(1)}) \big) - g'(s')}_2 + \lrVert{g\big( (\Phi^{-1})^\clS(s',v'^{(2)}) \big) - g'(s')}_2 \Big) \\
        ={} & \frac{1}{2} \Big( \sup_{v'^{(1)}\in\clV} \lrVert{g\big( (\Phi^{-1})^\clS(s',v'^{(1)}) \big) - g'(s')}_2 + \sup_{v'^{(2)}\in\clV} \lrVert{g\big( (\Phi^{-1})^\clS(s',v'^{(2)}) \big) - g'(s')}_2 \Big) \\
        \le{} & \delta,
      \end{align}
      where in the last inequality we have used the fact that $\Phi^{-1}$ satisfies requirement \bfii.
      So $p'$ is $\delta$-semantic-dependent to $p$ via the homeomorphism $\Phi^{-1}$.
  \end{itemize}
\end{proof}

The corresponding $\delta$-semantic-identifiability result follows.
\begin{theorem}[$\delta$-semantic-identifiability]
  Assume the same as Thm.~\ref{thm:id-formal} and Prop.~\ref{prop:equiv-delta}, and let the bounds $B$'s defined in \emph{Thm.~\ref{thm:id-formal}\bfiii} hold.
  For two such \ourmodels $p$ and $p'$, if they have $p(x,y) = p'(x,y)$,
  then they are $\delta$-semantic-dependent for any $\delta \ge \sigma_\mu^2 B'^2_{f^{-1}} \big( 2 B'_{\log p} B'_g  + B''_g + 3 d B'_{f^{-1}} B''_f B'_g \big)$, where $d := d_\clS + d_\clV$.
  \label{thm:id-delta}
\end{theorem}
\begin{proof}
  Let $\Phi := f'^{-1} \circ f$, where $f$ and $f'$ are given by the two \ourmodels $p$ and $p'$ via the additive noise Assumption~\ref{assm:anm-bij}.
  We now show that $p$ and $p'$ are $\delta$-semantic-dependent via this $\Phi$ for any $\delta$ in the theorem.
  Obviously $\Phi$ is a homeomorphism on $\clS\times\clV$, and it satisfies requirement \bfi in Def.~\ref{def:equiv-delta} by construction due to \eqref{eqn:x-convert} in the proof~\ref{supp:proofs-id} of Thm.~\ref{thm:id-formal}.

  Consider requirement \bfii in Def.~\ref{def:equiv-delta}.
  Based on the same assumptions as Thm.~\ref{thm:id-formal}, we have \eqref{eqn:bound-1.3} hold for both \ourmodels:
  \begin{align}
    \max \lrbrace{ \lrVert{\bbE[y|x] - \ggb(x)}_2, \lrVert{\bbE'[y|x] - \ggb'(x)}_2 }
    \le \sigma_\mu^2 B'^2_{f^{-1}} \big( B'_{\log p} B'_g  + \frac{1}{2} B''_g + \frac{3}{2} d B'_{f^{-1}} B''_f B'_g \big),
  \end{align}
  where we have denoted $\sigma_\mu^2 := \bbE[\mu\trs \mu]$.
  Since both \ourmodels induce the same $p(y|x)$, so $\bbE[y|x] = \bbE'[y|x]$. This gives:
  \begin{align}
    & \lrVert{\ggb(x) - \ggb'(x)}_2 = \lrVert{\big( \bbE'[y|x] - \ggb'(x) \big) - \big( \bbE[y|x] - \ggb(x) \big)}_2 \\
    \le{} & \lrVert{\bbE'[y|x] - \ggb'(x)}_2 + \lrVert{\bbE[y|x] - \ggb(x)}_2 \\
    \le{} & \sigma_\mu^2 B'^2_{f^{-1}} \big( 2 B'_{\log p} B'_g  + B''_g + 3 d B'_{f^{-1}} B''_f B'_g \big).
  \end{align}
  So for any $(s,v) \in \clS\times\clV$, by denoting $x := f(s,v)$, we have:
  \begin{align}
    & \lrVert{g(s) - g'(\Phi^\clS(s,v))}_2 = \lrVert{g((f^{-1})^\clS(x)) - g'((f'^{-1})^\clS(f(s,v)))}_2 = \lrVert{\ggb(x) - \ggb'(x)}_2 \\
    \le{} & \sigma_\mu^2 B'^2_{f^{-1}} \big( 2 B'_{\log p} B'_g  + B''_g + 3 d B'_{f^{-1}} B''_f B'_g \big).
  \end{align}
  So the requirement is satisfied.

  For requirement \bfiii, note from the proof of Prop.~\ref{prop:equiv-delta} that when $g$ is bijective and its inverse is $\frac{1}{2}$-Lipschitz, requirement \bfii implies requirement \bfiii.
  So this $\Phi$ is a homeomorphism that makes $p$ $\delta$-semantic-dependent to $p'$ for any $\delta \ge \sigma_\mu^2 B'^2_{f^{-1}} \big( 2 B'_{\log p} B'_g  + B''_g + 3 d B'_{f^{-1}} B''_f B'_g \big)$.
\end{proof}

Note that although the $\delta$-semantic-dependency does not have transitivity, the above theorem is still informative:
for any two \ourmodels sharing the same data distribution, particularly for a well-learned \ourmodel $p$ and the ground-truth \ourmodel $p^*$, the likelihood conversion error $\sup_{(s,v) \in \clS\times\clV} \lrVert{g(s) - g'(\Phi^\clS(s,v))}_2$, and the degree of mixing $v$ into $s$, measured by $\sup_{v^{(1)}, v^{(2)} \in \clV} \lrVert{\Phi^\clS(s, v^{(1)}) - \Phi^\clS(s, v^{(2)})}_2$, are bounded by $\sigma_\mu^2 B'^2_{f^{-1}} \big( 2 B'_{\log p} B'_g  + B''_g + 3 d B'_{f^{-1}} B''_f B'_g \big)$.

\section{More Explanations on the Model} \label{supp:model}

\paragraph{Explanations on our model.}
We see the data generating process as coming up with a conceptual latent factors $(s,v)$ first, and then generating both $x$ and $y$ based on the factors.
A prototyping example is that a photographer takes an image $x$ of an object and meanwhile gives a label $y$ to it, based on conceptual features $(s,v)$ in the scene (\eg, shape, color, texture, orientation and pose of the object, background objects and environment, illumination during imaging).
The image $x$ is produced by assembling these factors $(s,v)$ in the scene and passing the reflected light through a camera,
and the label $y$ is produced by processing causally relevant factors $s$ (\eg, object shape, texture) by the photographer.
Under this view, intervening the image $x$ is to break the imaging process (\eg, by malfunctioning the camera by breaking a sensor unit or making the sensor noisy), which does not alter the latent factors $(s,v)$ and the labeling process, hence also the label $y$.
Similarly, intervening the label $y$ is to break the labeling process (\eg, by reforming the labeling rule or randomly flipping the labels), which does not alter the latent factors $(s,v)$ and the imaging process, hence also the image $x$.
On the other hand, intervening the latent factors $(s,v)$ (\eg, by replacing the object with a different one at the imaging and labeling moment) may change both $x$ and $y$ through the imaging and labeling processes.
This verifies the model in Fig.~\ref{fig:gen-sv} by checking its causal implications.

This view of the data generating process is also adopted and promoted by popular existing works.
\citet{mcauliffe2008supervised} treat both a document and its label be generated by the involved topics in the document (represented as a topic proportion), which is an abstract latent factor.
\citet[Sec.~1.4]{peters2017elements}; \citet{kilbertus2018generalization} view the generation of an OCR dataset under a causal perspective as the writer first comes up with an intension to write a character, and then writes down the character and gives its label based on the intension.
\citet{teshima2020few} treat both an image and its label be produced from a set of latent factors.
This view of the data generating process is also natural for medical image datasets, where the label may be diagnosed based on more fundamental features (\eg, PCR test results showing the pathogen) that are not included in the dataset but actually cause the medical image.

On the labeling process from images that one would commonly think of, we also view it as a $s \to y$ process.
Human directly knows the critical semantic feature $s$ (\eg, the shape and position of each stroke) by seeing the image, through the nature gift of the vision system~\citep{biederman1987recognition}.
The label is given by processing the feature (\eg, the angle between two linear strokes, the position of a circular stroke relative to a linear stroke), which is a $s \to y$ process.

The causal graph in Fig.~\ref{fig:gen-sv} implies that $x \pperp y \mid s$.
However, this does not indicate that the semantic factor $s$ generates an image $x$ regardless of the label $y$.
Given $s$, the generated image is dictated to hold the given semantics regardless of randomness, so the statistical independence does not mean semantic irrelevance.
If an image $x$ is given, the corresponding label is given by $p(y|x)$, which is $\int p(s|x) p(y|s) \dd s$ by the causal graph.
So the semantic concept to cause the label through $p(y|s)$, is inferred from the image through $p(s|x)$.

\newcommand{\tx}{\mathrm{tx}}
\newcommand{\rx}{\mathrm{rx}}
\paragraph{Comparison with the graph $y_\tx \to s \to x \to y_\rx$.}
One may consider this graph as a communication channel, where $y_\tx$ is a transmitted signal and $y_\rx$ is the received signal.

If the observed label $y$ is treated as $y_\tx$, the graph then implies $y \to s$.
This is argued at the end of item~\bftwo in Sec.~\ref{sec:model} that it may make unreasonable implications.
Moreover, the graph also implies that $y$ is a cause of $x$, as is challenged in item~\bfone in Sec.~\ref{sec:model}.
The unnatural implications arise since intervening $y$ is different from intervening the ``ground-truth'' label.
We consider $y$ as an observation that may be noisy, while the ``ground-truth label'' is never observed: one cannot tell if the labels at hand are noise-corrupted, based on the dataset alone.
For example, the label of either image in Fig.~\ref{fig:ambiguity} may be given by a labeler's random guess.
Our adopted causal direction $s \to y$ is consistent with these examples and is also argued and adopted by \citet{mcauliffe2008supervised}; \citet[Sec.~1.4]{peters2017elements}; \citet{kilbertus2018generalization}; \citet{teshima2020few}.

If the observed label $y$ is treated as $y_\rx$, the graph then implies $x \to y$, as is challenged in item~\bfone in Sec.~\ref{sec:model}.
It is also argued by \citet{scholkopf2012causal}; \citet[Sec.~1.4]{peters2017elements}; \citet{kilbertus2018generalization}.
Treating the observed label $y$ as $y_\rx$ and $y_\tx$ as the ``ground-truth'' label may be the motivation of this graph.
But the graph implies $y_\tx \pperp y_\rx \mid x$, that is, $p(y_\tx|x, y_\rx) = p(y_\tx|x)$ and $p(y_\rx|x, y_\tx) = p(y_\rx|x)$.
So modeling $y_\tx$ (resp. $y_\rx$) does not benefit predicting $y_\rx$ (resp. $y_\tx$) from $x$.

\section{More Related Work} \label{supp:relw}

\textbf{Generative supervised learning}
is not new~\citep{mcauliffe2008supervised, kingma2014semi}, 
but most works do not consider the encoded causality. 
Other works consider solving causality tasks, notably causal/treatment effect estimation~\citep{louizos2017causal, yao2018representation, wang2019blessings}.
The task does not focus on OOD prediction, and requires labels for both treated and controlled groups.

\textbf{Causality with latent variable}
has been considered in a rich literature~\citep{verma1991equivalence, spirtes2000causation, richardson2002ancestral, hoyer2008estimation, shpitser2014introduction}, while most works focus on the consequence on observation-level causality. 
Others consider identifying the latent variable.
\citet{janzing2009identifying, lee2019leveraging} show the identifiability 
under additive noise or similar assumptions.
For discrete data, a ``simple'' latent variable can be identified under various specifications~\citep{janzing2011detecting, sgouritsa2013identifying, kocaoglu2018entropic}.
\citet{romeijn2018intervention} consider using interventional datasets for identification.
Over these works, we step further to separate and identify the latent variable as semantic and variation factors, and show the benefit for OOD prediction. 

\section{Relation to Existing Domain Adaptation Theory} \label{supp:da-dir}

In this section, to align with the domain adaptation (DA) literature, we call ``training/test domain'' as ``source/target domain'',
and use $p(x,y)$ and $\ppt(x,y)$ to denote the underlying data-generating distributions $p^*(x,y)$ and $\ppt^*(x,y)$ on the source and target domains, respectively.
In a DA task, supervised data from $p(x,y)$ on the source domain are available, but on the target domain, only unsupervised data from $\ppt(x) = \int \ppt(x,y) \dd y$ \footnote{
  Under the general definition of an integral (\eg, \citet[p.211]{billingsley2012probability}), it also allows a discrete $\clY$, in which case $\ud y$ is the counting measure and the integral reduces to a summation.
} are available.
The goal is to find a labeling function $h: \clX \to \clY$ within a hypothesis space $\clH$ that minimizes the target-domain risk $\Rt(h) := \bbE_{\ppt(x,y)}[\ell(h(x), y)]$ defined by a loss function $\ell: \clY \times \clY \to \bbR$.

\paragraph{General DA theory}
\newcommand{\clHt}{{\tilde \clH}}

Since $\ppt(x,y)$ is not accessible, it is of practical interest to consider the source-domain risk $R(h)$ and investigate its relation to $\Rt(h)$.
\citet[Thm.~1]{ben2010theory} give a bound relating the two risks:
\begin{align}
  \Rt(h) \le{} & R(h) + 2 d_1(p_x, \ppt_x) \\
  & {}+ \min\{\bbE_{p(x)} [\lrvert*[big]{h^*(x) - \hht^*(x)}], \bbE_{\ppt(x)} [\lrvert*[big]{h^*(x) - \hht^*(x)}]\},
  \label{eqn:da-bound-tv} \\
  \text{where: } &
  d_1(p_x, \ppt_x) := \sup_{X \in \scX} \lrvert{p_x[X] - \ppt_x[X]}
\end{align}
is the \emph{total variation} between the two distributions, $\scX$ denotes the sigma-field on $\clX$,
and $h^* \in \argmin_{h \in \clH} R(h)$ and $\hht^* \in \argmin_{\hht \in \clH} \Rt(\hht)$ are the oracle labeling functions on the source and target domains, respectively (\eg, $h^*(x) = \bbE[y|x]$ and $\hht^*(x) = \bbEt[y|x]$ if $\supp(p_x) = \supp(\ppt_x)$).
Note that as oracle labeling functions, $h^*$ and $\hht^*$ are two \emph{certain} but not \emph{any} risk minimizers.
The second and third terms on the r.h.s measure the domain difference in terms of the distribution on $x$ and the correspondence of $y$ on $x$, respectively.
\citet[Thm.~4.1]{zhao2019learning} give a similar bound in the case of binary classification $\clY = \{0,1\}$, in terms of the $\clHt$-divergence $d_\clHt$ in place of the total variance $d_1$,
which is defined as $d_\clHt (p_x, \ppt_x) := \sup_{X \in \scX_\clHt} \lrvert{p_x[X] - \ppt_x[X]}$, where $\scX_\clHt := \{h^{-1}(1): h \in \clHt\}$ and $\clHt := \{\sign(\lrvert{h(x) - h'(x)} - t): h, h' \in \clH, t \in [0, 1]\}$.

\citet{ben2010theory} also argue that in this bound, the total variation $d_1$ is overly strict (thus making the bound unnecessarily loose) and hard to estimate from finite data samples, so they develop another bound which is better known (\citealp[Thm.~2]{ben2010theory}; \citealp[Thm.~1]{johansson2019support})
(only showing the asymptotic version here, \ie, omitting the estimation error from finite samples):
\begin{align}
  \Rt(h) \le{} & R(h) + d_{\clH\Delta\clH}(p_x, \ppt_x) + \lambda_\clH,
  \label{eqn:da-bound-hdh} \\
  \text{where: }
  d_{\clH\Delta\clH}(p_x, \ppt_x) :={} & \sup_{h, h' \in \clH} \lrvert{\bbE_{p(x)}[\ell(h(x), h'(x))] - \bbE_{\ppt(x)}[\ell(h(x), h'(x))]}, \\
  \lambda_\clH :={} & \inf_{h \in \clH} \left[ R(h) + \Rt(h) \right].
\end{align}
Here, $d_{\clH\Delta\clH}(p_x, \ppt_x)$ is called the \emph{$\clH\Delta\clH$-divergence} measuring the difference between $p(x)$ and $\ppt(x)$, under the discriminative efficacy of the labeling function family $\clH$ (thus not as strict as the total variation $d_1$),
and $\lambda_\clH$ is the \emph{ideal joint risk} achieved by $\clH$ measuring the richness or expressiveness of $\clH$ for the two prediction tasks.
The $\clH\Delta\clH$-divergence $d_{\clH\Delta\clH}$ is also estimable from finite data samples~\citep[Lemma~1]{ben2010theory}.
\citet[Thm.~1]{long2015learning} give a similar bound in terms of maximum mean discrepancy (MMD) $d_K$ in place of $d_{\clH\Delta\clH}$.

For successful adaptation, some assumptions on the unknown distribution $\ppt(x,y)$ are required.
A commonly adopted one is:
\begin{align}
  \text{(covariate shift) } \hht^*(x) = h^*(x) \text{ or } p(y|x) = \ppt(y|x),
  \forall x \in \supp(p_x, \ppt_x) := \supp(p_x) \cup \supp(\ppt_x).
  \label{eqn:da-cov-shift}
\end{align}

\paragraph{DA-DIR}
Domain-invariant representation (DIR) based DA methods (DA-DIR)~\citep{pan2010domain, baktashmotlagh2013unsupervised, long2015learning, ganin2016domain} aims to learn a deterministic representation extractor $\eta: \clX \to \clS$ to some representation space $\clS$, in order to achieve a domain-invariant representation:
\begin{align}
  \text{(DIR) } p(s) = \ppt(s),
  \text{where } p(s) := \eta_\#[p_x](s) \text{ and } \ppt(s) := \eta_\#[\ppt_x](s)
  \label{eqn:da-dir}
\end{align}
are the representation distributions on the two domains.
The motivation is that, once DIR is achieved, the distribution difference term (the second term on the r.h.s) of bound \eqref{eqn:da-bound-tv} or \eqref{eqn:da-bound-hdh} diminishes on the representation space $\clS$.
So the bound on $\clS$ is then controlled by the source risk (the first term), 
and driving $h$ to let $R(h)$ approach $R(h^*)$ (\ie, to minimize the source risk $R(h)$) effectively minimizes the target risk.

Let $g: \clS \to \clY$ be a labeling function on the representation space $\clS$.
The end-to-end labeling function is then $h = g \circ \eta$.
Combining the two desiderata of achieving DIR and $R(h^*)$, the typical objective of DA-DIR is in the following form:
\begin{align}
  \min_{\eta \in \clE, g \in \clG} R(g \circ \eta) + \lambda d(\eta_\#[p_x], \eta_\#[\ppt_x]),
  \label{eqn:da-obj-typical}
\end{align}
where $d(\cdot,\cdot)$ is a metric or discrepancy ($d(q,p) \ge 0$; $d(q,p)=0 \Longleftrightarrow q=p$) on distributions, $\lambda$ is a weighting parameter, and $\clE$ and $\clG$ are the hypothesis spaces for $\eta$ and $g$, respectively.

For the existence of the solution of this problem, 
\citet{johansson2019support} consider the following assumption:
\begin{align}
  \text{(strong existence assumption) } \exists \eta^* \in \clE, g^* \in \clG, \st
  \eta^*_\#[p_x] = \eta^*_\#[\ppt_x], g^* \circ \eta^* = h^*.
  \label{eqn:da-strongexist}
\end{align}
They also mention that this is not guaranteed to hold in practice, since it is quite strong:
both DIR and $R(h^*)$ can be simultaneously achieved.

\paragraph{Problem of DA-DIR}
\citet{johansson2019support, zhao2019learning} give examples where even under the strong assumption of both covariate shift and the strong existence assumption~\citep[Assumption~3]{johansson2019support},
simultaneously achieving both DIR and $R(h^*)$ still leads the target risk $\Rt(g \circ \eta)$ to the worst value.

We first analyze the problem through the lens of the above DA bounds.
We will show that when reducing the bounds on $\clS$, they can be uselessly large.

\noindent\bfone For the bound \eqref{eqn:da-bound-tv}.
Applying the bound on the representation space $\clS$ gives:
\begin{align}
  \Rt(g \circ \eta) \le{} & R(g \circ \eta) + 2 d_1(\eta_\#[p_x], \eta_\#[\ppt_x]) \\
  & {}+ \min\{\bbE_{\eta_\#[p_x](s)} [\lrvert*[big]{g^*_\eta(s) - \ggt^*_\eta(s)}], \bbE_{\eta_\#[\ppt_x](s)} [\lrvert*[big]{g^*_\eta(s) - \ggt^*_\eta(s)}]\},
  \label{eqn:da-bound-tv-repr}
\end{align}
where $g^*_\eta$ and $\ggt^*_\eta$ are the optimal labeling functions on top of the representation extractor $\eta$.
It is shown that under the assumption of covariate shift~\citep{ben2010impossibility, gong2016domain} or additionally strong existence~\citep{johansson2019support},
simultaneously achieving both DIR and $R(h^*)$ is not sufficient to guarantee $g^*_\eta = \ggt^*_\eta$,
so the bound may still be large.

In both examples of \citet{johansson2019support} and \citet{zhao2019learning}, the considered $\eta$, although achieving both desiderata, is not $\eta^*$, and this $\eta$ renders different optimal representation-level labeling functions on the two domains: $g^*_\eta \ne \ggt^*_\eta$, so the bound is still large.
\citet{johansson2019support} claim that it is necessary to require $\eta$ to be invertible to make $g^*_\eta = \ggt^*_\eta$, 
and develop a bound (Thm.~2) that explicitly shows the effect of the invertibility of $\eta$.
The $\eta$ functions in the examples are not invertible.

\noindent\bftwo For the bound \eqref{eqn:da-bound-hdh}.
Applying the bound on the representation space $\clS$ gives:
\begin{align}
  \bbE_{\ppt(s,y)}[\ell(g(s), y)] \le{} & \bbE_{p(s,y)}[\ell(g(s), y)] + d_{\clG\Delta\clG}(\eta_\#[p_x], \eta_\#[\ppt_x]) \\
  & {}+ \inf_{g \in \clG} \left[ \bbE_{\ppt(s,y)}[\ell(g(s), y)] + \bbE_{p(s,y)}[\ell(g(s), y)] \right],
\end{align}
where $p_{s,y} := (\eta, \id_y)_\#[p_{x,y}]$ with $\id_y: (x,y) \mapsto y$ and similarly $\ppt_{s,y} := (\eta, \id_y)_\#[\ppt_{x,y}]$.
Note that $\bbE_{p(s,y)}[\ell(g(s), y)] = \bbE_{p(x,y)}[\ell(g(\eta(x)), y)] = R(g \circ \eta)$ and similarly $\bbE_{\ppt(s,y)}[\ell(g(s), y)] = \Rt(g \circ \eta)$.
So the last term on the r.h.s becomes
$\inf_{g \in \clG} \left[ \Rt(g \circ \eta) + R(g \circ \eta) \right] = \lambda_{\clG \circ \eta}$, where $\clG \circ \eta := \{g \circ \eta: g \in \clG\}$,
and the bound then reformulates to:
\begin{align}
  \Rt(g \circ \eta) \le R(g \circ \eta) + d_{\clG\Delta\clG}(\eta_\#[p_x], \eta_\#[\ppt_x]) + \lambda_{\clG \circ \eta}.
  \label{eqn:da-bound-hdh-repr}
\end{align}
This result is shown by \citet{johansson2019support}.
They argue that finding $\eta$ that achieves both DIR and $R(h^*)$ simultaneously (with some $g^*_\eta$) cannot guarantee a tighter bound since the last term $\lambda_{\clG \circ \eta}$ may be very large.

In both examples of \citet{johansson2019support} and \citet{zhao2019learning}, it holds that $\supp(p_x) \cap \supp(\ppt_x) = \emptyset$.
It may cause the problem that $g \circ \eta$ is very different from $h^*$ on $\supp(\ppt_x)$ even when $R(h^*)$ is achieved, since $R(g \circ \eta) = R(h^*)$ only constraints the behavior of $g \circ \eta$ on $\supp(p_x)$.
The developed bound by \citet[Thm.~2]{johansson2019support} also explicitly shows the role of a support overlap, thus is called a support-invertibility bound.
They also give an example showing that DIR (particularly implemented by minimizing MMD) is not necessary (``sometimes too strict'') for learning the shared/invariant $p(y|x)$.

The problem of DA-DIR is also studied under more modern bounds \bfthr \bffor and arguments \bffiv.

\noindent\bfthr A third bound.
\citet{zhao2019learning} develop another bound for binary classification $\clY := \{0,1\}$, under the risk function $R(h) := \bbE_{p(x)}[\lrvert*[big]{h^*(x) - h(x)}]$.
The bound is expressed in terms of the JS distance~\citep{endres2003new} $d_\JS(p,q) := \sqrt{\JS(p,q)}$, where $\JS(p,q)$ is the JS divergence, which is bounded: $0 \le \JS(p,q) \le 1$
\footnote{This bound is under the unit of bits, \ie, base 2 logarithm is used in the KL divergence defining the JS divergence. Under the unit of nats, \ie, the natural logarithm $\ln$ is used, the bound becomes $0 \le \JS(p,q) \le \ln 2$.}.
It is shown that~\citep[Lemma~4.8]{zhao2019learning}:
\begin{align}
  d_\JS(p_y, \ppt_y) \le d_\JS(\eta_\#[p_x], \eta_\#[\ppt_x]) + \sqrt{R(g \circ \eta)} + \sqrt{\Rt(g \circ \eta)}.
  \label{eqn:da-bound-js-lem}
\end{align}
If $d_\JS(p_y, \ppt_y) \ge d_\JS(\eta_\#[p_x], \eta_\#[\ppt_x])$\footnote{
  Unfortunately, it seems that the opposite direction of the inequality holds when there exist $\eta^*$ and $g^*$ (unnecessarily the ones in the strong existence assumption or Assumption~3 of \citet{johansson2019support}) such that
  $p_y = (g^* \circ \eta^*)_\#[p_x]$ and $\ppt_y = (g^* \circ \eta^*)_\#[\ppt_x]$ and that $\eta$ is a reparameterization of $\eta^*$, due to the celebrated data processing inequality.
}, the bound is given as~\citep[Thm.~4.3]{zhao2019learning}:
\begin{align}
  R(g \circ \eta) + \Rt(g \circ \eta) \ge \frac{1}{2} \left( d_\JS(p_y, \ppt_y) - d_\JS(\eta_\#[p_x], \eta_\#[\ppt_x]) \right)^2,
  \label{eqn:da-bound-js-thm}
\end{align}
or when the two domains are allowed to have their own representation-level labeling functions $g$ and $\ggt$, we have~\citep[Corollary~4.1]{zhao2019learning}:
\begin{align}
  R(g \circ \eta) + \Rt(\ggt \circ \eta) \ge \frac{1}{2} \left( d_\JS(p_y, \ppt_y) - d_\JS(\eta_\#[p_x], \eta_\#[\ppt_x]) \right)^2.
  \label{eqn:da-bound-js-cor}
\end{align}
When $p(y) \ne \ppt(y)$, we have $d_\JS(p_y, \ppt_y) > 0$, so DIR, which minimizes $d_\JS(\eta_\#[p_x], \eta_\#[\ppt_x])$, becomes harmful to minimizing the target risk $\Rt(\ggt \circ \eta)$.

\noindent\bffor
\citet[Thm.~6]{chuang2020estimating} probe into the mysterious term $\lambda_{\clG\circ\eta}$ in the bound \eqref{eqn:da-bound-hdh-repr} and show how it is affected by the complexity of $\clE$ (the hypothesis space of $\eta$):
\begin{align}
  \Rt(g \circ \eta) \le{} & R(g \circ \eta) + d_{\clG\Delta\clG}(\eta_\#[p_x], \eta_\#[\ppt_x]) + d_{\clG_{\clE\Delta\clE}}(p_x, \ppt_x) + \lambda_{\clG \circ \clE}(\eta),
  \label{eqn:da-bound-ecomplex} \\
  \text{where: }
  d_{\clG_{\clE\Delta\clE}}(p_x, \ppt_x) :={} & \sup_{g \in \clG; \eta, \eta' \in \clE} \lrvert{\bbE_{p_x}[\ell(g \circ \eta, g \circ \eta')] - \bbE_{\ppt_x}[\ell(g \circ \eta, g \circ \eta')]}, \\
  \lambda_{\clG \circ \clE}(\eta) :={} & \inf_{g' \in \clG, \eta' \in \clE} 2 R(g' \circ \eta) + R(g' \circ \eta') + \Rt(g' \circ \eta').
\end{align}
Here, $d_{\clG\Delta\clG}(\eta_\#[p_x], \eta_\#[\ppt_x])$ measures the representation distribution difference,
$d_{\clG_{\clE\Delta\clE}}(p_x, \ppt_x)$ measures the complexity of the representation-extractor family $\clE$ w.r.t $\clG$~\citep[Def.~5]{chuang2020estimating}, 
and $\lambda_{\clG \circ \clE}(\eta)$ is ``a variant of the best in-class joint risk''.
For a given $\clG$, although a more expressive $\clE$ lowers $\lambda_{\clG \circ \clE}(\eta)$ and contains a more capable $\eta$ to reduce $d_{\clG\Delta\clG}(\eta_\#[p_x], \eta_\#[\ppt_x])$, such an $\clE$ also incurs a larger $d_{\clG_{\clE\Delta\clE}}(p_x, \ppt_x)$,
so there is a trade-off when choosing a proper $\clE$.
\citet{chuang2020estimating} illustrate this trade-off by a toy example, and observe this trade-off in experiments.
Similarly, there is also a trade-off in the complexity of $\clG$ (a more expressive $\clG$ lowers $\lambda_{\clG \circ \clE}(\eta)$ but increases $d_{\clG\Delta\clG}(\eta_\#[p_x], \eta_\#[\ppt_x])$ and $d_{\clG_{\clE\Delta\clE}}(p_x, \ppt_x)$),
but \citet{chuang2020estimating} find the performance of DA-DIR much less sensitive to it empirically.
They also point out the implication of this trade-off in choosing which layer in a neural network as the representation (Prop.~7) with an empirical study.

\citet{chuang2020estimating} also propose a method to estimate the target-domain performance (\ie, the OOD generalization performance) in terms of $\Rt(h)$ of a supervised model $h$ using a set of DA-DIR models $\hat{\clH}^*$.
The method is supported by its Lemma~4: $\lrvert{\Rt(h) - \sup_{h' \in \hat{\clH}^*} \bbE_{\ppt(x)}[\ell(h(x), h'(x))]} \le \sup_{h' \in \hat{\clH}^*} \Rt(h')$.
The supremum on the l.h.s can be estimated using unsupervised data on the target domain, and it is treated as an estimate to $\Rt(h)$ given that the r.h.s is believed to be small for DA-DIR models $\hat{\clH}^*$.

\noindent\bffiv
\citet{arjovsky2019invariant} point out that in the covariate shift case $p(y|s) = \ppt(y|s)$, achieving DIR $p(s) = \ppt(s)$ implies $p(y) = \ppt(y)$ (since $p(s) p(y|s) = \ppt(s) \ppt(y|s)$).
This may not hold in practice.
When it does not hold, the bound \eqref{eqn:da-bound-js-thm} shows that DIR may limit the target-domain performance.

\paragraph{Comparison with \ourmodel}
The key feature of our \ourmodel is that it is based on causal invariance.
In most of the above bounds, including Eqs.~(\ref{eqn:da-bound-tv},~\ref{eqn:da-bound-hdh}) for general DA and Eqs.~(\ref{eqn:da-bound-tv-repr}, \ref{eqn:da-bound-hdh-repr}, \ref{eqn:da-bound-js-thm}, \ref{eqn:da-bound-ecomplex}) for DA-DIR,
the same labeling function $h$ or $g \circ \eta$ is used in both domains (the risks $R$ and $\Rt$ on both domains measure the same $h$ or $g \circ \eta$).
So for successful adaptation, covariate shift (invariant $h^*$ or $p(y|x)$) is a basic assumption, which implies inference invariance (invariant $\eta^*$ or $p(s|x)$) for DA-DIR.
Yet, as explained in Sec.~\ref{sec:model-infinv}, since the data at hand is produced from a certain mechanism of nature anyway, the invariance in the causal generative direction $p(x|s,v)$ is more fundamental and reliable than covariate shift or inference invariance.
The causal invariance allows $p(s) \ne \ppt(s)$ and subsequently a difference in the inference direction: $p(s|x) \ne \ppt(s|x)$ or $\eta^* \ne \etat^*$, and $p(y|x) \ne \ppt(y|x)$ or $h^* \ne \hht^*$.
Following this new philosophy, \ourmodel{}-ind and \ourmodel{}-DA use a different inference and prediction rule in the target domain, and Theorems~\ref{thm:ood} and~\ref{thm:da} give OOD prediction guarantees for this different prediction rule.
This is in contrast to most existing DA methods and theory.

Another advantage of \ourmodel is that it has an identifiability guarantee (Thm.~\ref{thm:id}).
In the above analyses \bfone and \bftwo, we see that the problem of DA-DIR arises since achieving both DIR and $R(h^*)$ simultaneously cannot guarantee $\eta = \eta^*$ or $g = g^*$ or $g \circ \eta = h^*$ on $\supp(p_x, \ppt_x)$, even in some sense of semantic or performance equivalence.
This is essentially an identifiability problem.
\ourmodel achieves identifiability by fitting the entire data distribution $p(x,y)$.
In contrast, DA-DIR is not a generative method, and only fits $p(y|x)$.
Although DA-DIR also seeks to achieve DIR, it is a weaker goal than fitting $p(x)$ (DIR cannot give $p(x)$).
So DA-DIR does not fully exploit the data distribution $p(x,y)$, and identifiability is a problem even with the strong assumption of both covariate shift and the strong existence assumption.

In terms of the considered quantity in the bounds, all the existing ones above 
bound the objective of the target risk $\Rt(h)$ in terms of the accessible source risk $R(h)$ for an arbitrary labeling function $h$,
while our bound \eqref{eqn:ood-expc} relates the target risks of the optimally-learned source-domain labeling function ${h'}^*$ and of the target-domain oracle labeling function $\hht^*$, \ie, it bounds $\lrvert*{\Rt({h'}^*) - \Rt(\hht^*)}$.
It measures the risk gap of the best source labeling function on the target domain.
After adaptation, Thm.~\ref{thm:da} (\eqref{eqn:da-gen}) shows that \ourmodel{}-DA achieves the optimal labeling function on the target domain.

Under bounds Eqs.~(\ref{eqn:da-bound-js-thm},~\ref{eqn:da-bound-js-cor}), we are not minimizing $d_\JS(\eta_\#[p(x)], \eta_\#[\ppt(x)])$, so our method is good under that view.
In fact, in \ourmodel the representation distributions on the two domains are $p(s) = \int p(s,v) \dd v$ and $\ppt(s) = \int \ppt(s,v) \dd v$ (replacing $\eta_\#[p(x)]$ and $\eta_\#[\ppt(x)]$).
They are generally different and we do not seek to match them.

\section{Methodology Details} \label{supp:meth}

\subsection{Derivation of Learning Objectives} \label{supp:meth-obj}

\subsubsection{The Evidence Lower BOund (ELBO).} \label{supp:meth-obj-elbo}
A common and effective approach to let the model $p$ match the data distribution $p^*(x,y)$ is maximizing likelihood, that is to maximize $\bbE_{p^*(x,y)}[\log p(x,y)]$.
It is equivalent to minimizing $\KL(p^*(x,y) \Vert p(x,y))$ (since $\bbE_{p^*(x,y)} [\log p^*(x,y)]$ is constant of $p$), so it drives $p(x,y)$ towards $p^*(x,y)$.
But the likelihood function $p(x,y) = \int p(s,v,x,y) \dd s \ud v$ involves an intractable integration, which is hard to estimate and optimize.
To address this, the popular method of \emph{variational expectation-maximization} (variational EM) introduces a tractable (has closed-form density function and easy to draw samples from it) distribution $q(s,v|x,y)$ of the latent variables given observed variables, and a lower bound of the likelihood function can be derived:
\begin{align}
  \log p(x,y) ={} & \log \bbE_{p(s,v)} [p(s,v,x,y)] = \log \bbE_{q(s,v|x,y)} \bigg[ \frac{p(s,v,x,y)}{q(s,v|x,y)} \bigg] \\
  \ge{} & \bbE_{q(s,v|x,y)} \bigg[ \log \frac{p(s,v,x,y)}{q(s,v|x,y)} \bigg]
  =: \clL_{p, \, q_{s,v|x,y}} (x,y),
  \label{eqn:elbo-def}
\end{align}
where the inequality follows Jensen's inequality and the concavity of the $\log$ function.
The function $\clL_{p, \, q_{s,v|x,y}} (x,y)$ is thus called \emph{Evidence Lower BOund} (ELBO).
The tractable distribution $q(s,v|x,y)$ is called \emph{variational distribution}, and is commonly instantiated by a standalone model (from the generative model) called an \emph{inference model}.
Moreover, we have:
\begin{align}
  & \clL_{p, \, q_{s,v|s,y}} (x,y) + \KL(q(s,v|x,y) \Vert p(s,v|x,y)) \\
  ={} & \bbE_{q(s,v|x,y)} \bigg[ \log \frac{p(s,v,x,y)}{q(s,v|x,y)} \bigg] + \bbE_{q(s,v|x,y)} \bigg[ \log \frac{q(s,v|x,y)}{p(s,v|x,y)} \bigg] \\
  ={} & \bbE_{q(s,v|x,y)} \bigg[ \log \frac{p(s,v,x,y)}{p(s,v|x,y)} \bigg]
  = \bbE_{q(s,v|x,y)} [ \log p(x,y) ] \\
  ={} & \log p(x,y),
  \label{eqn:elbo-equality}
\end{align}
so maximizing $\clL_{p, \, q_{s,v|x,y}} (x,y)$ w.r.t $q(s,v|x,y)$ is equivalent to minimizing $\KL( q(s,v|x,y) \Vert p(s,v|x,y) )$ (since the r.h.s $\log p(x,y)$ is constant of $q(s,v|x,y)$), which drives $q(s,v|x,y)$ towards the true posterior (\ie, the goal of \emph{variational inference}),
and once this is (perfectly) done, $\clL_{p, \, q_{s,v|x,y}} (x,y)$ becomes a lower bound of $\log p(x,y)$ that is tight at the current model $p$, so maximizing $\clL_{p, \, q_{s,v|x,y}} (x,y)$ w.r.t $p$ effectively maximizes $\log p(x,y)$ (\ie, the goal of maximizing likelihood).
So the training objective becomes the expected ELBO $\bbE_{p^*(x,y)} [\clL_{p, \, q_{s,v|x,y}} (x,y)]$.
Optimizing it w.r.t $q(s,v|x,y)$ and $p$ alternately drives $q(s,v|x,y)$ towards $p(s,v|x,y)$ and $p(x,y)$ towards $p^*(x,y)$ eventually.
The derivations and conclusions above hold for general latent variable models, with $(s,v)$ representing the latent variables, and $(x,y)$ observed variables (data variables).

This standard form of ELBO gives the objective for fitting unsupervised test-domain data from the underlying data distribution $\ppt^*(x)$.
In this case, the observed variable is only $x$ while the latent variable is still $(s,v)$,
so the required joint distribution for latent and observed variables is $\ppt(s,v,x) = \ppt(s,v) p(x|s,v)$, and the inference model is in the form $\qqt(s,v|x)$.
Following the form of \eqref{eqn:elbo-def}, the ELBO objective for fitting $\ppt^*(x)$ (\ie, the lower bound for $\log \ppt(x)$) is:
\begin{align}
  \clL_{\ppt, \, \qqt_{s,v|x}} (x)
  = \bbE_{\qqt(s,v|x)} \Big[ \log \frac{\ppt(s,v,x)}{\qqt(s,v|x)} \Big].
\end{align}
This leads to \eqref{eqn:elbo-tgt}.

\subsubsection{Variational EM for learning \ourmodel.} \label{supp:meth-obj-train}
In the supervised case, the expected ELBO objective $\bbE_{p^*(x,y)} [\clL_{p, \, q_{s,v|x,y}} (x,y)]$ can also be understood as the conventional supervised learning loss, \ie the cross entropy, regularized by a generative reconstruction term.
As explained in the main text (Sec.~\ref{sec:meth}), after training, we only have the model $p(s,v,x,y)$ and an approximation $q(s,v|x,y)$ to the posterior $p(s,v|x,y)$, and prediction using $p(y|x)$ is still intractable.
So we employ a tractable distribution $q(s,v,y|x)$ to model the required variational distribution as $q(s,v|x,y) = q(s,v,y|x) / q(y|x)$, where $q(y|x) = \int q(s,v,y|x) \dd s \ud v$ is the derived marginal distribution of $y$ from $q(s,v,y|x)$ (we will show that it can be effectively estimated and sampled from).
With this instantiation, the expected ELBO becomes:
\begin{align}
  & \bbE_{p^*(x,y)} [\clL_{p, \, q_{s,v|x,y} = \cdots(q_{s,v,y|x})} (x,y)] \\
  ={} & \int p^*(x,y) \frac{q(s,v,y|x)}{q(y|x)} \log \frac{p(s,v,x,y) q(y|x)}{q(s,v,y|x)} \dd s \ud v \ud x \ud y \\
  ={} & \int p^*(x,y) \frac{q(s,v,y|x)}{q(y|x)} \log q(y|x) \dd s \ud v \ud x \ud y + \int p^*(x,y) \frac{q(s,v,y|x)}{q(y|x)} \log \frac{p(s,v,x,y)}{q(s,v,y|x)} \dd s \ud v \ud x \ud y \\
  ={} & \int p^*(x) \bigg( \int p^*(y|x) \frac{\int q(s,v,y|x) \dd s \ud v}{q(y|x)} \log q(y|x) \dd y \bigg) \dd x \\
  & {}+ \int p^*(x) \bigg( \int \frac{p^*(y|x)}{q(y|x)} q(s,v,y|x) \log \frac{p(s,v,x,y)}{q(s,v,y|x)} \dd s \ud v \ud y \bigg) \dd x \\
  ={} & \bbE_{p^*(x)} \bbE_{p^*(y|x)} [\log q(y|x)] + \bbE_{p^*(x)} \bbE_{q(s,v,y|x)} \bigg[ \frac{p^*(y|x)}{q(y|x)} \log \frac{p(s,v,x,y)}{q(s,v,y|x)} \bigg],
  \label{eqn:elbo-interp-deduc}
\end{align}
which is \eqref{eqn:elbo-interp}.
Here, we use the shorthand ``$q_{s,v|x,y} = \cdots(q_{s,v,y|x})$'' for the above substitution $q(s,v|x,y) = q(s,v,y|x) / \int q(s,v,y|x) \dd s \ud v$ and highlight the argument therein.
The first term is the (negative) expected cross entropy loss, which drives the inference model (predictor) $q(y|x)$ towards $p^*(y|x)$ for $p^*(x)$-a.e. $x$.
Once this is (perfectly) done, the second term becomes $\bbE_{p^*(x)} \bbE_{q(s,v,y|x)} [ \log \big( p(s,v,x,y) / q(s,v,y|x) \big) ]$, which is the expected ELBO $\bbE_{p^*(x)} [\clL_{p, \, q_{s,v,y|x}} (x,y)]$ for $q(s,v,y|x)$.
It thus drives $q(s,v,y|x)$ towards $p(s,v,y|x)$ and $p(x)$ towards $p^*(x)$.
It accounts for a regularization by fitting the input distribution $p^*(x)$ and align the inference model (predictor) with the generative model.

The target of $q(s,v,y|x)$, \ie $p(s,v,y|x)$, adopts a factorization $p(s,v,y|x) = p(s,v|x) p(y|s)$ due to the graphical structure (Fig.~\ref{fig:gen-sv}) of \ourmodel (\ie, $y \pperp (x,v) \mid s$).
The factor $p(y|s)$ is known (the invariant causal mechanism to generate $y$ in \ourmodel), so we only need to employ an inference model $q(s,v|x)$ for the intractable factor $p(s,v|x)$, so $q(s,v,y|x) = q(s,v|x) p(y|s)$.
Using this relation, we can reformulate \eqref{eqn:elbo-interp} as:
\begin{align}
  & \bbE_{p^*(x,y)} [\clL_{p, \, q_{s,v|x,y} = \cdots(q_{s,v|x}, p_{y|s})} (x,y)] \\
  ={} & \bbE_{p^*(x,y)} [\log q(y|x)] + \bbE_{p^*(x)} \bigg[ \int q(s,v|x) p(y|s) \frac{p^*(y|x)}{q(y|x)} \log \frac{p(s,v,x)}{q(s,v|x)} \dd s \ud v \ud y \bigg] \\
  ={} & \bbE_{p^*(x,y)} [\log q(y|x)] + \bbE_{p^*(x)} \bigg[ \int \frac{p^*(y|x)}{q(y|x)} \bigg( \int q(s,v|x) p(y|s) \log \frac{p(s,v,x)}{q(s,v|x)} \dd s \ud v \bigg) \ud y \bigg] \\
  ={} & \bbE_{p^*(x,y)} [\log q(y|x)] + \bbE_{p^*(x,y)} \bigg[ \frac{1}{q(y|x)} \bbE_{q(s,v|x)} \Big[ p(y|s) \log \frac{p(s,v,x)}{q(s,v|x)} \Big] \bigg],
  \label{eqn:elbo-src-deduc}
\end{align}
which is \eqref{eqn:elbo-src}.
We used the shorthand ``$q_{s,v|x,y} = \cdots(q_{s,v|x}, p_{y|s})$'' for the substitution for $q(s,v|x,y)$ using $q(s,v|x)$ and $p(y|s)$.
With this form of $q(s,v,y|x) = q(s,v|x) p(y|s)$, we have $q(y|x) = \bbE_{q(s,v|x)}[p(y|s)]$ which can also be estimated and optimized using reparameterization.
For prediction, we can sample from the approximation $q(y|x)$ instead of the intractable $p(y|x)$.
This can be done by ancestral sampling: first sample $(s,v)$ from $q(s,v|x)$, and then use the sampled $s$ to sample $y$ from $p(y|s)$.

\subsubsection{Variational EM for learning \ourmodel with test-domain inference model
(Learning \ourmodel{}-ind and \ourmodel{}-DA on the training domain).} \label{supp:meth-obj-test}

See the main text in Sec.~\ref{sec:meth-ood} and Sec.~\ref{sec:meth-da} for motivations and the basic idea of the methods.
Methods for \ourmodel{}-ind and \ourmodel{}-DA are similar, so we mainly show the detailed derivation for \ourmodel{}-ind.

Since the prior is the only difference between $p(s,v,x,y)$ and $\pind(s,v,x,y)$, we have $\frac{p(s,v,x,y)}{\pind(s,v,x,y)} = \frac{p(s,v)}{\pind(s,v)}$.
So $p(s,v,y|x) = \frac{p(s,v)}{\pind(s,v)} \frac{\pind(x)}{p(x)} \pind(s,v,y|x)$.
As explained, inference models now only need to approximate the posterior $(s,v) \mid x$.
Since $p(s,v,y|x) = p(s,v|x) p(y|s)$ and $\pind(s,v,y|x) = \pind(s,v|x) p(y|s)$ share the same $p(y|s)$ factor, we have $p(s,v|x) = \frac{p(s,v)}{\pind(s,v)} \frac{\pind(x)}{p(x)} \pind(s,v|x)$.
The variational distributions $q(s,v|x)$ and $\qind(s,v|x)$ target $p(s,v|x)$ and $\pind(s,v|x)$ respectively, so we can express the former with the latter:
\begin{align} \label{eqn:qind-subs}
  q(s,v|x) = \frac{p(s,v)}{\pind(s,v)} \frac{\pind(x)}{p(x)} \qind(s,v|x).
\end{align}
Once $\qind(s,v|x)$ achieves its goal, such represented $q(s,v|x)$ also does so.
So we only need to construct an inference model for $\qind(s,v|x)$ and optimize it.
With this representation, we have:
\begin{align}
  q(y|x) ={} & \bbE_{q(s,v|x)} [p(y|s)] = \bbE_{\qind(s,v|x)} \bigg[ \frac{p(s,v)}{\pind(s,v)} \frac{\pind(x)}{p(x)} p(y|s) \bigg]
  = \frac{\pind(x)}{p(x)} \bbE_{\qind(s,v|x)} \bigg[ \frac{p(s,v)}{\pind(s,v)} p(y|s) \bigg] \\
  ={} & \frac{\pind(x)}{p(x)} \pi(y|x),
  \label{eqn:q-and-pi}
\end{align}
where $\pi(y|x) := \bbE_{\qind(s,v|x)} \big[ \frac{p(s,v)}{\pind(s,v)} p(y|s) \big]$ as in the main text, which can be estimated and optimized using the reparameterization of $\qind(s,v|x)$.
From \eqref{eqn:elbo-src}, the expected ELBO training objective can be reformulated as:
\begin{align}
  & \bbE_{p^*(x,y)} [\clL_{p, \, q_{s,v|x,y} = \cdots(\qind_{s,v|x}, p)} (x,y)] \\
  ={} & \bbE_{p^*(x,y)} \bigg[ \log q(y|x) + \frac{1}{q(y|x)} \bbE_{q(s,v|x)} \Big[ p(y|s) \log \frac{p(s,v,x)}{q(s,v|x)} \Big] \bigg] \\
  ={} & \bbE_{p^*(x,y)} \bigg[ \log \frac{\pind(x)}{p(x)} + \log \pi(y|x) \\
  & {}+ \frac{p(x)}{\pind(x)} \frac{1}{\pi(y|x)} \bbE_{\qind(s,v|x)} \Big[ \frac{p(s,v)}{\pind(s,v)} \frac{\pind(x)}{p(x)} p(y|s) \log \frac{p(s,v) p(x|s,v)}{\frac{p(s,v)}{\pind(s,v)} \frac{\pind(x)}{p(x)} \qind(s,v|x)} \Big] \bigg] \\
  ={} & \bbE_{p^*(x,y)} \bigg[ \log \frac{\pind(x)}{p(x)} + \log \pi(y|x) \\
  & {}+ \frac{1}{\pi(y|x)} \bbE_{\qind(s,v|x)} \Big[ \frac{p(s,v)}{\pind(s,v)} p(y|s) \Big( \log \frac{p(x)}{\pind(x)} + \log \frac{\pind(s,v) p(x|s,v)}{\qind(s,v|x)} \Big) \Big] \bigg] \\
  ={} & \bbE_{p^*(x,y)} \bigg[ \log \frac{\pind(x)}{p(x)} + \log \pi(y|x)
    + \frac{1}{\pi(y|x)} \bbE_{\qind(s,v|x)} \Big[ \frac{p(s,v)}{\pind(s,v)} p(y|s) \Big] \log \frac{p(x)}{\pind(x)} \\
  & {}+ \frac{1}{\pi(y|x)} \bbE_{\qind(s,v|x)} \Big[ \frac{p(s,v)}{\pind(s,v)} p(y|s) \log \frac{\pind(s,v) p(x|s,v)}{\qind(s,v|x)} \Big] \bigg] \\
  ={} & \bbE_{p^*(x,y)} \bigg[ \log \frac{\pind(x)}{p(x)} + \log \pi(y|x)
    + \frac{1}{\pi(y|x)} \pi(y|x) \log \frac{p(x)}{\pind(x)} \\
  & {}+ \frac{1}{\pi(y|x)} \bbE_{\qind(s,v|x)} \Big[ \frac{p(s,v)}{\pind(s,v)} p(y|s) \log \frac{\pind(s,v,x)}{\qind(s,v|x)} \Big] \bigg] \\
  ={} & \bbE_{p^*(x,y)} \bigg[ \log \pi(y|x) + \frac{1}{\pi(y|x)} \bbE_{\qind(s,v|x)} \Big[ \frac{p(s,v)}{\pind(s,v)} p(y|s) \log \frac{\pind(s,v,x)}{\qind(s,v|x)} \Big] \bigg],
  \label{eqn:elbo-src-ind-deduc}
\end{align}
where in the second-last equality we have used the definition of $\pi(y|x)$.
The shorthand ``$q_{s,v|x,y} = \cdots(\qind_{s,v|x}, p)$'' represents the substitution using $\qind(s,v|x)$ and $p = \lrangle{p(s,v), p(x|s,v), p(y|s)}$ for $q(s,v|x,y) = q(s,v|x) p(y|s) / \int q(s,v|x) p(y|s) \dd s \ud v$ where $q(s,v|x)$ is determined by $\qind(s,v|x)$ and $p$ via \eqref{eqn:qind-subs} (recall that $\pind(s,v)$ is determined by $p(s,v)$, so $\pind(x)$ is also determined by $p(s,v)$ and $p(x|s,v)$).
This \eqref{eqn:elbo-src-ind-deduc} gives \eqref{eqn:elbo-src-ind} for \ourmodel{}-ind.
Note that $\pi(y|x)$ is not used in prediction, so there is no need to sample from it.
Prediction is done by ancestral sampling from $\qind(y|x)$, that is to first sample from $\qind(s,v|x)$ and then from $p(y|s)$.
Using this reformulation, we can train a \ourmodel with independent prior even on data that manifests a correlated prior.

For \ourmodel{}-DA, we only need to replace the independent prior $\pind(s,v)$ hypothesized for the test domain with the standalone prior model $\ppt(s,v)$ dedicated to learning the test-domain prior,
and re-denote the test-domain inference model $\qind(s,v|x)$ with $\qqt(s,v|x)$.
By doing so, \eqref{eqn:elbo-src-ind-deduc} gives \eqref{eqn:elbo-src-qt}, \ie the objective for \ourmodel{}-DA on the training domain.
For numerical stability, we employ the \texttt{log-sum-exp} trick to estimate the expectations and compute the gradients.

\subsubsection{Methods for \ourmodelablated for ablation study.} \label{supp:meth-obj-svae}
The conclusions and methods can also be applied to general latent-variable generative models for supervised learning, by replacing $(s,v)$ with their latent variables.
Particularly, the method also applies to the counterpart of \ourmodel in the ablation study experiment,
which does not distinguish the two latent factors $s$ and $v$ and treats them as a united latent variable $z = (s,v)$.
We thus call it \textbf{\ourmodelablated}.
The essential difference from \ourmodel is that \ourmodelablated keeps the $v \to y$ arrow, which is unlikely a causal relation as we argued in Sec.~\ref{sec:model}, item~\bffor.
Formally, a \ourmodelablated model is defined as the tuple $p := \lrangle{p(z), p(x|z), p(y|z)}$, and the corresponding inference model is in the form $q(z|x)$.

Following a similar derivation of \eqref{eqn:elbo-src-deduc}, we have the objective for fitting training-domain data:
\begin{align}
  & \bbE_{p^*(x,y)} [\clL_{p, \, q_{z|x,y} = \cdots(q_{z|x}, p_{y|z})} (x,y)] \\
  ={} & \bbE_{p^*(x,y)} [\log q(y|x)] + \bbE_{p^*(x,y)} \bigg[ \frac{1}{q(y|x)} \bbE_{q(z|x)} \Big[ p(y|z) \log \frac{p(z,x)}{q(z|x)} \Big] \bigg],
  \label{eqn:elbo-src-svae}
\end{align}
where $q(y|x) = \bbE_{q(z|x)} [p(y|z)]$.
The shorthand ``$q_{z|x,y} = \cdots(q_{z|x}, p_{y|z})$'' is similarly for the substitution $q(z|x,y) = q(z|x) p(y|z) / \int q(z|x) p(y|z) \dd z$ using $q(z|x)$ and $p(y|z)$.

As \ourmodelablated does not consider the distinction between $s$ and $v$, there is no \ourmodelablated{}-ind version.
The \ourmodelablated{}-DA version for domain adaptation is possible by using a standalone prior model $\ppt(z)$ for the test domain, which is learned by optimizing the corresponding ELBO objective similar to \eqref{eqn:elbo-tgt}:
\begin{align}
  \max_{\ppt, \, \qqt_{z|x}} \bbE_{\ppt^*(x)}  [\clL_{\ppt, \, \qqt_{z|x}} (x)],
  \text{where } \clL_{\ppt, \, \qqt_{z|x}} \!(x)
  = \bbE_{\qqt(z|x)} \! \Big[ \! \log \frac{\ppt(z) p(x|z)}{\qqt(z|x)} \Big].
  \label{eqn:elbo-tgt-svae}
\end{align}
To fit training-domain data using the test-domain inference model $\qqt(z|x)$, following a similar derivation of \eqref{eqn:elbo-src-ind-deduc}, we have the objective on the training domain for \ourmodel{}-DA:
\begin{align}
  \max_{p, \, \qqt_{z|x}} \bbE_{p^*\!(x,y)} \Big[
    \log \pi(y|x) + \frac{1}{\pi(y|x)} \bbE_{\qqt(z|x)}
    \Big[ \frac{p(z)}{\ppt(z)} p(y|z) \log \frac{\ppt(z) p(x|z)}{\qqt(z|x)} \Big] \Big],
  \label{eqn:elbo-src-qt-svae}
\end{align}
where $\pi(y|x) := \bbE_{\qqt(z|x)} \big[ \frac{p(z)}{\ppt(z)} p(y|z) \big]$.

\subsection{Instantiating the Inference Model} \label{supp:meth-instant}

Although motivated from learning a generative model, the method 
can be implemented using a general discriminative model (with hidden nodes) with causal behavior.
By parsing some of the hidden nodes as $s$ and some others as $v$, a discriminative model could formalize a distribution $q(s,v,y|x)$, which implements the inference model and the generative mechanism $p(y|s)$. 
The parsing mode is shown in Fig.~\ref{fig:discr-sv}, which is based on the following consideration.

\begin{wrapfigure}{r}{.324\textwidth}
  \centering
  \vspace{-10pt}
  \hspace{-16pt}\includegraphics[width=.160\textwidth]{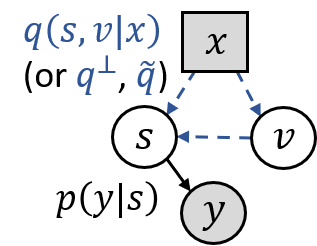}
  \vspace{-4pt}
  \caption{Parsing a general discriminative model as an inference model for \ourmodel.
    The black solid arrow constructs $p(y|s)$ in the generative model, and the blue dashed arrows (representing computational but not causal directions) construct $q(s,v|x)$ (or $\qind(s,v|x)$ or $\qqt(s,v|x)$) as the inference model.
  }
  \vspace{-16pt}
  \label{fig:discr-sv}
\end{wrapfigure}

\bfone 
The graphical structure of \ourmodel in Fig.~\ref{fig:gen-sv} indicates that $(v,x) \pperp y \mid s$, so the hidden nodes for $s$ should isolate $y$ from $v$ and $x$.
The model then factorizes the distribution as $q(s,v,y|x) = q(s,v|x) q(y|s)$, and since the inference and generative models share the distribution on $y|s$ (see the main text for explanation), we can thus use the component $q(y|s)$ given by the discriminative model to implement the generative mechanism $p(y|s)$.

\bftwo The graphical structure in Fig.~\ref{fig:gen-sv} also indicates that $s \notpperp v \mid x$ due to the v-structure (collider) at $x$ (``explain away'').
The component $q(s,v|x)$ should embody this dependence, so the hidden nodes chosen as $v$ should have an effect on those as $s$.
Note that the arrows in Fig.~\ref{fig:discr-sv} represent computation directions but not causal directions.
We orient the computation direction $v \to s$ since all hidden nodes in a discriminative model eventually contribute to computing $y$.

After parsing, the discriminative model gives a mapping $(s,v) = \eta(x)$.
We implement the distribution by\footnote{
  Other approaches to introducing randomness are also possible, such as employing stochasticity on the parameters/weights as in Bayesian neural networks~\citep{neal1995bayesian}, or using dropout~\citep{srivastava2014dropout, gal2016dropout}.
  Here we adopt this simple treatment to highlight the main contribution.
} $q(s,v|x) = \clN(s,v | \eta(x), \Sigma_q)$.
For all the three cases of \ourmodel, \ourmodel{}-ind and \ourmodel{}-DA, only one inference model for $(s,v) \mid x$ is required.
The component $(s,v) \mid x$ of the discriminative model thus parameterizes $\qind(s,v|x)$ and $\qqt(s,v|x)$ for \ourmodel{}-ind and \ourmodel{}-DA.
The expectations in all objectives (except for expectations over $p^*$ which are estimated by averaging over data) are all under the respective $(s,v) \mid x$.
They can be estimated using $\eta(x)$ by the reparameterization trick~\citep{kingma2014auto}, and the gradients can be back-propagated.

We need two more components beyond the discriminative model to implement the method, \ie the prior $p(s,v)$ and the generative mechanism $p(x|s,v)$.
The latter can be implemented using a generator or decoder architecture comparable to the component $q(s,v|x)$.
The prior can be commonly implemented using a multivariate Gaussian distribution,
$p(s,v) = \clN(\begin{psmallmatrix} s \\ v \end{psmallmatrix} | \begin{psmallmatrix} \mu_s \\ \mu_v \end{psmallmatrix},
\Sigma = \begin{psmallmatrix} \Sigma_{ss} & \Sigma_{sv} \\ \Sigma_{vs} & \Sigma_{vv} \end{psmallmatrix})$.
In implementation, the means $\mu_s$ and $\mu_v$ are fixed as zero vectors.
We parameterize $\Sigma$ via its Cholesky decomposition, $\Sigma = L L\trs$, where $L$ is a lower-triangular matrix with positive diagonals, which is in turn parameterized as
$L = \begin{psmallmatrix} L_{ss} & 0 \\ M_{vs} & L_{vv} \end{psmallmatrix}$ with smaller lower-triangular matrices $L_{ss}$ and $L_{vv}$ and any matrix $M_{vs}$.
Matrices $L_{ss}$ and $L_{vv}$ are parameterized by a summation of positive diagonals (guaranteed via an exponential map) and a lower-triangular (excluding diagonals) matrix.
Training \ourmodel{}-ind via \eqref{eqn:elbo-src-ind} requires estimating the ratio $\frac{p(s,v)}{\pind(s,v)} = \frac{p(s,v)}{p(s) p(v)} = \frac{p(v|s)}{p(v)}$,
where $p(v) = \clN(v | \mu_v, \Sigma_{vv})$ with $\Sigma_{vv} = L_{vv} L_{vv}\trs + M_{vs} M_{vs}\trs$,
and the conditional distribution $p(v|s)$ is given by $p(v|s) = \clN(v | \mu_{v|s}, \Sigma_{v|s})$
with $\mu_{v|s} = \mu_v + M_{vs} L_{ss}^{-1} (s - \mu_s)$, $\Sigma_{v|s} = L_{vv} L_{vv}\trs$ (see \eg, \citet{bishop2006pattern}).
This prior does not imply a causal direction between $s$ and $v$ (the linear Gaussian case of \citet{zhang2009identifiability}) thus well serves as a prior for \ourmodel.


\subsection{Model Selection Details} \label{supp:meth-mselect}

We use a validation set on the training domain for hyperparameter selection, to avoid overfitting due to the finiteness of training data samples, and to guarantee a good fit to the training-domain data distribution $p^*(x,y)$ as the semantic-identifiability theorem~\ref{thm:id} recommends.
We note that model selection in OOD prediction tasks is itself controversial and nontrivial, and it is still an active research direction~\citep{you2019towards, gulrajani2020search}.
It is argued that if a validation set from the test domain is available, the OOD setup that there is no supervision on the test domain is violated,
and then a better choice would be to incorporate it in learning as the semi-supervised adaptation task, instead of using it just for validation.
As our methods are designed to fit the training domain data and our theory shows guarantees under a good fit to the training-domain data distribution, model selection using a training-domain validation set is reasonable.
This does not contradict the trade-off between training- and test-domain accuracies shown in some prior works (\eg,~\citep{rothenhausler2018anchor}),
since they consider arbitrary distribution change, and using the same prediction rule in both domains, while we leverage causal invariance and develop a different prediction rule in the test domain.
In implementation, the training and validation sets are constructed by a 80\%-20\% random split of all training-domain data in each task.

More specifically, for hyperparameter selection, we align the scale of the supervision loss terms ($\bbE_{p^*(x,y)}[\log \pi(y|x)]$ for \ourmodel{}-ind/-DA and \ourmodelablated{}-DA, and the CE loss term for others) in the objectives of all methods, and tune the coefficients of the ELBOs to be their largest values that make the accuracy near 1 on the validation set, 
so that they wield the most power on the test domain while being faithful to explicit supervision.
The coefficients are preferred to be large to well fit $p^*(x)$ (and $\ppt^*(x)$ for domain adaptation) to gain generalizability in the test domain, while they should not affect training accuracy, which is required for a good fit to the training distribution.

For \ourmodel{}-ind/-DA and \ourmodelablated{}-DA, since their inference models target the test domain, it is not reasonable to evaluate validation accuracy directly using them in the form of $\bbE_{q^{\text{test}}(s,v|x)}[p(y|s)]$ ($q^{\text{test}}$ here refers to $\qind$ or $\qqt$).
Instead, \eqref{eqn:q-and-pi} shows that $\pi(y|x) := \bbE_{q^{\text{test}}(s,v|x)} [\frac{p(s,v)}{p^{\text{test}}(s,v)} p(y|s)]$ ($p^{\text{test}}(s,v)$ refers to $\pind(s,v)$ or $\ppt(s,v)$) is an unnormalized density of $q(y|x)$, the training-domain predictor.
So we evaluate $\pi(y|x)$ for every value of $y$ (which is not too large for classification tasks) and normalize them for the validation accuracy.

Compared with recent model selection methods~\citep{you2019towards, ye2021towards}, our method does not introduce additional hyperparameters or assumptions, and does not require multiple training domains.
These advantages stem from the explicit description of domain change of our \ourmodel model based on the causal invariance principle~\ref{prin:inv}.

\section{Experiment Details} \label{supp:expm}

\paragraph{The \ourmodelablated baseline for ablation study.}
To show the benefit of modeling $s$ and $v$ separately, we consider a counterpart of \ourmodel that does not separate its latent variable $z$ into $s$ and $v$; or equivalently, it does not remove the edge $v \to y$.
This means that all its latent variables in $z$ directly (\ie, not mediated by $s$) affect the output $y$.
We thus call it \ourmodelablated.
Detailed methods for OOD generalization (\ourmodelablated; note it does not have a ``-ind'' version) and domain adaptation (\ourmodelablated{}-DA) are introduced in Appx.~\ref{supp:meth-obj-svae}.
To align the model architecture for fair comparison, this means that the latent variable $z$ of \ourmodelablated can only be taken as the latent variable $s$ in \ourmodel (see Appx.~\ref{supp:meth-instant}, Fig.~\ref{fig:discr-sv}).

\paragraph{More about the baselines.}
The \ourmodelablated{}(-DA) baselines are implemented in our codebase along with the proposed \ourmodel{}(-ind/-DA) methods.
The CNBB method~\citep{he2019towards} as an OOD generalization baseline is also implemented, based on the description in the paper.
For domain adaptation baseline methods DANN~\citep{ganin2016domain}, DAN~\citep{long2015learning}, CDAN~\citep{long2018conditional} and MDD~\citep{zhang2019bridging},
we use their implementation in the \texttt{dalib} package\footnote{\url{https://github.com/thuml/Transfer-Learning-Library}}~\citep{dalib}.
The BNM method~\citep{cui2020towards} is integrated into our codebase based on its official implementation\footnote{\url{https://github.com/cuishuhao/BNM}}.
Results of CE, DANN, DAN and CDAN are taken from~\citep{long2018conditional} for the ImageCLEF-DA dataset and from~\citep{gulrajani2020search} except DAN for the PACS and VLCS datasets.
All methods share the same optimization setup.

Note that we do not consider domain generalization baselines (\eg, invariant risk minimization~\citep{arjovsky2019invariant}) as they degenerate to the CE baseline (\ie, the standard supervised learning method, or empirical risk minimization) when given only one training domain.

\paragraph{Computation infrastructure.}
Each run of the experiment is on a single Tesla P100 GPU.
All the experiments are implemented in PyTorch~\citep{paszke2019pytorch}.

\paragraph{More analysis on the results.}
Complete results including the MDD, \ourmodelablated and \ourmodelablated{}-DA baselines, as well as the VLCS~\citep{fang2013unbiased} dataset,
are shown in Table~\ref{tab:res-oodgen} for OOD generalization and in Table~\ref{tab:res-da} for domain adaptation.
The complete results support the same conclusions in the main text.

In addition, for the \textbf{ablation study}, we observe that our \ourmodel methods outperform \ourmodelablated methods in all tasks,
demonstrating the benefit of modeling the semantic and variation factors separately.
Also, \ourmodelablated methods usually have a larger variance, possibly due to the lack of semantic-identifiability so the learned representation gets misled by the variation factor more or less from run to run.
On the other hand, \ourmodelablated methods still outperform existing methods most of the time, which are discriminative methods.
This shows the advantage of using a \emph{generative model}: the invariance of generative mechanisms (causal invariance) is more reliable.

From the domain adaptation results in Table~\ref{tab:res-da}, we note that the advantage of \ourmodel{}-DA on ImageCLEF-DA is not as significant as on other datasets (shifted-MNIST, PACS, VLCS);
existing methods CDAN and BNM achieve a comparable or sometimes better result than \ourmodel{}-DA on ImageCLEF-DA.
This reveals the \textbf{suitable problem} that our \ourmodel methods solve the best, as discussed in the main text.
We expand the analysis below.

Generally speaking, most domain adaptation methods are designed to extract prediction-informative features that are also common across domains,
but at the risk to end up with such a feature that leverages a spurious correlation and misleads prediction.
In contrast, our \ourmodel methods can be seen to filter out misleading candidates of such features,
but with the requirement for identifiability that the training domain shows a diverse $v$ for each $s$.
This requirement comes from the bounded prior condition in the identifiability theorem~\ref{thm:id}, or the intuition to reduce the risk of extreme cases (Thm.~\ref{thm:id} Remark~\bfone).

For the ImageCLEF-DA task, there is no severe spurious correlation, since the style factor as $v$ has no preference on a particular class in any domain.
So existing domain adaptation methods do not meet a serious problem.
But the task is hard for identifiability: for each value of a semantic factor, a single elementary training domain cannot show a diverse variation factor.
This weakens the power of \ourmodel{}-DA.
On other datasets (shifted-MNIST, PACS, VLCS), spurious correlation is stronger.
Shifted-MNIST is deliberately constructed to show a strong digit-position correlation in the training domain while the correlation disappears in test domains.
As for PACS and VLCS, whenever different domains have different class proportions, pooling them together introduces a class-style(domain) correlation, which does not hold in a test domain.
On the other hand, the training domain of shifted-MNIST shows a noisy position for each digit, and the pooled training domains of PACS and VLCS show a diverse style for each class.
So these datasets better satisfy the requirement of \ourmodel{}-DA meanwhile ameliorating spurious correlation is the key problem.
This makes the advantage of \ourmodel{}-DA more salient.

\begin{table*}[t]
  \centering
  \setlength{\tabcolsep}{4.0pt}
  \caption{Test accuracy (\%) for \textbf{OOD generalization} by various methods (ours in bold and line separated; \ourmodelablated baseline included) on \textbf{Shifted-MNIST} (top two rows), \textbf{ImageCLEF-DA} (mid-top four rows), \textbf{PACS} (mid-bottom four rows) and \textbf{VLCS} (bottom four rows) datasets.
    Results of CE are taken from~\citep{long2018conditional} for ImageCLEF-DA and from~\citep{gulrajani2020search} for PACS and VLCS.
    Averaged over 10 runs.
  }
  \label{tab:res-oodgen}
  \vspace{-4pt}
  \begin{tabular}{cc||ccc|cc}
    \toprule
    \multicolumn{2}{c||}{task} &
    CE & CNBB & \ourmodelablated & \textbf{\ourmodel} & \textbf{\ourmodel{}-ind} \\
    \midrule
    \multirow{2}{*}{\parbox[c]{30pt}{\raggedleft Shifted-MNIST}}
    & $\delta_0 = \delta_1 = 0$ &
    42.9\subpm{3.1} & 54.7\subpm{3.3} & 53.0\subpm{6.7} & 81.4\subpm{7.4} & \textbf{82.6\subpm{4.0}} \\
    & $\delta_0, \delta_1 \! \sim \! \clN (0, \! 2^2)$ &
    47.8\subpm{1.5} & 59.2\subpm{2.4} & 54.8\subpm{5.6} & 61.7\subpm{3.6} & \textbf{62.3\subpm{2.2}} \\
    \midrule
    \multirow{4}{*}{\parbox[c]{30pt}{\raggedleft Image CLEF-DA}}
    & \textbf{C}$\to$\textbf{P} &
    65.5\subpm{0.3} & 72.7\subpm{1.1} & 73.3\subpm{1.0} & 73.6\subpm{0.6} & \textbf{74.0\subpm{1.3}} \\
    & \textbf{P}$\to$\textbf{C} &
    91.2\subpm{0.3} & 91.7\subpm{0.2} & 91.6\subpm{0.9} & 92.3\subpm{0.4} & \textbf{92.7\subpm{0.2}} \\
    & \textbf{I}$\to$\textbf{P} &
    74.8\subpm{0.3} & 75.4\subpm{0.6} & 77.0\subpm{0.2} & 76.9\subpm{0.3} & \textbf{77.2\subpm{0.2}} \\
    & \textbf{P}$\to$\textbf{I} &
    83.9\subpm{0.1} & 88.7\subpm{0.5} & 90.4\subpm{0.3} & 90.4\subpm{0.3} & \textbf{90.9\subpm{0.2}} \\
    \midrule
    \multirow{4}{*}{\parbox[c]{30pt}{\raggedleft PACS}}
    & others$\to$\textbf{P} &
    \textbf{97.8\subpm{0.0}} & 96.9\subpm{0.2} & 97.7\subpm{0.3} & 97.7\subpm{0.2} & \textbf{97.8\subpm{0.2}} \\
    & others$\to$\textbf{A} &
    88.1\subpm{0.1} & 73.1\subpm{0.3} & 87.3\subpm{0.8} & \textbf{88.5\subpm{0.6}} & \textbf{88.6\subpm{0.6}} \\
    & others$\to$\textbf{C} &
    77.9\subpm{1.3} & 50.2\subpm{1.2} & 84.3\subpm{0.9} & 84.4\subpm{0.9} & \textbf{84.6\subpm{0.8}} \\
    & others$\to$\textbf{S} &
    79.1\subpm{0.9} & 43.3\subpm{1.2} & 80.6\subpm{1.4} & 80.7\subpm{1.0} & \textbf{81.1\subpm{1.2}} \\
    \midrule
    \multirow{4}{*}{\parbox[c]{30pt}{\raggedleft VLCS}}
    & others$\to$\textbf{V} &
    76.4\subpm{1.5} & 75.5\subpm{0.9} & 79.4\subpm{1.0} & 79.3\subpm{1.1} & \textbf{80.0\subpm{0.9}} \\
    & others$\to$\textbf{L} &
    63.3\subpm{0.9} & 61.1\subpm{1.2} & 69.6\subpm{0.8} & 69.6\subpm{0.5} & \textbf{70.1\subpm{0.8}} \\
    & others$\to$\textbf{C} &
    97.6\subpm{1.0} & 97.1\subpm{0.4} & 99.2\subpm{0.3} & \textbf{99.4\subpm{0.3}} & \textbf{99.5\subpm{0.2}} \\
    & others$\to$\textbf{S} &
    72.2\subpm{0.5} & 73.7\subpm{0.6} & 75.0\subpm{0.9} & 76.1\subpm{1.3} & \textbf{76.9\subpm{1.2}} \\
    \bottomrule
  \end{tabular}
%
  \bigskip
  \setlength{\tabcolsep}{3.0pt}
  \caption{Test accuracy (\%) for \textbf{domain adaptation} by various methods (ours in bold and line separated; BNM and \ourmodelablated{}-DA baselines included) on \textbf{Shifted-MNIST} (top two rows), \textbf{ImageCLEF-DA} (mid-top four rows), \textbf{PACS} (mid-bottom four rows) and \textbf{VLCS} (bottom four rows) datasets.
    Results of DANN, DAN and CDAN on ImageCLEF-DA are taken from~\citep{long2018conditional}, and results of DANN and CDAN on PACS and VLCS are taken from~\citep{gulrajani2020search}.
    Averaged over 10 runs.
  }
  \label{tab:res-da}
  \vspace{-4pt}
  \begin{tabular}{cc@{}||cccccc|c}
    \toprule
    \multicolumn{2}{c||}{task} &
    DANN & DAN & CDAN & MDD & BNM & \ourmodelablated{}-DA & \textbf{\ourmodel{}-DA} \\
    \midrule
    \multirow{2}{*}{\parbox[c]{24pt}{\raggedleft Shifted-MNIST}}
    & $\delta_0 = \delta_1= 0$ &
    40.9\subpm{3.0} & 40.4\subpm{2.0} & 41.0\subpm{0.5} & 41.9\subpm{0.8} & 40.8\subpm{1.0} & 78.0\subpm{27.2} & \textbf{97.6\subpm{4.0}} \\
    & $\delta_0, \! \delta_1 \! \sim \!\! \clN \!(0, \! 2^2) \,$ &
    46.2\subpm{0.7} & 45.6\subpm{0.7} & 46.3\subpm{0.6} & 45.8\subpm{0.3} & 45.7\subpm{1.0} & 68.1\subpm{17.4} & \textbf{72.0\subpm{9.2}} \\
    \midrule
    \multirow{4}{*}{\parbox[c]{30pt}{\raggedleft Image CLEF-DA}}
    & \textbf{C}$\to$\textbf{P} &
    74.3\subpm{0.5} & 69.2\subpm{0.4} & 74.5\subpm{0.3} & 74.1\subpm{0.7} & \textbf{75.2\subpm{1.4}} & 74.3\subpm{0.3} & \textbf{75.1\subpm{0.5}} \\
    & \textbf{P}$\to$\textbf{C} &
    91.5\subpm{0.6} & 89.8\subpm{0.4} & \textbf{93.5\subpm{0.4}} & 92.1\subpm{0.6} & \textbf{93.5\subpm{2.8}} & 92.7\subpm{0.4} & \textbf{93.4\subpm{0.3}} \\
    & \textbf{I}$\to$\textbf{P} &
    75.0\subpm{0.6} & 74.5\subpm{0.4} & 76.7\subpm{0.3} & 76.8\subpm{0.4} & 76.7\subpm{1.4} & 77.0\subpm{0.3} & \textbf{77.4\subpm{0.3}} \\
    & \textbf{P}$\to$\textbf{I} &
    86.0\subpm{0.3} & 82.2\subpm{0.2} & 90.6\subpm{0.3} & 90.2\subpm{1.1} & \textbf{91.0\subpm{0.8}} & 90.6\subpm{0.4} & \textbf{91.1\subpm{0.5}} \\
    \midrule
    \multirow{4}{*}{\parbox[c]{30pt}{\raggedleft PACS}}
    & others$\to$\textbf{P} &
    97.6\subpm{0.2} & 97.6\subpm{0.4} & 97.0\subpm{0.4} & 97.6\subpm{0.3} & 87.6\subpm{4.2} & 97.6\subpm{0.4} & \textbf{97.9\subpm{0.2}} \\
    & others$\to$\textbf{A} &
    85.9\subpm{0.5} & 84.5\subpm{1.2} & 84.0\subpm{0.9} & 88.1\subpm{0.8} & 86.4\subpm{0.4} & 88.0\subpm{0.8} & \textbf{88.8\subpm{0.7}} \\
    & others$\to$\textbf{C} &
    79.9\subpm{1.4} & 81.9\subpm{1.9} & 78.5\subpm{1.5} & 83.2\subpm{1.1} & 83.6\subpm{1.7} & \textbf{84.6\subpm{0.9}} & \textbf{84.7\subpm{0.8}} \\
    & others$\to$\textbf{S} &
    75.2\subpm{2.8} & 77.4\subpm{3.1} & 71.8\subpm{3.9} & 80.2\subpm{2.2} & 59.1\subpm{1.5} & 80.9\subpm{1.2} & \textbf{81.4\subpm{0.8}} \\
    \midrule
    \multirow{4}{*}{\parbox[c]{30pt}{\raggedleft VLCS}}
    & others$\to$\textbf{V} &
    78.3\subpm{0.3} & 74.6\subpm{0.8} & 76.9\subpm{0.2} & 79.0\subpm{1.1} & 70.0\subpm{2.5} & 79.1\subpm{1.4} & \textbf{81.1\subpm{0.8}} \\
    & others$\to$\textbf{L} &
    64.9\subpm{1.1} & 67.1\subpm{0.5} & 65.2\subpm{0.4} & 63.8\subpm{0.8} & 54.0\subpm{5.9} & 69.6\subpm{0.9} & \textbf{70.2\subpm{0.7}} \\
    & others$\to$\textbf{C} &
    98.5\subpm{0.2} & 98.5\subpm{0.6} & 97.5\subpm{0.1} & 99.3\subpm{0.3} & 96.5\subpm{5.1} & 99.3\subpm{0.3} & \textbf{99.5\subpm{0.2}} \\
    & others$\to$\textbf{S} &
    73.1\subpm{0.7} & 75.0\subpm{1.1} & 73.4\subpm{1.1} & 75.8\subpm{1.8} & 66.8\subpm{2.0} & 76.1\subpm{1.8} & \textbf{77.1\subpm{1.1}} \\
    \bottomrule
  \end{tabular}
  \vspace{-4pt}
\end{table*}

\subsection{Shifted-MNIST} \label{supp:expm-smnist}


\paragraph{Dataset.}
The dataset is based on the standard MNIST dataset\footnote{\url{http://yann.lecun.com/exdb/mnist/}}, where only images of ``0'' and ``1'' are collected.
The resulting training set has 5,923 (46.77\%) ``0''s and 6,742 (53.23\%) ``1''s (12,665 in total) and the test set has 980 (46.34\%) ``0''s and 1,135 (53.66\%) ``1''s (2,115 in total).
As described in the main text, we horizontally shift each ``0'' in the training data at random by $\delta_0$ pixels where $\delta_0 \sim \clN(-5, 1^2)$, and each ``1'' by $\delta_1 \sim \clN(5, 1^2)$ pixels.
We construct two test sets, where in the first one, each digit from the test set is not moved $\delta_0 = \delta_1 = 0$, and is horizontally shifted randomly by $\delta_0, \delta_1 \sim \clN(0, 2^2)$ pixels in the second.
All domains have balanced classes.

\paragraph{Setup and implementation details.}
For generative methods (\ie, \ourmodelablated{}(-DA) and our methods \ourmodel{}(-ind/-DA)), we use a multilayer perceptron (MLP) with 784(for $x$)-400-200(first 100 for $v$)-50(for $s$ or $z$)-1(for $y$) nodes in each layer for the inference model,
and use an MLP with 50(for $s$)-(100(for $v$)+100)-400-784(for $x$) nodes in each layer for the generative component (\ie, the mean function of the additive Gaussian $p(x|s,v)$).
The activation function in the MLPs is the sigmoid function, and the variables $s$ and $v$ are taken after the activation.
The expectation under $q(s,v|x)$ in ELBO is estimated by evaluating the function at the mode of the additive Gaussian with reparameterization.
For discriminative methods (\ie, CE, CNBB, DANN, DAN, CDAN, MDD, BNM), we use a larger MLP architecture with 784-600-300-75-1 nodes in each layer to compensate the additional parameters of the generative component in generative methods.

For all the methods, we use a mini-batch of size $128$ in each optimization step, and use the RMSprop optimizer~\citep{tieleman2012lecture}, with weight decay parameter $1\e{-5}$, and learning rate $1\e{-3}$ for OOD generalization and $3\e{-4}$ for domain adaptation.
These hyperparameters are chosen by running and validating using CE and DANN.
For generative methods, we take the additive Gaussian variance of the generative mechanism $p(x|s,v)$ 
as $0.03^2$.
The scale of the standard derivations of these additive Gaussian distributions are chosen small to meet the intense causal mechanism assumption in our theory.\footnote{
  Choosing small variances is also supported by a direct analysis of additive Gaussian VAEs~\citep{dai2019diagnosing} for well learning the data manifold.
}
For the Gaussian variances of $s$ and $v$ in $q(s,v|x)$, they are also outputs from the discriminative model through additional branches.
Each of these branches is a fully-connected layer forked from the last layer of $s$ or $v$, with a softplus activation to ensure positivity.
Their weights are learned via the same objectives.

\paragraph{Hyperparameter configurations.}
For both OOD generalization and domain adaptation tasks on the two test domains, we train the models for $100$ epochs (average runtime $10$ minutes) when all the methods converge in terms of loss and validation accuracy.
We align the scale of the supervision loss terms in the objectives of all methods, and scale the ELBO terms with the largest weight that makes training accuracy near 1 in OOD generalization.
We then fix the tuned ELBO weight and scale the weight of adaptation terms in a similar way for domain adaptation.
Other parameters are tuned similarly.
For generative methods (\ie, \ourmodelablated{}(-DA) and our methods \ourmodel{}(-ind/-DA)), the ELBO weight is $1\e{-4}$ selected from $\{1,3\}\e{\{-1,-2,\cdots,-6\}}$.
For domain adaptation methods, the adaptation weight is $1\e{-4}$ for DANN, $1\e{-8}$ for DAN, $1\e{-6}$ for CDAN, $1\e{-6}$ for MDD, $1\e{-7}$ for BNM, and $1\e{-4}$ for \ourmodelablated{}-DA and \ourmodel{}-DA, all selected from $1\e{\{-1,-2,\cdots,-8\}}$.
For CNBB, we use regularization coefficients $1\e{-4}$ and $3\e{-6}$ to regularize the sample weight and learned representation, and run $4$ inner gradient descent iterations with learning rate $1\e{-3}$ to optimize the sample weight.
These four parameters are selected from a grid search where the range of the parameters are: $\{1,3\}\e{\{-2,-3,-4\}}$, $\{1,3\}\e{\{-4,-5,-6\}}$, $\{4,8\}$, $1\e{\{-1,-2,-3\}}$.

\subsection{ImageCLEF-DA} \label{supp:expm-imageclef}

\paragraph{Dataset.}
ImageCLEF-DA\footnote{\url{http://imageclef.org/2014/adaptation}} is a standard benchmark dataset for the ImageCLEF 2014 domain adaptation challenge~\citep{imageclef2014}.
There are three domains in this dataset: \textbf{C}altech-256, \textbf{I}mageNet and \textbf{P}ascal~VOC~2012.
Each domain has 12 classes and 600 images. 
Each image is center-cropped to shape $(3,224,224)$ as $x$ (also for PACS and VLCS experiments).

\paragraph{Setup and implementation details.}
We adopt the same setup as in \citet{long2018conditional} \footnote{\url{https://github.com/thuml/CDAN}} for a common practice and fair comparison with existing results.
This means that we use the ResNet50 structure~\citep{he2016deep} pretrained on the ImageNet dataset as the backbone of the discriminative/inference model.
For \ourmodel{}(-ind/-DA), we select the first $128$ dimensions of the bottleneck layer (\ie, the layer that replaces the last fully-connected layer of the pretrained ResNet50; its output dimension is $1024$) as the variable $v$, 
and take $s$ as the $256$-dimensional output of the two-layer MLP (with $1024$ hidden nodes) built on the bottleneck layer.
Both $s$ and $v$ are taken before activation.
The logits for $y$ is produced by a linear layer built on $s$.

For generative methods (\ie, \ourmodelablated{}(-DA) and our methods \ourmodel{}(-ind/-DA)), we construct an image decoder/generator for the mean function of the additive Gaussian $p(x|s,v)$ that uses the DCGAN generator model~\citep{radford2015unsupervised} pretrained on the Cifar10 dataset as the backbone.
The pretrained DCGAN is taken from the PyTorch-GAN-Zoo\footnote{\url{https://github.com/facebookresearch/pytorch_GAN_zoo}}.
The generator connects to the DCGAN backbone by an MLP with 384(dimension of $(s,v)$)-128-120(input dimension of DCGAN) nodes in each layer, and generates images of desired size $(3,224,224)$ by appending to the output of DCGAN of size $(3,64,64)$ with an transposed convolution layer with kernel size 4, stride size 4, and padding size 16.
The expectation under $q(s,v|x)$ in ELBO is estimated by evaluating the function at the mean of the conditional Gaussian with reparameterization.

Following \citet{long2018conditional}, we use a mini-batch of size $n_B = 32$ in each optimization step, and adopt the SGD optimizer with Nesterov momentum parameter $0.9$, weight decay parameter $5\e{-4}$,
and a shrinking step size scheme $\varepsilon_i = \varepsilon_0 (1 + \alpha n_B i)^{-\beta}$ for optimization iteration $i$, with initial scale $\varepsilon_0 = 1\e{-3}$, per-datum coefficient\footnote{
  The coefficient $\alpha$ here is amortized onto each datum, so its value is different from that in \citet{long2018conditional} and a batch size $n_B$ is multiplied to the iteration number $i$.
} $\alpha = 6.25\e{-6}$, and shrinking exponent $\beta = 0.75$.
For the parameters of the backbone components, a $10$ times smaller learning rate is used.
For generative methods, the Gaussian variances of $s$ and $v$ in $q(s,v|x)$ are also outputs from the discriminative model through additional branches.
Each of these branches is a fully-connected layer forked from the last layer of $s$ or $v$, with a softplus activation to ensure positivity.
Their weights are learned via the same objectives.

\paragraph{Hyperparameter configurations.}
For all the four OOD prediction tasks, we train the models for $30$ epochs (average runtime $10$ minutes) when all the methods converge in terms of loss and validation accuracy.
For generative methods, the Gaussian variance of $p(x|s,v)$ is taken as $0.1$, which is searched within $\{1,3\}\e{\{-4,-2,-1,0,2,4\}}$.
The ELBO weight is $1\e{-7}$ for \ourmodelablated{}(-DA) and is $1\e{-8}$ for our \ourmodel{}(-ind/-DA), both selected from $1\e{\{-2,-4,-6\}} \cup \{1,3\}\e{\{-7,-8,-9,-10\}}$.
The adaptation weight is $1\e{-8}$ selected from $1\e{\{-2,-4,-6\}} \cup \{1,3\}\e{\{-7,-8,-9,-10\}}$ for both \ourmodelablated{}-DA and \ourmodel{}-DA,
$1\e{-2}$ selected from $1\e{\{-1,-2,-4,-6\}}$ for MDD,
and $1.0$ selected from $1\e{\{1,0,-1,-2,-4\}}$ for BNM.
Results of other domain adaptation baselines DANN, DAN and CDAN and the results of CE are taken from~\citep{long2018conditional} under the same setting.
For CNBB, we use regularization coefficients $1\e{-6}$ and $3\e{-6}$ to regularize the sample weight and learned representation, and run $4$ inner gradient descent iterations with learning rate $1\e{-4}$ to optimize the sample weight.
These four parameters are selected from a grid search where the range of the parameters are: $1\e{\{-4,-5,-6,-7\}} \cup \{3\e{-6}\}$, $\{1,3\}\e{\{-5,-6,-7\}}$, $\{4\}$, $1\e{\{-2,-3,-4,-5\}}$.

\subsection{PACS} \label{supp:expm-pacs}

\paragraph{Dataset.}
The PACS dataset~\citep{li2017deeper} has 7 classes.
It is named after its four domains: \textbf{P}hoto, \textbf{A}rt, \textbf{C}artoon, \textbf{S}ketch;
each contains images of a certain style.
It contains 9,991 images in total.
We use the dataset via the open-source \texttt{domainbed} repository\footnote{\url{https://github.com/facebookresearch/DomainBed}}~\citep{gulrajani2020search}.

\paragraph{Setup and implementation details.}
We adopt the same setup as in \citet{gulrajani2020search} for a common practice and fair comparison with existing results.
This means for each domain as the test domain, the single training domain is constructed by merging/pooling the other three domains.
This is done by merging the three mini-batches of size $32$ from each of the three domains for optimization.
The Adam optimizer~\citep{kingma2014adam} with learning rate $5\e{-5}$ 
is adopted.
Data augmentation is conducted by random flip and crop, gray-scaling and color-jitter (\ie, randomly changing brightness, contrast, saturation and hue).
Other setups are basically the same as in the ImageCLEF-DA experiment, except that the layer for variable $s$ has $512$ nodes, and that the backbone components use the same learning rate (\ie, not multiplied by $0.1$).

\paragraph{Hyperparameter configurations.}
For all methods we train for $40$ epochs (average runtime $30$ minutes) when they all converge in terms of loss and validation accuracy.
For all generative methods (\ie, \ourmodelablated{}(-DA) and our methods \ourmodel{}(-ind/-DA)), the Gaussian variance of $p(x|s,v)$ is taken as $0.3$.
The ELBO weight is $1\e{-7}$ for \ourmodelablated, \ourmodel and \ourmodel{}-ind, and is $1\e{-8}$ for \ourmodelablated{}-DA and \ourmodel{}-DA, both selected from $1\e{\{0,-2,-4,-5,-6,-7,-8,-9\}}$.
The adaptation weight is $1\e{-8}$ selected from $1\e{\{0,-2,-4,-6,-7,-8,-9\}}$ for \ourmodelablated{}-DA and \ourmodel{}-DA, $1\e{-2}$ selected from $1\e{\{0,-1,-2,-3,-4,-6\}}$ for DAN, and is the same as in the ImageCLEF-DA experiment for MDD and BNM.
Results of other domain adaptation baselines DANN and CDAN and the results of CE are taken from~\citep{gulrajani2020search} under the same setting.
For CNBB, the hyperparameters are the same as in the ImageCLEF-DA experiment, except the regularization coefficients for sample weights is $1\e{-4}$.
These hyperparameters are selected from the same range as used in the ImageCLEF-DA experiment.

\paragraph{Results using single training domains.}
We also conducted an experiment on PACS with single training domains, similar to the setup on ImageCLEF-DA.
The results are presented in Table~\ref{tab:res-single-traindom}.
We see that the advantage of our methods is not as significant as in the standard pooled training domain case.
This agrees with the discussion in the ``dataset analysis'' in the main paper:
our methods are more powerful in handling a misleading spurious $s$-$v$ correlation but which needs to be diverse/stochastic enough to allow identification,
following the intuition on the identifiability (Thm.~\ref{thm:id} Remark~\bfone).

\begin{table*}[t]
  \centering
  \setlength{\tabcolsep}{4.0pt}
  \caption{Test accuracy (\%) for \textbf{OOD generalization} (middle 4 columns) and \textbf{domain adaptation} (right 3 columns) by various methods (ours in bold and line separated) on \textbf{PACS} with \textbf{single training domains}.
    Averaged over 10 runs.
  }
  \label{tab:res-single-traindom}
  \vspace{-4pt}
  \begin{tabular}{c@{}c||cc|cc||cc|c}
    \toprule
    \multicolumn{2}{c||}{task} &
    CE & \ourmodelablated & \textbf{\ourmodel} & \textbf{\ourmodel{}-ind} &
    DAN & \ourmodelablated{}-DA & \textbf{\ourmodel{}-DA} \\
    \midrule
    \multirow{4}{*}{\parbox[c]{30pt}{\raggedleft PACS}}
    & \textbf{C}$\to$\textbf{A} &
    \textbf{78.9\subpm{1.1}} & 78.2\subpm{1.8} & 78.4\subpm{1.2} & \textbf{78.9\subpm{1.3}} &
    \textbf{80.9\subpm{1.2}} & 79.1\subpm{0.7} & 79.1\subpm{0.8} \\
    & \textbf{P}$\to$\textbf{A} &
    73.1\subpm{1.9} & 73.4\subpm{1.9} & \textbf{73.5\subpm{0.9}} & 73.4\subpm{1.5} &
    \textbf{76.6\subpm{2.6}} & 73.8\subpm{0.7} & 75.0\subpm{0.7} \\
    & \textbf{S}$\to$\textbf{A} &
    64.2\subpm{2.8} & 63.4\subpm{1.6} & 63.7\subpm{1.7} & \textbf{65.4\subpm{2.1}} &
    62.4\subpm{1.8} & 64.7\subpm{2.2} & \textbf{65.7\subpm{2.0}} \\
    \cmidrule(lr){2-9}
    & others$\to$\textbf{A} &
    88.1\subpm{0.1} & 87.3\subpm{0.8} & \textbf{88.5\subpm{0.6}} & \textbf{88.6\subpm{0.6}} &
    84.5\subpm{1.2} & 88.0\subpm{0.8} & \textbf{88.8\subpm{0.7}} \\
    \bottomrule
  \end{tabular}
  \vspace{-4pt}
\end{table*}

\subsection{VLCS} \label{supp:expm-vlcs}

The VLCS dataset~\citep{fang2013unbiased} has 5 classes.
It is also named after its four domains: \textbf{V}OC2007, \textbf{L}abelMe, \textbf{C}altech101, \textbf{S}UN09;
each is an image dataset collected in a certain way.
It contains 10,729 images in total.
We use the dataset also via the \texttt{domainbed} repository.
Setup, implementation details and hyperparameters are the same as in the PACS experiment.
Results are shown at the last four rows in Table~\ref{tab:res-oodgen} for OOD generalization and in Table~\ref{tab:res-da} for domain adaptation.

\subsection{Visualization of the Learned Representation} \label{supp:expm-viz}

To better understand how our methods work, we compare the visualization of the learned model by our methods with that by the corresponding baselines.
Visualization is done by the \emph{Local Interpretable Model-agnostic Explanation} (LIME) method~\citep{ribeiro2016why}\footnote{We use the official codebase at \url{https://github.com/marcotcr/lime-experiments}.},
which uses an interpretable model, \eg a linear model, to approximate the target model locally at the query image.
The learned weight of the linear model then reflects the importance of the components/dimensions of the input, \ie pixels in the image, which can be visualized after binarization as focused regions on the image.
This gives a hint on the learned representation by the model for making prediction.

The visualization results are shown in Fig.~\ref{fig:viz}.
We see that in each case, the focused regions of our methods (\ourmodel{}-ind and \ourmodel{}-DA) are more relevant to the semantic of the image, and the boundary of the region reflects the characterizing shape of the object.
In contrast, the baselines also involve much background regions.
This result shows our \ourmodel methods indeed better learn a causal semantic factor for prediction, which supports the motivation to introduce the \ourmodel model, verifies the theory, and explains the better robustness for OOD prediction.

\newcommand{\elemfigwidth}{1.8cm}
\begin{SCtable}[][t]
  \centering
  \setlength{\tabcolsep}{0.6pt}
  \begin{tabular}{c@{\hspace{8pt}}ccccc}
    \toprule
    CE &
    \begin{minipage}{\elemfigwidth}
      \centering
      \raisebox{-.5\height}{\includegraphics[width=\elemfigwidth]{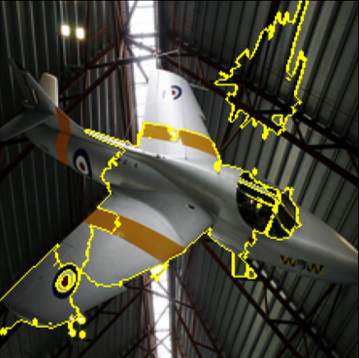}}
    \end{minipage}
    &
    \begin{minipage}{\elemfigwidth}
      \centering
      \raisebox{-.5\height}{\includegraphics[width=\elemfigwidth]{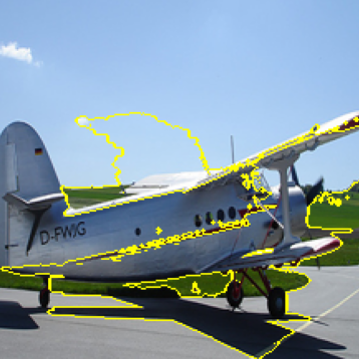}}
    \end{minipage}
    &
    \begin{minipage}{\elemfigwidth}
      \centering
      \raisebox{-.5\height}{\includegraphics[width=\elemfigwidth]{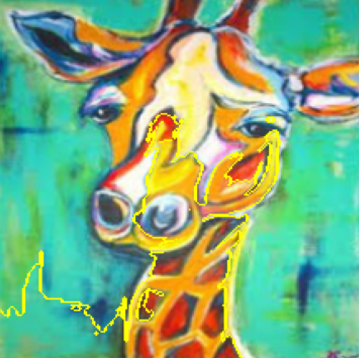}}
    \end{minipage}
    &
    \begin{minipage}{\elemfigwidth}
      \centering
      \raisebox{-.5\height}{\includegraphics[width=\elemfigwidth]{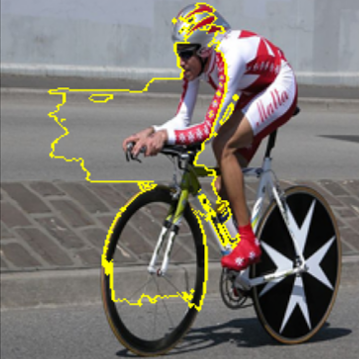}}
    \end{minipage}
    &
    \begin{minipage}{\elemfigwidth}
      \centering
      \raisebox{-.5\height}{\includegraphics[width=\elemfigwidth]{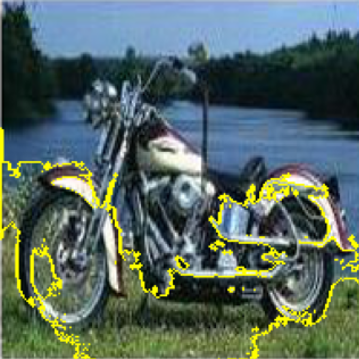}}
    \end{minipage}
    \\ \addlinespace[1pt]
    \textbf{\ourmodel{}-ind} &
    \begin{minipage}{\elemfigwidth}
      \centering
      \raisebox{-.5\height}{\includegraphics[width=\elemfigwidth]{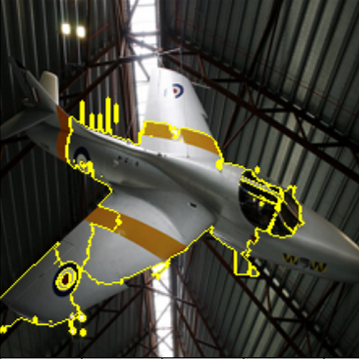}}
    \end{minipage}
    &
    \begin{minipage}{\elemfigwidth}
      \centering
      \raisebox{-.5\height}{\includegraphics[width=\elemfigwidth]{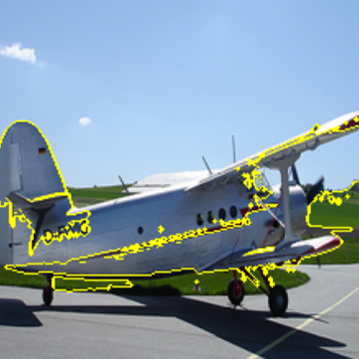}}
    \end{minipage}
    &
    \begin{minipage}{\elemfigwidth}
      \centering
      \raisebox{-.5\height}{\includegraphics[width=\elemfigwidth]{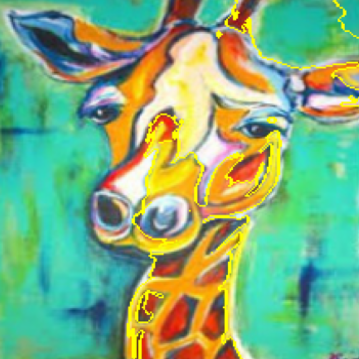}}
    \end{minipage}
    &
    \begin{minipage}{\elemfigwidth}
      \centering
      \raisebox{-.5\height}{\includegraphics[width=\elemfigwidth]{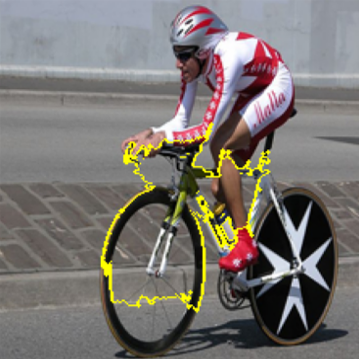}}
    \end{minipage}
    &
    \begin{minipage}{\elemfigwidth}
      \centering
      \raisebox{-.5\height}{\includegraphics[width=\elemfigwidth]{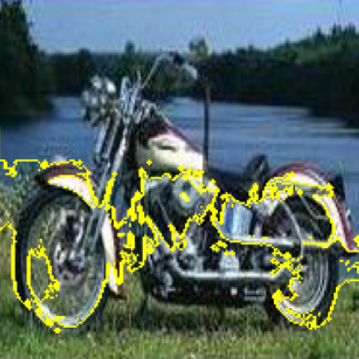}}
    \end{minipage}
    \\ \midrule \addlinespace[1.5pt]
    MDD &
    \begin{minipage}{\elemfigwidth}
      \centering
      \raisebox{-.5\height}{\includegraphics[width=\elemfigwidth]{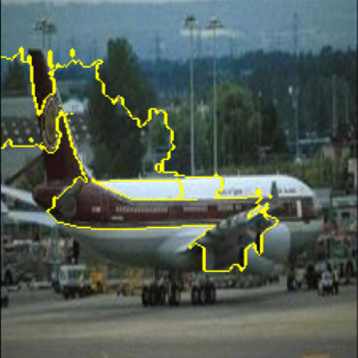}}
    \end{minipage}
    &
    \begin{minipage}{\elemfigwidth}
      \centering
      \raisebox{-.5\height}{\includegraphics[width=\elemfigwidth]{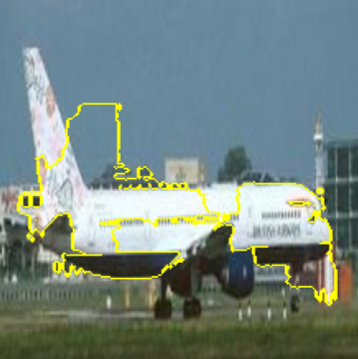}}
    \end{minipage}
    &
    \begin{minipage}{\elemfigwidth}
      \centering
      \raisebox{-.5\height}{\includegraphics[width=\elemfigwidth]{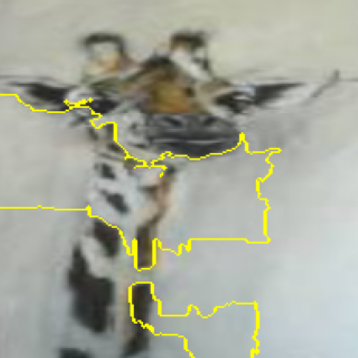}}
    \end{minipage}
    &
    \begin{minipage}{\elemfigwidth}
      \centering
      \raisebox{-.5\height}{\includegraphics[width=\elemfigwidth]{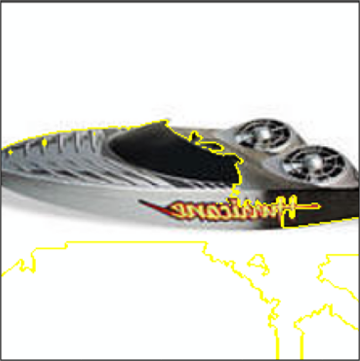}}
    \end{minipage}
    &
    \begin{minipage}{\elemfigwidth}
      \centering
      \raisebox{-.5\height}{\includegraphics[width=\elemfigwidth]{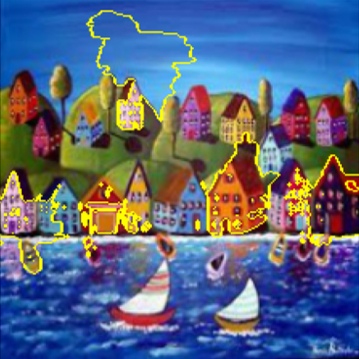}}
    \end{minipage}
    \\ \addlinespace[1pt]
    \textbf{\ourmodel{}-DA} &
    \begin{minipage}{\elemfigwidth}
      \centering
      \raisebox{-.5\height}{\includegraphics[width=\elemfigwidth]{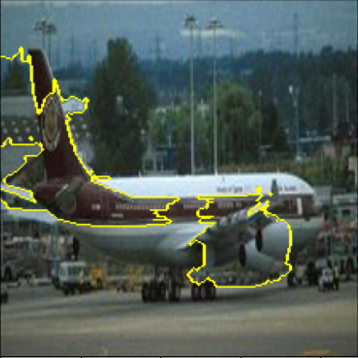}}
    \end{minipage}
    &
    \begin{minipage}{\elemfigwidth}
      \centering
      \raisebox{-.5\height}{\includegraphics[width=\elemfigwidth]{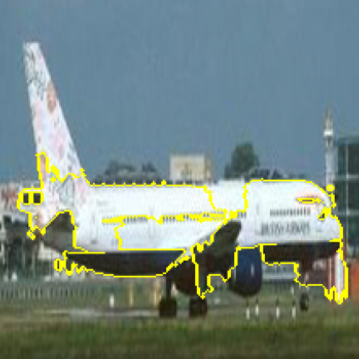}}
    \end{minipage}
    &
    \begin{minipage}{\elemfigwidth}
      \centering
      \raisebox{-.5\height}{\includegraphics[width=\elemfigwidth]{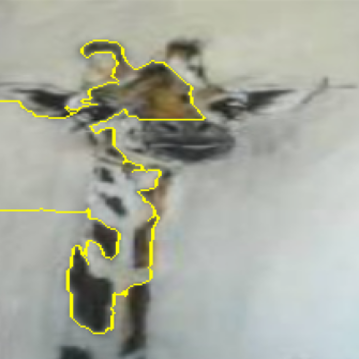}}
    \end{minipage}
    &
    \begin{minipage}{\elemfigwidth}
      \centering
      \raisebox{-.5\height}{\includegraphics[width=\elemfigwidth]{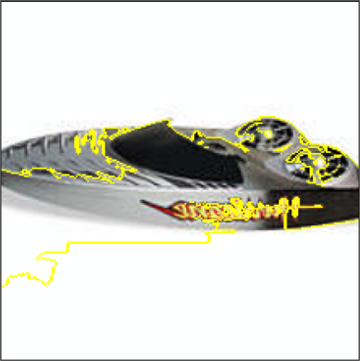}}
    \end{minipage}
    &
    \begin{minipage}{\elemfigwidth}
      \centering
      \raisebox{-.5\height}{\includegraphics[width=\elemfigwidth]{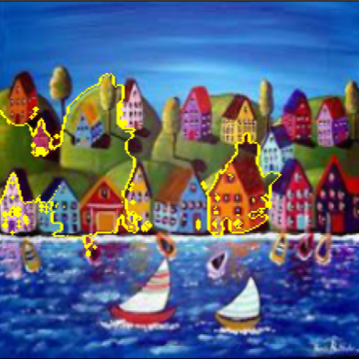}}
    \end{minipage}
    \\ \bottomrule
  \end{tabular}
  \captionsetup[table]{name=Figure,width=60pt,skip=-2pt}
  \captionof{table}{Visualization (via LIME~\citep{ribeiro2016why}) of the learned representation by various methods (ours in bold).
    The top two rows are for OOD generalization and the bottom two rows are for domain adaptation.
  }
  \label{fig:viz}
  \vspace{-8pt}
\end{SCtable}

\end{document}